\documentclass[12pt]{article}
\usepackage[margin=1in]{geometry}

\usepackage{diagbox}
\usepackage{cancel}
\usepackage{amsmath,amssymb,bm,url,mathrsfs}
\usepackage{algorithmic}
\usepackage[ruled,linesnumbered]{algorithm2e}
\usepackage{amsthm}
\usepackage{hyperref}
\usepackage{makecell}
\usepackage{tablefootnote}
\usepackage{comment}
\usepackage{soul}
\usepackage{graphics,graphicx}
\usepackage{subcaption,caption,cleveref}
\hypersetup{
    colorlinks=true,
    linkcolor=blue,
    citecolor=blue,
    filecolor=magenta,      
    urlcolor=cyan,
}
\usepackage{booktabs}
\usepackage{multirow}
\usepackage{multicol}
\usepackage{natbib}
\usepackage{xcolor}
\usepackage{dsfont}
\usepackage{bbm}
\usepackage{enumitem}
\usepackage{blkarray}
\usepackage{amsfonts}
\usepackage{caption}
\usepackage{xr}
\makeatletter
\newcommand*{\addFileDependency}[1]{
  \typeout{(#1)}
  \@addtofilelist{#1}
  \IfFileExists{#1}{}{\typeout{No file #1.}}
}
\makeatother

\usepackage{footmisc}
\def\bar{\overline}

\newcommand{\R}{\mathbb{R}}

\newcommand{\pnorm}[2]{\lVert #1\rVert_{#2}}

\newcommand{\abs}[1]{\lvert#1\rvert}
\newcommand{\bigabs}[1]{\big\lvert#1\big\rvert}
\newcommand{\biggabs}[1]{\bigg\lvert#1\bigg\rvert}
\newcommand{\iprod}[2]{\langle#1,#2\rangle}

\renewcommand{\epsilon}{\varepsilon}

\newcommand{\bigo}{\mathcal{O}}

\renewcommand{\tilde}{\widetilde}

\DeclareMathOperator{\E}{\mathbb{E}}
\DeclareMathOperator{\Prob}{\mathbb{P}}

\DeclareMathOperator{\tr}{tr}
\DeclareMathOperator{\var}{Var}
\DeclareMathOperator{\cov}{Cov}

\DeclareMathOperator{\op}{op}

\def\m0{\mathbf{0}}

\def \bX {\boldsymbol{X}}

\def \bOmega {\boldsymbol{\Omega}}

\newtheorem{theorem}{Theorem}

\newtheorem{proposition}{Proposition}
\newtheorem{remark}{Remark}

\newtheorem{lemma}{Lemma}
\newtheorem{corollary}{Corollary}

\theoremstyle{remark}
\newtheorem{exmp}{Example}

\newtheorem{setting}{Setting}

\newcommand\blfootnote[1]{%
  \begingroup
  \renewcommand\thefootnote{}\footnote{#1}%
  \addtocounter{footnote}{-1}%
  \endgroup
}
\newcommand{\independent}{\mathrel{\perp\!\!\!\perp}}
\renewcommand{\ldots}{\cdots}
\renewcommand{\dots}{\cdots}

\newcommand{\blind}{1}

\allowdisplaybreaks

\def\spacingset#1{\renewcommand{\baselinestretch}%
{#1}\small\normalsize}
\spacingset{1}

\begin{document}

\if1\blind
{
\title{Asymptotic FDR Control with Model-X Knockoffs: Is Moments Matching Sufficient?%
\thanks{This work was supported in part by NSF grants EF-2125142 
and DMS-2324490.}}

\author{Yingying Fan$^1$ \and 
Lan Gao$^2$ \and
Jinchi Lv$^1$ \and 
Xiaocong Xu$^1$}

\date{
University of Southern California$^1$ and University of Tennessee$^2$\\
\bigskip
\bigskip
February 9, 2025
}

\maketitle
} \fi

\begin{abstract}
We propose a unified theoretical framework for studying the robustness of the model-X knockoffs framework by investigating the asymptotic false discovery rate (FDR) control of the practically implemented approximate knockoffs procedure. This procedure deviates from the model-X knockoffs framework by substituting the true covariate distribution with a user-specified distribution that can be learned using in-sample observations. By replacing the distributional exchangeability condition of the model-X knockoff variables with three conditions on the approximate knockoff statistics, we establish that the approximate knockoffs procedure achieves the asymptotic FDR control. Using our unified framework, we further prove that an arguably most popularly used knockoff variable generation method—the Gaussian knockoffs generator based on the first two moments matching—achieves the asymptotic FDR control when the two-moment-based knockoff statistics are employed in the knockoffs inference procedure. For the first time in the literature, our theoretical results justify formally the effectiveness and robustness of the Gaussian knockoffs generator. Simulation and real data examples are conducted to validate the theoretical findings.
\end{abstract}


\blfootnote{}

\spacingset{1.9}
\section{Introduction} \label{new.sec.intro}

The model-X knockoffs \citep{candes2018panning} is a flexible framework for high-dimensional variable selection with theoretically guaranteed finite-sample false discovery rate (FDR) control bypassing the use of p-values. Without the need to know or impose assumptions on how the response variable $Y \in \mathbb{R}$ depends on a high-dimensional feature vector $X = (X_1, \ldots, X_p)^\top \in \mathbb{R}^p$, model-X knockoffs provides a general framework to select relevant features in the set $\mathcal H_1\subset\{1,\cdots, p\}$, where for each $j\in \mathcal H_1^c := \{1,\cdots, p\}\setminus \mathcal H_1$, we have $X_j\independent Y \vert X_{-j}$. Here, $X_{-j}$ is the subvector of $X$ with the $j$th component removed.  The core idea is to construct the knockoff variables $\widetilde X\in\mathbb R^p$ that satisfy two conditions: i) exchangeable in distribution with the original features $X$, and  ii) independent of response $Y$ conditional on the original features $X$.  These knockoff variables act as a control group, allowing researchers to identify the relevant original features in $\mathcal H_1$ related to the response. 

The exchangeability condition is key to controlling the FDR. Formally, it is stated as: for any subset $S \subset \{1, \ldots, p\}$, the joint distribution of $(X, \tilde{X})$ satisfies
\begin{align}\label{eq:exchangable} 
(X^\top, \tilde{X}^\top)_{\text{swap}(S)} \overset{d}{=} (X^\top, \tilde{X}^\top), 
\end{align}
where $(X^\top, \tilde{X}^\top)_{\text{swap}(S)}$ represents the vector obtained by swapping the components $X_j$ and $\tilde{X}_j$ for each $j \in S$ and $\overset{d}{=}$ stands for equal in distribution. For the ease of reference, in this paper we term the knockoff variables that satisfy the aforementioned two conditions as the \textit{perfect} knockoff variables. In principle, perfect knockoff variables can be constructed when the feature distribution 
$F$ of
$X$ is known using, for example, a generic algorithm outlined in \cite{candes2018panning} (see Algorithm 1 therein).

Since the true covariate distribution $F$ is generally unknown, in practice, a working distribution $\widehat F$ is almost always used to replace the true distribution $F$ when generating the knockoff variables in implementing the model-X knockoffs. For differentiation, we name the resulting model-X knockoff variables $\widehat X$ as the \textit{approximate knockoff variables} since they correspond to the working distribution $\widehat F$ rather than the true distribution $F$, and the corresponding knockoffs inference method as the \textit{approximate knockoffs procedure}. Because of the discrepancy of $\widehat F$ from $F$, the distributional exchangeability condition in \eqref{eq:exchangable} no longer holds for $\widehat X$. Consequently, we lose the theoretical guarantee of finite-sample FDR control as shown in \cite{candes2018panning}. This raises the first natural question:
\begin{itemize}
    \item When can the approximate knockoffs procedure remain robust, allowing for (asymptotic) FDR control under a misspecified feature distribution $\widehat F$?
\end{itemize}

An important goal of this paper is to provide a unified theoretical framework to address the above question. We identify three sufficient conditions, under which the approximate knockoffs procedure can achieve the asymptotic FDR control; see Theorem \ref{thm:general_framwork} for the formal statements of conditions and results. Such results equip researchers with a unified tool for verifying whether some popularly used $\widehat F$ in practice can yield valid asymptotic FDR control. Our unified theoretical framework for the asymptotic FDR control by the approximate knockoffs procedure is new to the literature. 

In implementation, an arguably most popularly used method for generating the knockoff variables is the \textit{Gaussian knockoffs generator} \citep{candes2018panning, TSKI2025,DeepLink2021,Kimi2020, DeepPink2018}, which we formally introduce as follows.   Assume that we have access to the first two moments of $X$: the mean vector $\bm{\mu} \in \R^p$ (without loss of generality, assume $\bm{\mu} = 0_p$) and the covariance matrix $\Sigma_X \in \R^{p \times p}$. The Gaussian knockoffs generator pretends that $X$ followed Gaussian distribution $\mathcal{N}(0_p, \Sigma_X)$, and generates knockoff variables as 
\begin{align}\label{eq:approx_knock}
	\widehat{X} = (I_p - \mathrm{diag}(\bm{r})\Sigma_X^{-1})X + (2\mathrm{diag}(\bm{r}) - \mathrm{diag}(\bm{r})\Sigma_X^{-1}\mathrm{diag}(\bm{r}))^{1/2}\mathsf{Z},
\end{align}
where $\bm{r} \in \mathbb{R}^p$ is chosen to ensure the positive definiteness of  $2\mathrm{diag}(\bm{r}) - \mathrm{diag}(\bm{r})\Sigma_X^{-1}\mathrm{diag}(\bm{r})$, and $\mathsf{Z} \in \R^p$ is a standard Gaussian random vector independent of all other variables. For the ease of presentation, we refer   to $\widehat X$ in \eqref{eq:approx_knock} as the Gaussian knockoff variables. We emphasize that the true distribution $F$ here can be \textit{different} from Gaussian distribution $\mathcal{N}(0_p, \Sigma_X)$. 

To demonstrate the practical performance of the Gaussian knockoffs generator, we apply the procedure to a real data application: modeling the protease inhibitor (PI)-type drug resistance of HIV-1 based on genetic mutation data, which was studied previously in \cite{rhee2006genotypic} and \cite{barber2015controlling}. Here, $X$ is categorical data, which is evidently non-Gaussian. Despite the distributional misspecification, the results indicate that, in almost all cases, the FDR is controlled near the prespecified level $q = 0.2$, as summarized in Table \ref{tab:realdata}. Additional details are provided in Section \ref{sec:real_data_setting}.

\begin{table}[t]
\caption{The empirical FDR and power results from applying the approximate knockoffs procedure to the HIV data described in Section \ref{sec:real_data_setting}. The Gaussian knockoffs generator is used to generate the approximate knockoff variables. The target FDR level is $q = 0.2$. We subsample data of size $200$ from the full data set of size $n$ ranging from $328$ to $842$, and repeat the procedure $100$ times to calculate the FDR and power.
}
\label{tab:realdata}
\centering
\renewcommand{\arraystretch}{0.6} 
\setlength{\tabcolsep}{8pt} 
\begin{tabular}{ll|ccccccc}
\hline
& & APV & ATV & IDV & LPV & NFV & RTV & SQV \\ \hline
\multirow{2}{*}{MC} & FDR   & 0.174 & 0.177 & 0.215 & 0.244 & 0.204 & 0.233 & 0.196 \\
                    & Power & 0.351 & 0.322 & 0.296 & 0.367 & 0.293 & 0.341 & 0.311 \\ \hline
\multirow{2}{*}{DL} & FDR   & 0.09     &   0.137    &  0.109     &  0.159     &    0.114   &  0.010     &  0.151     \\
                    & Power &  0.154     &    0.357   &   0.100    &   0.266    &   0.133    &   0.136    &    0.140   \\ \hline
\multirow{2}{*}{DC} & FDR   & 0.125     &   0.198    &    0.173   &  0.212     &     0.158   &    0.176   &   0.139    \\
                    & Power &  0.499     &    0.637   &     0.441  &    0.523   &   0.446    &  0.480    &   0.411    \\ \hline
\end{tabular}
\end{table}

It is seen that the Gaussian knockoffs generator relies on the first two moments matching with the original covariates $X$. We note that the deep knockoffs method in \cite{romano2020deep} falls along the same idea of moments matching, with the major difference that it also considers higher-order moments matching in generating the knockoff variables. The popularity and practical success of the Gaussian knockoffs and deep knockoffs motivate us to raise the second question: 
\begin{itemize}
    \item Is moments matching sufficient for generating the knockoff variables for achieving the asymptotic FDR control?
\end{itemize}

Another important contribution of this paper is to provide an affirmative answer to the second question above. By verifying the sufficient conditions provided in our Theorem \ref{thm:general_framwork}, we formally establish that Gaussian knockoffs can provide the asymptotic FDR guarantee for some popularly studied and used choices of knockoff statistics, including the marginal correlation knockoff statistics and regression coefficient difference knockoff statistics. The intuition behind such robustness is that these knockoff statistics are constructed based on the first two moments of the data matrix, and thus, the Gaussian knockoff variables constructed based on the first two moments matching are sufficient for providing asymptotic FDR control. For knockoff statistics constructed based on higher-order data moments, we conjecture that one needs to generate the knockoff variables via higher moments matching for achieving the asymptotic FDR control. Since such extension is highly nontrivial, we leave it for future investigation.

Our work in this paper falls into the general framework of robustness analysis of model-X knockoffs with respect to the deviation of working distribution $\widehat F$ from true distribution $F$. Prior to our work, efforts have been made toward addressing the robustness concern  \citep{FDLL2020-rank,FLSU2020-ipad}. Notably, \cite{Barber2020} relaxed the distributional exchangeability condition \eqref{eq:exchangable} and established an FDR upper bound. Their results show that the FDR inflation (caused by using $\widehat F$) depends on the Kullback--Leibler (KL) divergence of the conditional distributions derived from $F$ and $\widehat F$, and can asymptotically vanish provided that an independent large enough pre-training data set is available for very accurately estimating $F$. Recently, \cite{fan2023ark} exploited a coupling idea to prove the asymptotic FDR control for the approximate knockoffs procedure. Their results show that the procedure remains robust for a broader class of $\widehat F$ beyond those considered by \cite{Barber2020}. However, \textit{neither} of these studies applies directly to the Gaussian knockoffs to explain its strong practical performance observed in Table \ref{tab:realdata}. The connections and differences between our work and these two prior studies will be further discussed as the context becomes more clear in subsequent sections.

The rest of the paper is organized as follows. Section~\ref{sec:general_framework} introduces our general framework for the asymptotic FDR control by the approximate knockoffs procedure and the Gaussian knockoffs generator. We delve into specific constructions of knockoff statistics based on the first two moments matching and verify their asymptotic FDR control property in  Section~\ref{sec:examples}. Section~\ref{sec:simulation} presents several simulation and real data examples. We discuss some implications and extensions of our work in Section \ref{new.sec.disc}. All technical proofs and detailed derivations are provided in the Supplementary Material.

\subsection{Notation} \label{new.sec1.1}

For any positive integer $m \in \mathbb{Z}^+$, let $[m]\equiv \{1,\ldots,n\}$. For $a,b \in \R$, denote by $a\vee b\equiv \max\{a,b\}$ and $a\wedge b\equiv\min\{a,b\}$.
For any vector $x$, let $\pnorm{x}{\ell}$ denote its $\ell$-norm $(0\leq \ell\leq \infty)$, and we simply write $\pnorm{x}{}\equiv\pnorm{x}{2}$.  For any matrix $M$, let $\pnorm{M}{\op}$ be its spectral norm. $I_m$ is reserved for an $m\times m$ identity matrix. Throughout this paper, all asymptotic notation is considered in the limit as $n \to \infty$ and $p$ diverges with $n$. 
We write $a_n \ll b_n$ or $a_n = \mathfrak{o}(b_n)$ if $a_n/b_n \to 0$, and $a_n \gg b_n$ if $a_n/b_n \to \infty$. For random sequences, $a_n = \mathfrak{o}_\mathbf{P}(b_n)$ means $a_n/b_n \overset{p}{\to} 0$. Denote by $C$ an absolute constant, varying across lines unless specified. We write $a_n \lesssim b_n$ ($a_n = \bigo(b_n)$) and $a_n \gtrsim b_n$ ($a_n = \Omega(b_n)$) if $a_n \leq C b_n$ and $a_n \geq C b_n$, respectively, and $a_n \asymp b_n$ ($a_n = \Theta(b_n)$) if both hold (with different values of $C$).

\section{A general framework} \label{sec:general_framework}

\subsection{Model setting and approximate knockoffs inference} \label{sec:independent case procedure}

We assume that $(X,Y)$ follows the following flexible nonparametric regression model 
\begin{align}\label{def:general_model}
	Y = \mathcal{F}(X_{\mathcal{H}_1} ) + \xi,
\end{align}
where $X_{\mathcal{H}_1}$ represents the subvector consisting of important features in $\mathcal{H}_1\subset \{1,\dots, p\}$ that contribute to response $Y$, $\mathcal F$ is an unknown function, and $\xi$ is the model noise that is independent of $X$. It is seen that $ X_j \perp\!\!\!\perp Y \,|\, X_{-j}$ for each $j\in \mathcal H_0:= \{1,\cdots,p\}\setminus \mathcal H_1$, and we name such features as the null (noise) features.  To ensure the model identifiability, we follow \cite{candes2018panning} in assuming that $\mathcal{H}_1$ exists and is unique. This model is commonly used for feature selection tasks, focusing on selecting the relevant feature set $\mathcal H_1$.

Assume that we are given $n$ independent and identically distributed (i.i.d.) observations ${(\bm{X}_{i,1}, \ldots, \bm{X}_{i,p}, \bm{Y}_i)}_{i \in [n]}$ drawn from the joint distribution of $(X, Y)$.  Let $\widehat{\mathcal{S}}\subset\{1,\cdots, p\}$ be the set of features selected by some algorithm applied to ${(\bm{X}_{i,1}, \ldots, \bm{X}_{i,p}, \bm{Y}_i)}_{i \in [n]}$. We say that $\widehat{\mathcal{S}}$ achieves  the FDR control at some prespecified level $q\in (0,1)$ if 
\begin{align*}
	\mathrm{FDR} = \mathbb{E}\left[ \frac{\# \big\{ j : j \in \widehat{\mathcal{S}} \cap \mathcal{H}_0 \big\}}{ \# \big\{ j : j \in \widehat{\mathcal{S}} \big\} \vee 1} \right] \leq q.
\end{align*}

For convenience, let $\bm{X} = (\bm{X}_{i,j})_{i \in [n], j \in [p]} \in \mathbb{R}^{n \times p}$ denote the data matrix containing all the feature vectors, and let $\bm{Y} = (\bm{Y}_1, \ldots, \bm{Y}_n)^\top \in \mathbb{R}^n$ represent the response vector. Additionally, let $\bm{\xi} = (\bm{\xi}_1,\ldots,\bm{\xi}_n )^\top \in \R^n $ denote the model noise vector, which collects the noise associated with each sample.

We follow \cite{fan2023ark} to introduce the approximate knockoffs procedure, which consists of three main steps:
\begin{enumerate}
  \item[(1)] {\it Generation of approximate knockoff variables:} Specify or estimate (using either in-sample or out-of-sample data) a working distribution $\widehat F$ for covariate vector $X$. Pretend that $\widehat F$ were the true distribution of $X$, and generate the approximate knockoffs data matrix $\widehat \bX$ using $\widehat F$. An example of a generic algorithm that can be used is the sequential conditional independent pairs algorithm in \cite{candes2018panning}.   
  \item[(2)] {\it Construction of knockoff statistics:} Treating $\widehat{\bX}$ as though it were the perfect knockoff variable matrix, compute knockoff statistics analogously to the perfect knockoffs procedure in \cite{candes2018panning}. Specifically,  calculate the feature importance statistics as $(Q_1, \ldots, Q_p, \widehat{Q}_1, \ldots, \widehat{Q}_p ) = \mathcal{T}(\bm{X}, \widehat{\bm{X}}, \bm{Y})$, where $\mathcal{T}$ is a measurable function of $(\bm{X}, \widehat{\bm{X}}, \bm{Y})$, and $Q_j$ and $\widehat{Q}_j$ represent the importance of the $j$th feature and its knockoff counterpart, respectively. The approximate knockoff statistic $\widehat{\bm{W}}_j$ for each  original feature is then defined as $\widehat{\bm{W}}_j = f_j(Q_j, \widehat{Q}_j)$ for some antisymmetric function $f_j$.
  
  \item[(3)] {\it Feature selection:} For a prespecified FDR level $q \in (0,1)$, we select important features as $\widehat{\mathcal{S}} \equiv \{ j \in [p] : \widehat{\bm{W}}_j \geq T_q \}$,
  where the threshold $T_q$ for the FDR control is given by
  \begin{align*}
  	T_q \equiv \min \left\{ t \in \big\{ \abs{\widehat{\bm{W}}_1}, \ldots, \abs{\widehat{\bm{W}}_p} \big\} : \frac{\# \big\{ j : \widehat{\bm{W}}_j \leq -t \big\}}{\# \big\{ j : \widehat{\bm{W}}_j \geq t \big\} \vee 1} \leq q \right\}.
  \end{align*}
\end{enumerate}

It is seen that the only distinction between the approximate knockoffs procedure and the model-X knockoffs framework in \cite{candes2018panning} lies in how $\widehat{\bX}$ is generated. Indeed, when the true distribution $F$ is used in Step (1) above, we recover the (perfect) model-X knockoffs procedure.  For the ease of presentation, from now on we denote by $\tilde{\bm X}\in \mathbb R^{n\times p}$, $\{\tilde{\bm W}_j\}_{j=1,\dots, p}$, and $\tilde T_q$ the perfect knockoff variable matrix, knockoff statistics, and knockoff threshold using the true distribution $F$, respectively.  In practice, the true covariate distribution $F$ is generally unknown. We have discussed in the Introduction a most popularly used algorithm for practically generating knockoff variables, the Gaussian knockoffs generator. An important goal of our study is to provide a general framework for verifying the robustness of model-X knockoffs in terms of FDR control with respect to various practical choice of $\widehat F$, and we achieve this goal by studying the FDR property of the approximate knockoffs procedure.   Our general framework is presented in the next subsection.

\subsection{General framework for asymptotic FDR control} \label{new.sec2.2}

A critical condition for (perfect) model-X knockoffs to achieve finite-sample FDR control is the distributional exchangeability condition presented in \eqref{eq:exchangable}. Under this condition, it can be shown that the resulting perfect knockoff statistics $\widetilde{\bm W}_j$ for null features enjoy the following distributional symmetry property
\begin{align}\label{eq:W-j-exchangeability}
    \sum_{j\in \mathcal H_0}\mathbf 1\{\widetilde{\bm W}_j \leq -t\}\overset{d}{=} \sum_{j\in \mathcal H_0}\mathbf 1\{\widetilde{\bm W}_j \geq t\}
\end{align}
for any $t\in \mathbb R$. Such a property is key to establishing the finite-sample FDR control result for the perfect model-X knockoffs. It ensures that the number of false positive selections  $\sum_{j\in \mathcal H_0}\mathbf 1\{\widetilde{\bm W}_j \geq \tilde T_q\}$ can be approximated by $ \sum_{j\in \mathcal H_0}\mathbf 1\{\widetilde{\bm W}_j \leq -\tilde{T}_q\}$, a quantity that is more practically trackable. This further ensures that  $\frac{\sum_{j\in \mathcal H_0}\mathbf 1\{\widetilde{\bm W}_j \geq \tilde T_q\}}{\sum_{j\in \mathcal H_0}\mathbf 1\{\widetilde{\bm W}_j \leq -\tilde{T}_q\}}$ is close to $1$, a critical result in proving the FDR control by model-X knockoffs; see \cite{candes2018panning} for more details.

With the approximate knockoff statistics $\widehat{\bm W}_j$'s introduced in the last subsection, we lose the distributional symmetry property as in \eqref{eq:W-j-exchangeability}. The theorem below guarantees that under the three sufficient conditions  on null $\widehat{\bm W}_j$'s, the approximate knockoffs procedure enjoys the asymptotic FDR control.    

\begin{theorem}\label{thm:general_framwork}
	Fix $q \in (0,1)$ and assume that the following conditions hold for some $\alpha_n > 0$:
	\begin{enumerate}
	\item (Asymptotic approximate symmetry for $\widehat{\bm{W}}_j$):
		\begin{align*}
			\sup_{t \in (0, \alpha_n) }\biggabs{\frac{\sum_{j \in \mathcal{H}_0} \Prob \big( \widehat{\bm{W}}_j \geq t  \big)}{\sum_{j \in \mathcal{H}_0} \Prob \big(\widehat{\bm{W}}_j \leq -t   \big)} - 1}= \mathfrak{o}(1).
		\end{align*}
		\item (Approximation for indicator functions):
		\begin{align*}
        \sup_{t \in (0, \alpha_n) }\Bigg\{ \biggabs{\frac{ \sum_{j \in \mathcal{H}_0}  \bm{1}\{\widehat{\bm{W}}_j  \geq t\}  }{\sum _{j \in \mathcal{H}_0}  \Prob (\widehat{\bm{W}}_j  \geq t)  } - 1 } \vee \biggabs{\frac{ \sum_{j \in \mathcal{H}_0}  \bm{1}\{\widehat{\bm{W}}_j  \leq -t\}  }{\sum _{j \in \mathcal{H}_0}  \Prob (\widehat{\bm{W}}_j  \leq -t)  } - 1 } \Bigg\} = \mathfrak{o}_\mathbf{P}(1).
    \end{align*}	
		\item (Localization of $T_q$):
		$	\Prob \big(T_q > \alpha_n  \big) = \mathfrak{o}(1).$
	\end{enumerate}
	Then we have
	\begin{align*}
		\limsup_{n \to \infty} \mathrm{FDR} \leq q.
	\end{align*}
\end{theorem}

Theorem \ref{thm:general_framwork} above provides us a unified framework for verifying the robustness of practically used approximate knockoffs procedure. Any specification of $\widehat F$ that can yield $\widehat{\bm W}_j$'s satisfying the three conditions above guarantees asymptotic FDR control. It is worth mentioning that the recent study in \cite{fan2023ark} uses the coupling idea to verify the three conditions above for certain specification of $\widehat F$; we will provide more detailed discussions on this in the next subsection.

Conditions (2) and (3) in Theorem \ref{thm:general_framwork} are not entirely new; they have been previously established in \cite{fan2023ark} for perfect knockoff statistics $\widetilde{\bm{W}}_j$'s resulting from perfect knockoff variables $\widetilde{\bX}$. While the technical details may differ, these conditions share essential conceptual similarities with the earlier work. In the subsequent section, we will demonstrate how these conditions can be verified for specific examples of $\widehat{\bm{W}}_j$'s.

Condition (1), on the other hand, is novel. It requires that the knockoff statistics $\widehat{\bm{W}}_j$ exhibit asymptotic approximate symmetry over a suitable range of $t$ values for null features. For perfect knockoff statistics $\widetilde{\bm{W}}_j$ that satisfy the so-called flip-sign property \citep{candes2018panning}, this condition is inherently satisfied. Using this property, we may construct symmetric random variables $\mathfrak{W}_j$ such that $\mathfrak{W}_j \overset{d}{\approx} \widehat{\bm{W}}_j$, and consequently, $\Prob(\widehat{\bm{W}}_j \geq t ) \approx  \Prob (\mathfrak{W}_j \geq t) =  \Prob (\mathfrak{W}_j \leq -t) \approx \Prob(\widehat{\bm{W}}_j \leq -t )$, thereby confirming condition (1). Then, using condition (2) on the approximation for indicator functions, we can prove that with high probability, null $\widehat{\bm W}_j$'s satisfy that
\begin{equation}\label{What-approx-sym}
    \frac{ \sum_{j\in \mathcal H_0}\mathbf 1\{\widehat{\bm W}_j \leq -t\}}{ \sum_{j\in \mathcal H_0}\mathbf 1\{\widehat{\bm W}_j \geq t\}} \approx 1
\end{equation}
for a suitable range of $t$ values. It is seen that the above result is parallel to \eqref{eq:W-j-exchangeability}.  Using this result, we can prove the asymptotic FDR control for the approximate knockoffs procedure.

We end this section with the proof of Theorem \ref{thm:general_framwork}. In the subsequent section, we will provide concrete examples of the approximate knockoff statistics and demonstrate how the conditions of Theorem \ref{thm:general_framwork} can be verified.

\noindent\textit{Proof of Theorem \ref{thm:general_framwork}}. 
    Let us define $\mathcal{B}_1 = \{T_q \leq \alpha_n  \}$ and 
    \begin{align*}
        \mathcal{B}_{2,\epsilon} \equiv \Bigg\{   \sup_{t \in (0, \alpha_n) }\biggabs{\frac{ \sum_{j \in \mathcal{H}_0}  \bm{1}\{\#\widehat{\bm{W}}_j  \geq t\}  }{\sum _{j \in \mathcal{H}_0}  \Prob (\#\widehat{\bm{W}}_j  \geq t)  } - 1 }\leq \epsilon \ \text{ for } \# \in \{ +,- \}  \Bigg\}
    \end{align*}
    for $\epsilon > 0$. Conditions (2) and (3) above ensure that $ \Prob (\mathcal{B}_1^c  ) \to 0   $  and $ \Prob(\mathcal{B}_{2, \epsilon}^c  ) \to 0 $ for each $\epsilon > 0$.  In addition, it holds naturally that $0 \leq  \mathrm{FDP} \leq 1  $. Then it follows that
    \begin{align*}
        \mathrm{FDR}  \leq \E \left[ \frac { \sum_{j \in \mathcal{H}_0} \bm{1} \{ \widehat{\bm{W}}_j \geq T_q \} } { 1 \vee \sum_{ j  \in [p]} \bm{1} \{\widehat{\bm{W}}_j \geq T_q\}  } \cdot  \bm{1} \{\mathcal{B}_1\}\bm{1} \{\mathcal{B}_{2, \epsilon}\} \right] +  \mathfrak{o}(1).
    \end{align*}
    From the definition of $T_q$, we can deduce that 
    \begin{align*}
        &\frac { \sum_{j \in \mathcal{H}_0} \bm{1} \{ \widehat{\bm{W}}_j \geq T_q \} } { 1 \vee \sum_{ j  \in [p]} \bm{1} \{\widehat{\bm{W}}_j \geq T_q\}  }  \cdot \bm{1} \{\mathcal{B}_1 \}\bm{1} \{\mathcal{B}_{2, \epsilon}\} \\
        &=  \frac { \sum_{j \in \mathcal{H}_0} \bm{1} \{ \widehat{\bm{W}}_j \geq T_q \} } { \sum_{ j  \in \mathcal{H}_0} \bm{1} \{\widehat{\bm{W}}_j \leq -T_q\}  }  \cdot \frac {\sum_{j \in \mathcal{H}_0} \bm{1} \{ \widehat{\bm{W}}_j \leq -T_q \}  } { 1 \vee \sum_{ j  \in [p]} \bm{1} \{\widehat{\bm{W}}_j \geq T_q\}  }  \cdot \bm{1} \{\mathcal{B}_1\}\bm{1} \{\mathcal{B}_{2, \epsilon}\}\\
        &\leq q \cdot  \frac { \sum_{j \in \mathcal{H}_0} \bm{1} \{ \widehat{\bm{W}}_j \geq T_q \} } { \sum_{ j  \in \mathcal{H}_0} \bm{1} \{\widehat{\bm{W}}_j \leq -T_q\}  } \cdot \bm{1} \{\mathcal{B}_1\}\bm{1} \{\mathcal{B}_{2, \epsilon}\}.
    \end{align*}
    	Furthermore, it holds on event $\mathcal{B}_1 \cap \mathcal{B}_{2, \epsilon} $ that 
\begin{align*}
    & \frac { \sum_{j \in \mathcal{H}_0} \bm{1} \{ \widehat{\bm{W}}_j \geq T_q \} } { \sum_{ j  \in \mathcal{H}_0} \bm{1} \{\widehat{\bm{W}}_j \leq -T_q\}  } \leq  \sup_{t \in (0,\alpha_n )} \frac { \sum_{j \in \mathcal{H}_0} \bm{1} \{ \widehat{\bm{W}}_j \geq t \} } { \sum_{ j  \in \mathcal{H}_0} \bm{1} \{\widehat{\bm{W}}_j \leq -t\}  } \\
    &\leq \bigg(\frac{1 + \epsilon}{1 - \epsilon}\bigg)  \cdot \sup_{t \in (0, \alpha_n )} \frac { \sum_{j \in \mathcal{H}_0} \Prob ( \widehat{\bm{W}}_j \geq t ) } { \sum_{ j  \in \mathcal{H}_0} \Prob (\widehat{\bm{W}}_j \leq -t)  }
    \leq \bigg(\frac{1 + \epsilon}{1 - \epsilon} \bigg) \cdot (1 + \mathfrak{o}(1)).
\end{align*}
Consequently, for any $\epsilon > 0$ we can obtain that  $\mathrm{FDR} \leq q  \cdot \big(\frac{1 + \epsilon}{1 - \epsilon} \big) + \mathfrak{o}(1)$, 
which completes the proof of Theorem \ref{thm:general_framwork}.

\subsection{Approximate knockoffs via moments matching} \label{subsec:gaussian-knockoff}

As discussed in the Introduction, the moments matching method has been popularly used in the literature to practically generate knockoff variables. In particular, the Gausssian knockoffs generator defined around \eqref{eq:approx_knock} is arguably the most popularly implemented method for knockoff variables generation.  As also demonstrated in the Introduction, the Gaussian knockoffs appear to achieve the FDR control in the real application even when the features are binary, a significant deviation from Gaussian distribution. We will provide theoretical justification for its asymptotic FDR control property in the next section. Our results are the \textit{first} in the literature to theoretically verify the validity of moments matching methods for generating model-X knockoff variables in achieving asymptotic FDR control.    

A notable property of the Gaussian construction in \eqref{eq:approx_knock} is the exchangeability of the covariance structure. Specifically, for any permutation matrix $\bm{P} \in \R^{2p \times 2p}$ encoding the swaps in the exchangeability condition in (\ref{eq:exchangable}), the covariance matrix satisfies that 
\begin{align}\label{eq:cov_exchange}
	\bm{P}\cov \big[ (X^\top, \widehat{X}^\top )^\top \big]\bm{P} = \cov \big[ (X^\top, \widehat{X}^\top )^\top \big].
\end{align}
It is worth noting that the verification of condition (1) in Theorem \ref{thm:general_framwork} in this paper relies critically on the exchangeability property (\ref{eq:cov_exchange}). 

We are now ready to discuss the major distinctions of our study with the work by \cite{fan2023ark}.  Our Theorem \ref{thm:general_framwork} provides a unified framework for verifying the robustness of the knockoffs procedure. At a high level, \cite{fan2023ark} proved the robustness of the knockoffs procedure through verifying conditions (1)--(3) in our Theorem \ref{thm:general_framwork} by using the proof technique of coupling. The salient  idea is to couple the practically constructed approximate knockoff variable matrix $\widehat{\bm X}$ with a realization of the perfect knockoffs data matrix $\widetilde {\bm X}$; if $\widehat{\bm X}$ and the coupled $\widetilde {\bm X}$  can have asymptotically vanishing distance in realization with high probability, then the latter can be used as  a bridge to verify conditions (1)--(3) in Theorem \ref{thm:general_framwork} and to further prove the asymptotic FDR control using $\widehat{\bm X}$. The existence of such coupled  $\widetilde {\bm X}$ is the key for the coupling idea to work. 

Although the coupling idea can be applied to general distribution $F$, it can only be used to prove  certain type of deviation of $\widehat F$ from $F$. Indeed, the success of Gaussian knockoffs in the real data example mentioned in the Introduction \textit{cannot} be verified using the coupling idea in \cite{fan2023ark}. To formally characterize this point, and to better illustrate the significance of our study in this paper, we provide two examples below. We note that the first example is designed to mimic the real data example in the Introduction.

\begin{exmp}\label{exmp:discrete_covariate}
    Consider $X_1, \ldots, X_p$ as i.i.d. Rademacher random variables. In this case, the perfect knockoff random variables $\tilde{X}_1, \ldots, \tilde{X}_p$ are also i.i.d. Rademacher random variables. Let $\widehat{X}_1, \ldots, \widehat{X}_p \overset{\text{i.i.d.}}{\sim} \mathcal{N}(0, 1)$, and denote by $\{\bm{\widehat{X}}_{i,\cdot}\}_{i \in [n]} \overset{\text{i.i.d.}}{\sim} \widehat{\mu}$ and $\{\bm{\tilde{X}}_{i,\cdot}\}_{i \in [n]} \overset{\text{i.i.d.}}{\sim} \tilde{\mu}$. Then there exists some universal constant $c > 0$ such that 
    \begin{align*}
      \inf_{\gamma \in \Gamma( \widehat{\mu}, \tilde{\mu}) }  \Prob_{ \{(\bm{\widehat{X}}_{i,\cdot},\bm{\tilde{X}}_{i,\cdot} )\}_{i \in [n]} \overset{\text{i.i.d.}}{\sim} \gamma } \Big(  \max_{j \in [p]} n^{-1/2} \pnorm{ \bm{\widehat{X}}_{\cdot,j} - \bm{\tilde{X}}_{\cdot,j} }{2}  \geq c \Big) \geq c,
    \end{align*}
    where $ \Gamma( \widehat{\mu}, \tilde{\mu})$ is the set consisting of all couplings of $\widehat{\mu}$ and $\tilde{\mu}$. Since $c$ is a constant that does \textit{not} asymptotically vanish with sample size, the above result indicates that there does \textit{not} exist an ideal knockoff variable matrix that can be coupled with Gaussian knockoffs $\widehat{\bm X}$ with the required accuracy by the coupling idea for ensuring the asymptotic FDR control. On the contrary, our Corollary \ref{cor:specialdis} (to be presented in the next section), proved by verifying the three conditions in Theorem \ref{thm:general_framwork}, shows that Gaussian knockoffs can indeed achieve the asymptotic FDR control under some regularity conditions.   
\end{exmp}

\begin{exmp}\label{ex2-t-distr}
Consider $X_1, \ldots, X_p$ as i.i.d. random variables following the $t$-distribution with $\mathfrak{q}$ degrees of freedom, normalized to have variance $1$. Let $\widehat{X}_1, \ldots, \widehat{X}_p \overset{\text{i.i.d.}}{\sim} \mathcal{N}(0, 1)$. To ensure effective coupling, the theory in \cite{fan2023ark} requires $\mathfrak{q}$ to diverge to infinity at certain rate, cf. \cite[Proposition 2]{fan2023ark}. In contrast, our Corollary \ref{cor:specialdis} provides the asymptotic FDR control result for finite $\mathfrak{q} \geq 3$, significantly relaxing the requirement on the tail behavior of the $t$-distribution. 

\end{exmp}

We end this section by emphasizing that while we adopt the Gaussian knockoffs via two moments matching for simplicity, our general framework is \textit{not} constrained to this particular choice. Our theoretical analysis does \textit{not} rely on the Gaussianity of knockoffs in \eqref{eq:approx_knock}; random vector
$\mathsf{Z}$ in \eqref{eq:approx_knock} can be replaced with one from any distribution that meets the required moment conditions.
This approach eliminates the need for a fully specified feature distribution, streamlining the knockoffs generation process.

\section{Two-moment-based knockoff statistics} \label{sec:examples}

We consider three concrete constructions of the knockoff statistics in this section and verify that they can satisfy conditions (1)--(3) in Theorem \ref{thm:general_framwork} and thus achieves the asymptotic FDR control. It is important to note that in all examples in this section, i) the Gaussian knockoffs generator \eqref{eq:approx_knock} is used to generate $\widehat{\bm{X}}$, and ii) knockoff statistics $\widehat{\bm W}_j$'s depend only on the first two moments of data matrix.  All the notation will be local in each subsection.

\subsection{Marginal correlation knockoff statistics 
} \label{sec:marginial_corr}

We begin with a warm-up case when $X$ consists of independent features. The dependent case is considered at the end of this subsection. Let $\sigma_j \equiv \var^{1/2} (X_j )$ for each $j \in [p]$. 
For the independent case, it holds that $\Sigma_X = \mathrm{diag}( \{ \sigma_j^2\}_{j \in [p]} )$, and we may simply set $\bm{r} = (\sigma_1^2,\ldots,\sigma_p^2 )^\top$ in \eqref{eq:approx_knock}. In this setting, the knockoff variable matrix $\widehat{\bm{X}}$ is given by
$
	\widehat{\bm{X}} =  \widehat{\bm{Q}}\Sigma_X^{1/2},
$
where $\widehat{\bm{Q}} \in \R^{n \times p}$ consists of i.i.d. $\mathcal{N}(0,1)$ entries that are independent of $\bm{X}$ and $\bm{Y}$.

 Given the original data matrix $\bm{X}$ and the knockoff variable  matrix $\widehat{\bm{X}}$, the knockoff statistics based on the marginal correlation difference are defined as 
    \begin{align*}
        \widehat{\bm{W}}_j \equiv  \frac{\abs{ \bm{X}_{\cdot,j}^\top  \bm{Y} } }{\pnorm{\bm{Y}}{}} - \frac{\abs{ \widehat{\bm{X}}_{\cdot,j}^\top  \bm{Y} } }{\pnorm{\bm{Y}}{}}  , \quad j \in [p].
    \end{align*}
It measures the marginal linear dependence of each feature on the response variable. The normalization by $\pnorm{\bm{Y}}{}$ ensures the scale invariance, making these statistics robust to variations in the response magnitude. The marginal correlation statistics have also been studied in \cite{Barber2020, fan2023ark, Reconciling2024}.

We will work with the following moment conditions:
\begin{enumerate}[label=(C\arabic*), ref=C\arabic*]
	\item \label{cond:C1} The components of $\bm{Y}$ are i.i.d. and uniformly sub-exponential, and  matrix $ \bm{X} = \bm{Q}\Sigma_X^{1/2}$, where the entries of $\bm{Q}$ are independent with mean $0$, variance $1$, and uniformly sub-exponential tails.
        \item \label{cond:C2} $\sigma_1,\ldots, \sigma_p \in [\mathsf{M}^{-1}, \mathsf{M}]$  for some constant $\mathsf{M} > 0$.
\end{enumerate}

The  uniform sub-exponential assumption above ensures that $\sup_{i\in [n],j\in[p]}(\pnorm{\bm{Y}_i}{\psi_1} + \pnorm{{\bm{Q}}_{i,j}}{\psi_1}  ) \\\leq C$ for some universal $C > 0$, where $\psi_1$ is the Orlicz $1$-norm. Under conditions \eqref{cond:C1} and \eqref{cond:C2}, a central limit theorem heuristic implies that conditional on $\bm{Y}$, for any $j \in \mathcal{H}_0$,
\begin{align*}
\binom{ { \bm{X}_{\cdot,j}^\top \bm{Y} }/\pnorm{\bm{Y}}{}  }{{ \widehat{\bm{X}}_{\cdot,j}^\top \bm{Y} }/\pnorm{\bm{Y}}{}}\overset{d}{\approx}  \mathcal{N}\big(0, \sigma_j^2I_2 \big).
\end{align*} 
The invariance of the covariance structure in (\ref{eq:cov_exchange}) is reflected here by the equality that $\var \big( \frac{\bm{X}_{\cdot,j}^\top \bm{Y}}{\pnorm{\bm{Y}}{}}    |  \bm{Y}\big) = \var \big( \frac{\widehat{\bm{X}}_{\cdot,j}^\top \bm{Y} }{\pnorm{\bm{Y}}{}}  |  \bm{Y} \big) = \sigma_j^2$. Consequently, we expect that for each $j \in \mathcal{H}_0$,
\begin{align*}
	& \Prob ( \widehat{\bm{W}}_j \geq t) \approx \mathcal{P}_j(t)  \equiv \Prob_{(\mathsf{G}_1, \mathsf{G}_2)^\top \sim \mathcal{N}(0, \sigma_j^2I_2   ) }(\abs{\mathsf{G}_1} - \abs{\mathsf{G}_2}  \geq t) \\
	& =\Prob_{(\mathsf{G}_1, \mathsf{G}_2)^\top \sim \mathcal{N}(0, \sigma_j^2I_2  ) }(\abs{\mathsf{G}_1} - \abs{\mathsf{G}_2}  \leq -t) \approx \Prob ( \widehat{\bm{W}}_j \leq -t),
\end{align*}
which leads to the verification of condition 1 in Theorem \ref{thm:general_framwork}. It is important to note that the approximation above is understood in the sense of ratios, i.e., $\Prob ( \widehat{\bm{W}}_j \geq t)/ \mathcal{P}_j(t)  \approx 1 $ and $\Prob ( \widehat{\bm{W}}_j \leq -t)/\mathcal{P}_j(t) \approx 1$. This ratio-based formulation can be interpreted as a moderate deviation result, providing an effective approximation across a broad range of $t$ values. Proving the ratio-based approximation is much more challenging than proving the difference-based approximation, because the denominator can take vanishingly small values over some range of $t$ values.  A rigorous statement and proof of these results can be found in Theorem \ref{thm:moder_dis_mcks} in the Supplementary Material, where the proofs  are included therein as well.

The proposition below presents the concentration of the approximate knockoff statistics. 

\begin{proposition}\label{prop:conc_knockoff}
Assume that conditions \eqref{cond:C1} and \eqref{cond:C2} are satisfied. Then if $\log p = \mathfrak{o}(\sqrt{n})$, there exists some constant $C = C(\mathsf{M}) > 0$ such that for any $j \in [p]$, 
	    \begin{align*}
	    	\Prob \big(  \bigabs{ \widehat{\bm{W}}_j - \widehat{\bm{w}}_j } \geq C\delta_n \big) \leq   p^{-10} + \exp ( -n  \E^2 Y^2/C ),
	    \end{align*}
        	    where $\widehat{\bm{w}}_j \equiv  \sqrt{n}\abs{\E X_j Y} /\E^{1/2}Y^2$ and $\delta_n \equiv  \frac{\log p}{\E^{1/2}Y^2} + \frac{\log^2 p}{\sqrt{ n\E Y^2}} $.
\end{proposition}

The quantity $ \widehat{\bm{w}}_j$ in Proposition \ref{prop:conc_knockoff} above can be interpreted as the signal strength. Clearly, $ \widehat{\bm{w}}_j = 0 $ for each $ j \in \mathcal{H}_0 $. For knockoff statistics, one expects large and positive signals for $ j \in \mathcal{H}_1 $, cf. \cite{barber2015controlling,candes2018panning}. Specifically, we impose the following conditions on the class of approximate knockoff statistics:

\begin{enumerate}[label=(C\arabic*), ref=C\arabic*, start = 3]
	\item \label{cond:C3} For a large enough constant $C >0$, $a_n \equiv \#\{j \in \mathcal{H}_1: \widehat{\bm{w}}_j \geq C{\delta}_n  \} \to \infty $.
    \item \label{cond:C4} For any $q \in (0,1)$, $(qa_n)^{-1} \sum_{j \in \mathcal{H}_1}\Prob ( \widehat{\bm{W}}_j <- \mathcal{P}^{-1} ( \frac{q  a_n }{2p}) ) \to 0$. Here, for any $x \in [0,1/2]$, $\mathcal{P}^{-1}(x) \equiv \sup \{t \geq 0 : \mathcal{P}(t) \geq x \}$ with $\mathcal{P}(t) \equiv \abs{\mathcal{H}_0}^{-1}\sum_{j \in \mathcal{H}_0} \mathcal{P}_j(t)$. 
\end{enumerate}

Intuitively, the two conditions above characterize the desired quality for the approximate knockoff statistics $\widehat{\bm W}_j$ as a variable importance measure. Our study reveals that the accuracy of $\widehat{\bm W}_j$'s as variable importance measure has a great implication on the robustness of the knockoffs procedure. Ideally, accurate $\widehat{\bm W}_j$'s should possess the following properties: i)  $\widehat{\bm W}_j$'s for signals to concentrate around large positive values with small variation, and ii) $\widehat{\bm W}_j$'s for null features to concentrate around $0$ with small variation. The higher quality the $\widehat{\bm W}_j$'s, the more robust the knockoffs procedure is with respect to the deviation of $\widehat F$ from $F$. 

Condition \eqref{cond:C3} states that there should be enough number of nonnull features whose $\widehat{\bm W}_j$'s are reasonably large (i.e., $\geq C\delta_n$).
The technical condition \eqref{cond:C4} imposes a regularity condition on the distribution of knockoff statistics, ensuring that extreme negative values of $\widehat{\bm{W}}_j$ are sufficiently rare for signals $j \in \mathcal{H}_1$. Intuitively, this condition prevents the knockoff statistics for true signals from behaving like the the ones for null features in the tail regions. 
Below we discuss two special cases when condition \eqref{cond:C4} is satisfied:
\begin{itemize}
    \item Assume that the signal strength is sufficiently strong such that  $\widehat{\bm{w}}_j= \sqrt{n}\abs{\E X_j Y} /\E^{1/2}Y^2 \\\geq C\delta_n$ for all $j \in \mathcal{H}_1$. Then from the concentration result in Proposition \ref{prop:conc_knockoff}, we have 
    \begin{align*}
        \sum_{j \in \mathcal{H}_1}\Prob \bigg( \widehat{\bm{W}}_j <- \mathcal{P}^{-1} \bigg( \frac{q  a_n }{2p} \bigg) \bigg) \leq \sum_{j \in \mathcal{H}_1} \Prob \big( \widehat{\bm{W}}_j <0 \big) \to 0.
    \end{align*}
    \item Assume that, on average, the probability of a relevant feature having a negative valued $\widehat{\bm{W}}_j$ is smaller than the corresponding probability for an irrelevant feature. Specifically, assume that $\abs{\mathcal{H}_1}^{-1} \sum_{j \in \mathcal{H}_1}\Prob ( \widehat{\bm{W}}_j <- \mathcal{P}^{-1} ( \frac{q  a_n }{2p}) ) \leq \abs{\mathcal{H}_0}^{-1} \sum_{j \in \mathcal{H}_0}\Prob ( \widehat{\bm{W}}_j <- \mathcal{P}^{-1} ( \frac{q  a_n }{2p}) ) $. Then from the moderate deviation result in Theorem \ref{thm:moder_dis_mcks} in the Supplementary Material, 
    \begin{align*}
        &\frac{1}{qa_n} \sum_{j \in \mathcal{H}_1}\Prob \bigg( \widehat{\bm{W}}_j <- \mathcal{P}^{-1} \bigg( \frac{q  a_n }{2p} \bigg) \bigg) \leq \frac{\abs{\mathcal{H}_1}}{qa_n}\cdot \frac{1 }{ \abs{\mathcal{H}_0}} \sum_{j \in \mathcal{H}_0}\Prob \bigg( \widehat{\bm{W}}_j <- \mathcal{P}^{-1} \bigg( \frac{q  a_n }{2p} \bigg) \bigg)  \\
        &\quad \leq \frac{\abs{\mathcal{H}_1}}{qa_n} \mathcal{P}\bigg(  \mathcal{P}^{-1} \bigg( \frac{q  a_n }{2p} \bigg)   \bigg) \cdot (1 + \mathfrak{o}(1)) = \frac{\abs{\mathcal{H}_1}}{2p}\cdot (1 + \mathfrak{o}(1)) \to 0.
    \end{align*}
    A similar assumption regarding the tail probability of the relevant feature is made in \cite{fan2023ark}; see Condition 10 therein.
\end{itemize}

The sub-exponential tail assumption in Condition \eqref{cond:C1} rules out the $t$-distribution in Example \ref{ex2-t-distr}. Yet, simulation studies as in \cite{fan2023ark} show that approximate knockoffs procedure achieves the FDR control even when the degrees of freedom for the $t$-distribution is as small as $\mathfrak{q}=5$. Motivated by this, we extend our theoretical analysis to establish asymptotic FDR control for heavy-tailed covariates. To this end, we introduce the following conditions tailored for heavy-tailed settings:

\begin{enumerate}[label=(C\arabic*'), ref=C\arabic*']
	\item \label{cond:C1'}  
    The components of $\bm{Y}$ are i.i.d. and uniformly sub-exponential, and matrix $ \bm{X} = \bm{Q}\Sigma_X^{1/2}$, where the entries of $\bm{Q}$ are independent and have mean $0$, variance $1$, and finite $\mathfrak{q}$-th moments.
	   \setcounter{enumi}{2}
    \item \label{cond:C3'}  For a large enough constant $C >0$, $a_n \equiv \#\{j \in \mathcal{H}_1: \widehat{\bm{W}}_j \geq C{\delta}_{n;\mathfrak{q}}  \} \to \infty $. Here, $\delta_{n;\mathfrak{q}} \equiv \frac{p^{2/\mathfrak{q}}\log p}{\E^{1/2}Y^2} +  \frac{p^{2/\mathfrak{q}}\log^2 p}{\sqrt{n\E Y^2}} $.
\end{enumerate}

We are now ready to present the result for the asymptotic FDR control of marginal correlation knockoff statistics.

\begin{theorem}\label{thm:FDR_mcks}
Assume that either of the following conditions holds:
\begin{enumerate}
	\item \eqref{cond:C1}--\eqref{cond:C4} hold and $\log p = \mathfrak{o}(n^c)$ for some small enough constant $c>0$. 
    \item  \eqref{cond:C1'} and \eqref{cond:C3'} hold for some $\mathfrak{q} \geq 3$, \eqref{cond:C2} and \eqref{cond:C4} hold, and  $p \log n = \mathfrak{o}(\sqrt{n})$.
\end{enumerate}
Then for any $q \in (0,1)$, conditions (1)--(3) in Theorem \ref{thm:general_framwork} are satisfied for $\alpha_n = \mathcal{P}^{-1} \big( \frac{q a_n}{2p} \big)$. Consequently, we have 
\begin{align*} 
\limsup_{n \to \infty} \mathrm{FDR} \leq q. 
\end{align*} 
\end{theorem}

\begin{remark}
    In the heavy-tailed setting, additional requirements are imposed compared to the light-tailed case:
    \begin{enumerate}
        \item Condition \eqref{cond:C3'} requires that the signal strength increases as the tails of the entries in the knockoffs data matrix become heavier. This requirement is intuitive because heavier tails introduce larger variation in the knockoff statistics $\widehat{\bm W}_j$'s, making it more challenging to distinguish true signals from noise ones. A stronger signal condition is imposed here to mitigate these effects and ensure the asymptotic FDR control.
        \item The restriction $p \log n = \mathfrak{o}(\sqrt{n})$ limits the allowable growth of $p$ relative to $n$. This condition arises because the moderate deviation result for $\widehat{\bm W}_j$'s  used in the light-tailed case (e.g., Theorem \ref{thm:moder_dis_mcks} in the Supplementary Material) may not hold for heavy-tailed distributions. Instead, we have to rely on the Berry--Esseen Theorem to prove condition (1) in Theorem \ref{thm:general_framwork}, which is evidently less precise than the moderate deviation result when $t$ is moderately large, specifically for $t = \Theta(\log p)$. We conjecture that if a robust construction of $\widehat{\bm W}_j$'s (with respect to heavy tails) is used, then the dimensionality assumption here can be further relaxed. We leave this for future investigation.  
    \end{enumerate}
\end{remark}

The corollary below is to formally justify the asymptotic FDR control using Gaussian knockoffs in the settings presented in Examples \ref{exmp:discrete_covariate} and \ref{ex2-t-distr} in the last section.

\begin{corollary}\label{cor:specialdis}
Assume that $Y$ is sub-exponential and either of the following holds:
\begin{enumerate}
    \item $X_1,\ldots,X_p$ are i.i.d. Rademacher random variables, \eqref{cond:C3} and \eqref{cond:C4} are satisfied, and $\log p = \mathfrak{o}(n^{c})$ for some small enough constant $c>0$. 
    \item $X_1,\ldots,X_p$ are i.i.d. $t$-distributed random variables with $\mathfrak{q} \geq 3$ degrees of freedom, \eqref{cond:C3'} and \eqref{cond:C4} are satisfied, and $p\log n  = \mathfrak{o}(\sqrt{n})$.
\end{enumerate}
Then we have 
\begin{align*} 
\limsup_{n \to \infty} \mathrm{FDR} \leq q. 
\end{align*} 
\end{corollary}

Note that in the $t$-distribution case using marginal correlation knockoff statistics, although our theory can accommodate finite $\mathfrak q$, Corollary \ref{cor:specialdis} above requires $p \log n = \mathfrak{o}(\sqrt{n})$. In comparison, \cite{fan2023ark} can handle the high-dimensional regime $p \gg n$ by using the coupling technique. 
This tradeoff between the range of allowable $\mathfrak{q}$ and $p$ reflects the inherent challenges in balancing tail robustness and dimensional scalability, highlighting the \textit{complementary} strengths of the coupling proof technique and our theoretical framework here.

  While the results in this section are presented for the case of independent features, the framework extends naturally to correlated features under some assumptions on the correlation strength. 
    To save space, the asymptotic FDR control results and conditions for correlated features with marginal correlation knockoff statistics are presented in Section \ref{sec:mcks_correlated} of the Supplementary Material.

\subsection{Regression coefficient difference with ordinary least squares} \label{sec:rcd_ols}

Different from the case of marginal correlation knockoff statistics which does not need a model assumption on how $Y$ depends on $X_{\mathcal H_1}$, the case of regression coefficient based knockoff statistics needs a regression model assumption to verify conditions (1)--(3) in Theorem \ref{thm:general_framwork}. Specifically, we consider the linear regression model 
\begin{align*}
	Y = X^\top \beta_\ast + \xi,
\end{align*}  
where $\beta_\ast = (\beta_{\ast;1},\dots, \beta_{\ast;p})^T$ with $\beta_{\ast;j} = 0$ for each $j \in \mathcal{H}_0$.
The above model is a special case of (\ref{def:general_model}). We consider the knockoff statistics based on the differences in regression coefficients estimated using the \emph{ordinary least squares (OLS) estimator}. The knockoff statistics using the \emph{debiased Lasso estimator}, which shares conceptual similarity, is provided in next section.

Recall that $\bm{\xi} = (\bm{\xi}_1,\ldots,\bm{\xi}_n )^\top \in \R^n $ denotes the noise vector.
 Given the original data matrix $\bm{X}$ and the knockoffs data matrix $\widehat{\bm{X}}$ in \eqref{eq:approx_knock}, let $\widehat{\bm{Z}} = [\bm{X}, \widehat{\bm{X}}] \in \R^{n \times 2p}$. The knockoff statistic based on the OLS estimator is defined as 
 \begin{align*}  
        \widehat{\bm{W}}_j \equiv \abs{\widehat{\beta}_j^{\mathsf{LS}}} - \abs{\widehat{\beta}_{j+p}^{\mathsf{LS}}}, \  j \in [p] \text{ with }  \widehat{\beta}^{\mathsf{LS}} \equiv \frac{1}{\sqrt{n}} \bigg(\frac{\widehat{\bm{Z}}^\top \widehat{\bm{Z}}}{n}\bigg)^{-1} \widehat{\bm{Z}}^\top \bm{Y}.  
    \end{align*} 
 Note that our definition of   $\widehat{\beta}^{\mathsf{LS}}$ differs from the conventional definition by a factor of $\sqrt n$.  Here, we impose the the following working assumptions:
\begin{enumerate}[label=(O\arabic*), ref=O\arabic*]
	\item \label{cond:O1} $n/2p \geq 1 + \tau$ holds for some $\tau \in (0,1)$.
   \item \label{cond:O2} $\widehat{\bm{Z}} = \widehat{\bm{Q}} \Sigma^{1/2}$ is independent of the normalized noise vector $\bm{\xi}_0 = \bm{\xi} / \sigma_{\xi}$. The components of $\bm{\xi}_0$ are i.i.d., mean-zero, unit-variance, and uniformly sub-Gaussian. The rows of $\widehat{\bm{Q}}$ are i.i.d., with mean-zero, unit-variance, and uniformly sub-Gaussian entries, which are either (i) independent or (ii) satisfy a convex concentration property: there exists some constant $\mathsf{c}_{\mathsf{L}} > 0$ such that for any $1$-Lipschitz convex function $\varphi: \mathbb{R}^{2p} \to \mathbb{R}$ and for each $t > 0$ and $i \in [n]$, $\mathbb{P}( | \varphi(\widehat{\bm{Q}}_{i,\cdot}) - \mathbb{E}[\varphi(\widehat{\bm{Q}}_{i,\cdot})] | \geq t ) \leq 2 \exp(-t^2 / \mathsf{c}_{\mathsf{L}})$.
	\item \label{cond:O3} $\pnorm{\Sigma}{\op} \vee \pnorm{\Sigma^{-1}}{\op} \leq \mathsf{M}$, and $\sigma_\xi \in [\mathsf{M}^{-1}, \mathsf{M}]$ for some constant $\mathsf{M} > 0$.
\end{enumerate}

It is important to note that we no longer require $\Sigma$ to be diagonal, allowing for \textit{dependent} covariates. The technical motivation for condition \eqref{cond:O2} above is to establish the concentration of quadratic forms of $\widehat{\bm{Q}}_{i,\cdot}$ and ensure that the smallest eigenvalue of $\widehat{\bm{Z}}^\top\widehat{\bm{Z}}/n$ is bounded from below. These results, in turn, play a crucial role in enabling the approximation of the precision matrix $(\widehat{\bm{Z}}^\top \widehat{\bm{Z}} / n)^{-1}$ to its population counterpart. We note that, by Herbst's argument, the convex concentration property holds for random vectors satisfying a log-Sobolev inequality \cite[Theorem 5.4]{boucheron2013concentration}, with strongly log-concave random vectors (including the Gaussian vector) being a special case \cite[Theorem 5.2]{ledoux2001concentration}.

Let us provide some heuristics to explain why the asymptotic approximate symmetry of $\widehat{\bm{W}}_j$ holds for each $j \in \mathcal{H}_0$. To this end, with slight abuse of notation, we define for any $j \in \mathcal{H}_0$  and $t > 0$,
\begin{align}\label{eq:Pj_positive}
    \mathcal{P}_j(t) \equiv \Prob_{(\mathsf{G}_1,\mathsf{G}_2)^\top \sim \mathcal{N} (0, \sigma_\xi^2\Sigma_n^{(j)})} \big(\abs{\mathsf{G}_1} - \abs{\mathsf{G}_2} \geq t\big),
\end{align}
where $\Sigma_n^{(j)} \equiv (1 - 2p/n)^{-1}[\Sigma^{-1}]_{\{j, j+p\}, \{j, j+p\}}$. Recall that $\Sigma = \cov \big[ (X^\top, \widehat{X}^\top )^\top \big]$.
Using the invariance of the covariance structure in (\ref{eq:cov_exchange}), and letting $\bm{P}_{(j)}$ be the permutation matrix that swaps the $j$th and $(j+p)$th rows, we have 
\begin{align}\label{eq:invar_swap}
	e_j^\top \Sigma^{-1} e_j = e_j^\top (\bm{P}_{(j)} \Sigma \bm{P}_{(j)})^{-1} e_j = e_j^\top \bm{P}_{(j)} \Sigma^{-1} \bm{P}_{(j)} e_j = e_{j+p}^\top \Sigma^{-1} e_{j+p}.
\end{align}
Here, $\{e_k \}_{k \in [2p]}$ is the canonical basis of $\R^{2p}$.
Consequently, for any $j \in \mathcal{H}_0$ and $t > 0$,
\begin{align}\label{eq:Pj_negative}
    \mathcal{P}_j(t) = \Prob_{(\mathsf{G}_1,\mathsf{G}_2)^\top \sim \mathcal{N}(0, \sigma_\xi^2 \Sigma_n^{(j)})} \big(\abs{\mathsf{G}_1} - \abs{\mathsf{G}_2} \leq -t\big).
\end{align}

The asymptotic approximate symmetry of $\widehat{\bm{W}}_j$ is established through two steps: the Gaussian approximation and the covariance approximation.

For the Gaussian approximation, conditional on $\widehat{\bm{Z}}$, a central limit theorem heuristic suggests that for any $j \in \mathcal{H}_0$,
\begin{align}\label{eq:LS_ga_heuristic}
    \binom{\widehat{\beta}_j^{\mathsf{LS}}}{\widehat{\beta}_{j+p}^{\mathsf{LS}}} = \frac{1}{\sqrt{n}} \binom{e_j^\top (\widehat{\bm{Z}}^\top \widehat{\bm{Z}}/n)^{-1} \widehat{\bm{Z}}^\top \bm{\xi}}{e_{j+p}^\top (\widehat{\bm{Z}}^\top \widehat{\bm{Z}}/n)^{-1} \widehat{\bm{Z}}^\top \bm{\xi}} \overset{d}{\approx} \mathcal{N} \big(0, \sigma_\xi^2[(\widehat{\bm{Z}}^\top \widehat{\bm{Z}}/n)^{-1}]_{\{j, j+p\}, \{j, j+p\}} \big),
\end{align}
approximating the joint distribution of $\widehat{\beta}_j^{\mathsf{LS}}$ and $\widehat{\beta}_{j+p}^{\mathsf{LS}}$ by a multivariate Gaussian despite the non-Gaussian nature of noise $\bm{\xi}$. 

For the covariance approximation, we can show that $[(\widehat{\bm{Z}}^\top \widehat{\bm{Z}}/n)^{-1}]_{\{j, j+p\}, \{j, j+p\}} \overset{\mathbb{P}}{\approx} \Sigma_n^{(j)}$, where the shrinkage factor $(1 - 2p/n)^{-1}$ accounts for the overestimation of eigenvalues in the high-dimensional sample covariance matrix, ensuring accurate calibration of the covariance structure. Consequently,  $[(\widehat{\bm{Z}}^\top \widehat{\bm{Z}}/n)^{-1}]_{\{j, j+p\}, \{j, j+p\}}$ in (\ref{eq:LS_ga_heuristic}) can be replaced with $\Sigma_n^{(j)}$. 

A combination of steps above yields that 
\begin{align*}
    \Prob(\widehat{\bm{W}}_j \geq t) &\approx \Prob_{(\mathsf{G}_1,\mathsf{G}_2)^\top \sim \mathcal{N}(0, \sigma_\xi^2 \Sigma_n^{(j)})} \big(\abs{\mathsf{G}_1} - \abs{\mathsf{G}_2} \geq t\big) \\
	&\overset{\eqref{eq:Pj_positive} \& \eqref{eq:Pj_negative}}{=} \Prob_{(\mathsf{G}_1,\mathsf{G}_2)^\top \sim \mathcal{N}(0, \sigma_\xi^2 \Sigma_n^{(j)})} \big(\abs{\mathsf{G}_1} - \abs{\mathsf{G}_2} \leq -t\big) \approx \Prob(\widehat{\bm{W}}_j \leq -t),
\end{align*}
which ensures that condition (1) in Theorem \ref{thm:general_framwork} can be verified. 
A formal statement of this result is provided in Theorem \ref{thm:moder_dis_ols} of the Supplementary Material.

The proposition below provides the concentration of the approximate knockoff statistics.
\begin{proposition}\label{prop:conc_knockoff_ols}
Assume that conditions \eqref{cond:O1}--\eqref{cond:O3} are satisfied. Then there exists some constant $C = C(\mathsf{M},\tau, \mathsf{c}_{\mathsf{L}}) > 0$ such that for any $j \in [p]$ and $p \geq C$,
	   $ 	\Prob (  \abs{ \widehat{\bm{W}}_j -  \sqrt{n}\abs{\beta_{\ast,j}}  } \geq C\sqrt{\log p} ) \leq   Cp^{-10}.$
\end{proposition}
The signal strength $\abs{\beta_{\ast,j}}$ significantly influences the quality of the knockoff statistics. To ensure robust performance, we impose the conditions below on the desired quality of the approximate knockoff statistics as variable importance measures. Recall that $\mathcal{P}(t) = \abs{\mathcal{H}_0}^{-1}\sum_{j \in \mathcal{H}_0} \mathcal{P}_j(t)$ and $\mathcal{P}^{-1}(x) =\sup \{t \geq 0 : \mathcal{P}(t) \geq x \}$. 

\begin{enumerate}[label=(O\arabic*), ref=O\arabic*, start = 4]
	\item \label{cond:O4} $a_n \equiv \#\{j \in \mathcal{H}_1: \sqrt{n}\abs{\beta_{\ast,j}} \gg \sqrt{\log p}  \} \to \infty $ and  $m_n/a_n \to 0$. Here, $m_n$ measures the  sparsity in precision matrix $\Sigma^{-1}$ and is defined as $m_n \equiv  \max_{j \in \mathcal{H}_0} \abs{N(j)}$ with 
$
    N(j) \equiv \big\{\ell \in \mathcal{H}_0 \setminus\{j\}: \sum_{\mathfrak{j} \in \{j, j + p\} } \big(\abs{ [\Sigma^{-1}]_{\mathfrak{j}, \ell }} + \abs{ [\Sigma^{-1}]_{\mathfrak{j}, \ell+p }}\big) \neq 0  \big\}.
$
	\item \label{cond:O5} For any $q \in (0,1)$, $(qa_n)^{-1} \sum_{j \in \mathcal{H}_1}\Prob ( \widehat{\bm{W}}_j <- \mathcal{P}^{-1} ( \frac{q  a_n }{2p}) ) \to 0$. 
\end{enumerate}

Conditions \eqref{cond:O4} and \eqref{cond:O5} above mimic those of \eqref{cond:C3} and \eqref{cond:C4}, which are introduced in the context of the marginal correlation knockoff statistics in the previous subsection. Although  \eqref{cond:O5} and \eqref{cond:C4} share an identical form, we emphasize that the definitions of $a_n$, $\widehat{\bm{W}}_j$, and $\mathcal{P}^{-1}(x)$ here differ from those in \eqref{cond:C4}, as they depend on the context under consideration.
 Compared to \eqref{cond:C3}, \eqref{cond:O4} accounts for feature correlations by balancing the correlation strength and the number of strong signals. In the independent case considered in \eqref{cond:C3}, $m_n = 0$, such a balance is unnecessary. However, in correlated settings, as the strength of feature correlations increases, it becomes more difficult to distinguish null features from non-null features, we need \eqref{cond:O4} to counteract this effect and maintain asymptotic FDR control.

We are now ready to present the asymptotic FDR control result for knockoff statistics based on the differences in regression coefficients estimated via the OLS estimator.

\begin{theorem}\label{thm:FDR_ols}
Assume that conditions \eqref{cond:O1}--\eqref{cond:O5} hold. Then for any $q \in (0,1)$, conditions (1)--(3) in Theorem \ref{thm:general_framwork} are satisfied for $\alpha_n = \mathcal{P}^{-1} \big( \frac{q a_n}{2p} \big)$. Consequently, we have 
\begin{align*} 
\limsup_{n \to \infty} \mathrm{FDR} \leq q. 
\end{align*}
\end{theorem}

\subsection{Regression coefficient difference with debiased Lasso} \label{sec:debiased Lasso}
The OLS estimator, though conceptually simple, is not well-defined when the number of features exceeds the number of observations, specifically when $p > n$, as discussed in the previous section. In this section, we turn to the \emph{debiased Lasso estimator} (\cite{Zhang2013DL}), which is particularly suited for high-dimensional sparse linear regression with $p > n$. 
The debiased Lasso estimator overcomes the rank deficiency issue inherent in high-dimensional settings and naturally adapts to the sparse signal structure, where the true regression coefficient vector $\beta_\ast$ contains many zero components. The covariates are allowed to be correlated in this section.

We now formally define the debiased Lasso estimator. Given a sequence of regularization parameters $\{\lambda_j\}_{j \in [2p]}$, let the initialization $\widehat{\beta}^{(0)} \in \mathbb{R}^{2p}$ and the score vector $\widehat{\bm{z}}_j \in \mathbb{R}^n$ be defined as 
\begin{align*}
    \widehat{\beta}^{(0)} \equiv \arg\min_{\bm{b} \in \mathbb{R}^{2p}} \left\{ \frac{1}{2n} \|\bm{Y} - \widehat{\bm{Z}} \bm{b}\|^2 + \lambda_0 \|\bm{b}\|_1 \right\}, \quad 
    \widehat{\bm{z}}_j \equiv \widehat{\bm{Z}}_{\cdot,j} - \widehat{\bm{Z}}_{\cdot,-j} \widehat{\bm{\gamma}}_j, \quad j \in [2p],
\end{align*}
where the  regularization parameter $\lambda_0=C\sqrt{n^{-1}\log (2p)}$ for some constant $C > 0$ and $\widehat{\bm{\gamma}}_j \equiv \arg\min_{\bm{b} \in \mathbb{R}^{2p}}\{ (2n)^{-1}\pnorm{ \widehat{\bm{Z}}_{\cdot,j} - \widehat{\bm{Z}}_{\cdot,-j} \bm{b} }{}^2  + \lambda_j \pnorm{\bm{b}_j}{1} \}$.
The debiased Lasso estimator $\widehat{\beta}^{\mathsf{dL}}$ is then given by  
\begin{align*}
    \widehat{\beta}^{\mathsf{dL}}_j &\equiv \widehat{\beta}^{(0)}_j + \frac{\langle \widehat{\bm{z}}_j, \bm{Y} - \widehat{\bm{Z}} \widehat{\beta}^{(0)} \rangle}{\langle \widehat{\bm{z}}_j, \widehat{\bm{Z}}_{\cdot,j} \rangle}, \quad j \in [2p].
\end{align*}  
The knockoff statistics based on the debiased Lasso estimator is then defined as 
    \begin{align*}  
        \widehat{\bm{W}}_j \equiv \sqrt{n} \abs{\widehat{\beta}_j^{\mathsf{dL}}} -\sqrt{n}  \abs{\widehat{\beta}_{j+p}^{\mathsf{dL}}}, \quad j \in [p].  
    \end{align*}  
Here, the scaling factor $\sqrt{n}$ ensures that $\widehat{\bm{W}}_j$'s are order $1$ random variables for $j \in \mathcal{H}_0$. Let $s \equiv \pnorm{\beta_\ast }{0}$. We will work under the following conditions:
\begin{enumerate}[label=(L\arabic*), ref=L\arabic*]
	\item \label{cond:L1} $\widehat{\bm{Z}} = \widehat{\bm{Q}} \Sigma^{1/2}$ is independent of the normalized noise vector $\bm{\xi}_0 = \bm{\xi} / \sigma_{\xi}$. The components of $\bm{\xi}_0$ are i.i.d., mean-zero, unit-variance, and uniformly sub-Gaussian. The rows of $\widehat{\bm{Q}}$ are i.i.d., with mean-zero, unit-variance, and uniformly sub-Gaussian entries. Moreover, the restricted eigenvalue property holds for $\widehat{\bm{Z}} \widehat{\bm{Z}}^{\top}/n$ with high probability: there exist some large enough 
 constants $\mathsf{C}_1,\mathsf{C}_2,\mathsf{C}_3 > 0$ such that 
	\begin{align*}
		\Prob\bigg( \min_{\pnorm{\delta}{0} \leq  \mathsf{C}_1s }  \frac{\delta^\top \widehat{\bm{Z}}^\top\widehat{\bm{Z}} \delta }{n \pnorm{\delta}{ }^2 } \geq \mathsf{C}_2^{-1} \bigg) \geq 1- \mathsf{C}_3p^{-10}.
	\end{align*}
\end{enumerate}

Condition \eqref{cond:L1} above is commonly imposed in the literature to ensure the consistency of the Lasso estimator. This consistency facilitates the Gaussian approximation for the debiased Lasso estimator. Specifically, for each $j \in \mathcal{H}_0$, we can establish that 
\begin{align*}
    \Prob \big( \widehat{\bm{W}}_j \geq t\big) = \Prob \big( \sqrt{n} \abs{\widehat{\beta}_j^{\mathsf{dL}}} -\sqrt{n}  \abs{\widehat{\beta}_{j+p}^{\mathsf{dL}}} \geq t\big) \approx    \Prob_{(\mathsf{G}_1,\mathsf{G}_2)^\top \sim \mathcal{N} (0, \sigma_\xi^2\Sigma_n^{(j)}) }\big(\abs{\mathsf{G}_1} - \abs{\mathsf{G}_2} \geq t \big) \equiv \mathcal{P}_j(t ),
\end{align*}
where  $\Sigma^{(j)} \equiv [ \Sigma^{-1}]_{\{j, j+p\} , \{j, j+p\}}$. It follows from (\ref{eq:invar_swap}) that $\mathcal{P}_j(t ) =\Prob_{(\mathsf{G}_1,\mathsf{G}_2)^\top \sim \mathcal{N} (0, \sigma_\xi^2\Sigma_n^{(j)}) }\big(\abs{\mathsf{G}_1} - \abs{\mathsf{G}_2} \leq -t \big)$. Consequently, it also holds that 
\begin{align*}
    \Prob \big( \widehat{\bm{W}}_j \leq -t\big) = \Prob \big( \sqrt{n} \abs{\widehat{\beta}_j^{\mathsf{dL}}} -\sqrt{n}  \abs{\widehat{\beta}_{j+p}^{\mathsf{dL}}} \leq -t\big) \approx     \mathcal{P}_j(t ).
\end{align*}
Combining these expressions, we observe the asymptotic approximate symmetry of $\widehat{\bm{W}}_j$. A formal statement of this result can be found in Theorem \ref{thm:moder_dis_dl} in the Supplementary Material.

Let $b_n \equiv \max_{j \in [2p]} \pnorm{\Sigma^{-1}_{\cdot,j}}{0} $. Recall that $\mathcal{P}(t) = \abs{\mathcal{H}_0}^{-1}\sum_{j \in \mathcal{H}_0} \mathcal{P}_j(t)$ and $\mathcal{P}^{-1}(x) =\sup \{t \geq 0 : \mathcal{P}(t) \geq x \}$.  We have the asymptotic FDR control result for knockoff statistics based on the differences in regression coefficients estimated via the debiased Lasso estimator.
\begin{theorem}\label{thm:FDR_dl}
Assume that conditions \eqref{cond:O3}--\eqref{cond:O5} and \eqref{cond:L1} hold.
  Then for any $q \in (0,1)$, if $b_n s\log^7(np) /\sqrt{n} = \mathfrak{o}(1)$, conditions (1)--(3) in Theorem \ref{thm:general_framwork} are satisfied for $\alpha_n = \mathcal{P}^{-1} \big( \frac{q a_n}{2p} \big)$. Consequently, we have 
\begin{align*} 
\limsup_{n \to \infty} \mathrm{FDR} \leq q. 
\end{align*}
\end{theorem}

\subsection{Gaussian knockoffs generator using estimated moments} \label{new.sec3.4}

The examples discussed in the previous subsections assume exact knowledge of $\Sigma_X$ and $\Sigma_X^{-1}$. However, in practice, $\Sigma_X^{-1}$ is typically estimated from observed data ${\bm{X}}$; see, e.g., \cite{Fan2016precision} and the references therein for an overview of precision matrix estimation methods. Assume that we have an estimator for $\Sigma_X^{-1}$, which may depend on data ${\bm{X}}$, denoted as $\widehat{\bm{\Omega}}$. In this case, we can also consider a fully sample-based knockoffs data matrix constructed as
    \begin{align} \label{new.eq.L001}
    \widehat{\bm{X}}^{\mathsf{in}} = \bm{X}(I_p - \widehat{\bm{\Omega}}\mathrm{diag}(\bm{r})) + \bm{\mathsf{Z}}(2\mathrm{diag}(\bm{r}) - \mathrm{diag}(\bm{r})\widehat{\bm{\Omega}}\mathrm{diag}(\bm{r}))^{1/2}.
    \end{align}
We discuss in this subsection that using $ \widehat{\bm{X}}^{\mathsf{in}}$ in (\ref{new.eq.L001}) above can also lead to asymptotic FDR control, using the coupling proof technique in \cite{fan2023ark}. More specifically, the sample moments based knockoff variable matrix $ \widehat{\bm{X}}^{\mathsf{in}}$ can be coupled with the population moments based matrix in \eqref{eq:approx_knock}, provided that the identical realization for $\bm{\mathsf{Z}}$ is used in both places. The proposition below formally characterizes their coupling accuracy.     

\begin{proposition}\label{prop:couple}
    Let ${\bm{S}}_Z \equiv 2\mathrm{diag}(\bm{r}) - \mathrm{diag}(\bm{r})\Sigma_X^{-1}\mathrm{diag}(\bm{r})$ and $\widehat{\bm{S}}_Z \equiv 2\mathrm{diag}(\bm{r}) - \mathrm{diag}(\bm{r})\widehat{\bm{\Omega}}\mathrm{diag}(\bm{r})$.
    For any constants $\mathsf{c}_1, \mathsf{c}_2 > 0$, let $\bm{r} \in \R^n$ be chosen such that $ \mathsf{c}_1 \leq \lambda_{\min}( {\bm{S}}_Z ) \leq \lambda_{\max}( {\bm{S}}_Z ) \leq \mathsf{c}_1^{-1}$. Then on event $\{ \mathsf{c}_2 \leq \lambda_{\min}( \widehat{\bm{S}}_Z  ) \leq \lambda_{\max}( \widehat{\bm{S}}_Z  ) \leq \mathsf{c}_2^{-1} \}$, we have that with $\Prob_{\bm{\mathsf{Z}}}$-probability at least $1-2\exp (-n/32)$,
   \begin{align*}
       \pnorm{  \widehat{\bm{X}}^{\mathsf{in}} - \widehat{\bm{X}} }{1,2} \leq   \pnorm{\bm{r}}{\infty} \cdot  \pnorm{\bm{X} ( \widehat{\bm{\Omega}} - \Sigma_X^{-1})}{1,2}  +   6( \mathsf{c}_1^{-3/2} \vee \mathsf{c}_2^{-3/2})\pnorm{\bm{r}}{\infty}^2  \cdot \pnorm{ \widehat{\bm{\Omega}} - \Sigma_X^{-1}}{\op}, 
   \end{align*}
   where for any matrices $\mathsf{A},\mathsf{B} \in \R^{n \times p}$, $ \pnorm{ \mathsf{A} - \mathsf{B} }{1,2} \equiv n^{-1/2}\max_{j \in [p]} \pnorm{  \mathsf{A}_{\cdot,j} -  \mathsf{B}_{\cdot,j} }{}$.  
\end{proposition}
Proposition \ref{prop:couple} above provides a way to verify Condition 6 in \cite{fan2023ark}, provided that $\widehat{\bm{\Omega}}$ is a good estimator of $\Sigma_X^{-1}$, in the sense of both $\pnorm{\bm{X} ( \widehat{\bm{\Omega}} - \Sigma_X^{-1})}{1,2}$ and $ \pnorm{ \widehat{\bm{\Omega}} - \Sigma_X^{-1}}{\op}$ being small.
We have verified in the previous three subsections that $\widehat{\bm{X}}$ can lead to asymptotic FDR control. As guaranteed by Proposition \ref{prop:couple}, $\widehat{\bm{X}}^{\mathsf{in}}$ in (\ref{new.eq.L001}) and the coupled $\widehat{\bm{X}}$ are close in realization. 
Thus, using $\widehat{\bm{X}}$ as a bridge, the coupling framework in \cite{fan2023ark} can then be applied to show that knockoff statistics constructed using $\widehat{\bm{X}}^{\mathsf{in}}$ also satisfy the conditions in Theorem \ref{thm:general_framwork} under mild additional assumptions, thereby ensuring asymptotic FDR control. Since the detailed proofs are very similar to those in \cite{fan2023ark},  for the sake of brevity and to focus on the main contributions of this paper, we omit further technical details. 

From a broader perspective, the above procedure indicates that establishing asymptotic FDR control does \textit{not} necessarily require coupling the approximate knockoffs matrix with the ideal knockoffs matrix; in some cases, coupling the approximate knockoffs matrix with another well-structured knockoffs matrix, such as a Gaussian knockoffs matrix, is sufficient, provided that the knockoff statistics constructed using the Gaussian knockoffs are regular enough, i.e., satisfying conditions (1)--(3) in Theorem \ref{thm:general_framwork}.

\section{Numerical studies} \label{sec:simulation}

In this section, we demonstrate the finite-sample performance of the Gaussian knockoffs with two moments matching through some simulating examples and a real data application. 

\subsection{Simulation examples} \label{new.sec4.1}

In the following simulation settings, we consider a $p$-dimensional discrete covariate vector denoted as $X = (X_1, \ldots, X_p)^\top$, where  $X_j = \mathbbm{1} (Z_j \geq \alpha_j)$  with   $(Z_1,\ldots, Z_p)^\top \overset{d}{\sim} \mathcal{N}({\bf 0}_p, \bOmega^{-1}).$
Here, $\bOmega = (\bOmega_{i, j})_{i,j \in [p]} \in \R^{p \times p}$ with $\bOmega_{i, j} = \rho^{\abs{i-j}}\bm{1}\{\abs{i-j}\leq 5\}$ for some $\text{AR}(1)$ correlation $\rho \in [0, 1)$, and $(\alpha_1, \ldots, \alpha_p)$ are a sequence of $p$ cutoff points evenly located on $[-1, 1]$.  
Based on this covariate setup, we evaluate the performance of the two moments matched approximate Gaussian knockoffs (i.e., the second-order Gaussian knockoffs) under two distinct models for generating response variable $Y$: the linear model and the nonlinear model, as detailed below.

\begin{setting} \label{simu-linear}
     Let $Y = X^\top \beta + \epsilon$, where $\epsilon \stackrel{d}{\sim} \mathcal{N}(0, 1)$ and the true coefficient vector $\beta \in \mathbb{R}^p $ is sparse with $40$ nonzero components randomly distributed in $[p]$ and generated from $\{ \pm 3\}$.  
\end{setting}

\begin{setting}\label{simu-non-linear}
     Let $Y = 5 (|X^\top \beta|)^{1/2} \tanh( X^\top \beta) + \epsilon$ with $\epsilon$ and $\beta$ set similarly as in Setting \ref{simu-linear}.
\end{setting}

We consider two constructions of the knockoff statistics: one based on the marginal correlation (MC) difference and the other based on the debiased Lasso (DL) coefficient difference. 
When the $\text{AR}(1)$ correlation $\rho = 0$, the exact true mean and covariance of covariates are applied to generate the approximate Gaussian knockoffs. For $\rho > 0$, due to the challenge in computing the exact true covariance matrix of covariates, we apply the out-of-sample estimation of the mean and covariance with sufficient ($10^5$) training observations of covariates. 
Tables \ref{table-FDR-Power-linear} and \ref{table-FDR-Power-nonlinear} present the simulation results for the FDR and power of the approximate knockoffs inference under various combinations of $(n, p)$ and $\rho$ for Settings \ref{simu-linear} and \ref{simu-non-linear}, respectively. The simulating results indicate that the approximate knockoffs inference using Gaussian knockoffs with matched first two moments generally achieves desired FDR control across all settings, aligning with our theoretical results. 

\begin{table}[t]
\caption{The empirical FDR (power) for 
the Gaussian knockoffs under Setting \ref{simu-linear}, with targeted FDR level $q = 0.2$. Results are averaged over $100$ replications.}
\label{table-FDR-Power-linear}
\centering
\renewcommand{\arraystretch}{0.6}
\begin{tabular}{c|c|c|c|c|c}
\hline
\diagbox[width=6em]{$\rho$}{$(n, p)$} & Statistics & $(150, 300)$ & $(300, 300)$ & $(300, 600)$ & $(600, 600)$ \\ 
\hline
\multirow{2}{*}{0}   & MC & 0.168 (0.143) & 0.176 (0.525) & 0.180 (0.364) & 0.183 (0.799) \\ 
                     & DL & 0.199 (0.602) & 0.185 (1) & 0.208 (1) & 0.189 (1) \\ 
\hline
\multirow{2}{*}{0.3} & MC & 0.146 (0.127) & 0.193 (0.503) & 0.188 (0.380) & 0.216 (0.798) \\ 
                     & DL & 0.177 (0.625) & 0.207 (1) & 0.197 (1) & 0.198 (1) \\ 
\hline
\multirow{2}{*}{0.7} & MC & 0.120 (0.084) & 0.174 (0.303) & 0.194 (0.221) & 0.190 (0.617) \\ 
                     & DL & 0.175 (0.657) & 0.181 (1) & 0.184 (1) & 0.192 (1) \\ 
\hline
\end{tabular}
\end{table}

\begin{table}[t]
\caption{The empirical FDR (power) for the Gaussian knockoffs under Setting \ref{simu-non-linear}, with  targeted FDR level $q = 0.2$. Results are averaged over $100$ replications.}
\label{table-FDR-Power-nonlinear}
\centering
\renewcommand{\arraystretch}{0.6} 
\begin{tabular}{c|c|c|c|c|c}
\hline
\diagbox[width=6em]{$\rho$}{$(n, p)$} & Statistics & $(150, 300)$ & $(300, 300)$ & $(300, 600)$ & $(600, 600)$ \\ 
\hline
\multirow{2}{*}{0}   & MC & 0.163 (0.136) & 0.175 (0.468) & 0.176 (0.300) & 0.183 (0.734) \\ 
                     & DL & 0.190 (0.365) & 0.212 (0.999) & 0.195 (0.934) & 0.201 (1)\\ 
\hline
\multirow{2}{*}{0.3} & MC & 0.128 (0.107) & 0.188 (0.452) & 0.180 (0.289) & 0.197 (0.721) \\ 
                     & DL & 0.153 (0.338) & 0.194 (0.998) & 0.186 (0.929) &  0.185 (1)\\ 
\hline
\multirow{2}{*}{0.7} & MC & 0.143 (0.068) & 0.184 (0.264) & 0.164 (0.155) & 0.200 (0.542) \\ 
                     & DL & 0.156 (0.326) & 0.194 (0.996) & 0.192 (0.906) & 0.186 (1) \\ 
\hline
\end{tabular}
\end{table}

\subsection{Real data application}  \label{sec:real_data_setting}

In this section, we provide details of the real data results presented in Table \ref{tab:realdata} in the Introduction. 
The data set contains information of genetic mutations and the drug resistance for seven PI-type drugs: APV, ATV, IDV, LPV, NFV, RTV, and SQV. We will select the associated mutations for each drug. 
The design matrix $\bX = (\bX_{i,j})$ consists of binary entries and encodes the mutation information, where $\bX_{i, j} = 1$ if mutation $\# j $ is present in the $i$th observation and $\bX_{i, j} = 0$ otherwise. The response variable $Y_i$ is defined as the log-fold increase of lab-tested drug resistance in the $i$th sample.
Following the pre-processing steps in \cite{barber2015controlling}, for the data set of each drug, we remove the observations with missing $Y$ values and exclude the mutations (i.e., corresponding columns in the design matrix) that appear less than $3$ times in the entire sample. 
After the pre-processing procedure, the numbers of observations $(n)$ and genetic mutations $(p)$ vary across drugs: APV ($n = 767$, $p = 201$), ATV ($n = 328$, $p = 147$), IDV ($n = 825$, $p = 206$), LPV ($n = 515$, $p = 184$), NFV ($n = 842$, $p = 207$), RTV ($n = 793$, $p = 205$), and SQV ($n = 824$, $p = 206$).

Following \cite{barber2015controlling}, we adopt the mutation positions identified in the treatment-selected mutation (TSM) panel for the PI-type drugs as the groundtruth for calculating the FDR and power.
We subsample $200$ observations from the entire data set and run $100$ replications. It is seen that for most drugs we have dimension exceeding the sample size $n=200$. The approximate knockoffs procedure with second-order Gaussian knockoffs is applied to the subsampled data in each replication, and the empirical FDR and power are obtained by averaging over $100$ replications. The target FDR level is set as $q = 0.2$. 

Table \ref{tab:realdata} summarizes the results for the FDR and power, presenting the performance of three constructions for knockoff statistics: marginal correlation difference (MC), regression coefficient difference with debiased Lasso (DL), and distance correlation difference (DC). The first two knockoff statistics have been defined earlier and quantify exclusively the linear association, while distance correlation \citep{Szekely2007dcor, 
Gaoetal-distance-2021}  captures both linear and nonlinear dependence between variables. All three methods generally control the FDR at the target level $q =0.2$ in most cases. MC has slightly inflated FDR in a few cases, but the violation can be considered minor when accounting for the Monte Carlo error. DL controls the FDR in all cases but exhibits the lowest power among the three methods, while DC achieves the highest power.
These results suggest that the dependence between the drug resistance and associated mutations can be highly nonlinear. Overall, we observe from this real application that the approximate knockoffs procedure using Gaussian knockoffs with two-moment matching can achieve effective FDR control in practice. 
Finally, we emphasize that the relatively low power is also because of the small subsampling size of 200.

\section{Discussions} \label{new.sec.disc}
We have developed a unified theoretical framework for studying the asymptotic FDR control of the approximate knockoffs procedure that generates knockoff variables using a user-specified covariate distribution. We have focused on a widely used knockoff variable generation method: the Gaussian knockoffs generator based on the first two moments matching of covariates. With such Gaussian knockoff variables, using our unified theoretical framework, we have formally verified that some commonly used knockoff statistics, constructed based on matching the first two moments of data, can achieve asymptotic FDR control. While our formal results assume known first two population moments in generating Gaussian knockoff variables, our theory can be extended to scenarios when moments are estimated.

In Section \ref{sec:examples}, we have considered both light-tailed and heavy-tailed covariates, where for the latter,  we need to assume $p \ll \sqrt{n}$. We conjecture that if a robust construction of knockoff statistics (with respect to the tail distribution) is exploited, the dimensionality constraint can be relaxed. We leave this for future investigation. Based on our study, we recommend the following approach for practitioners using model-X knockoffs: (i) first select the knockoff statistic $\widehat{\bm W}_j$ that best suits the application, and (ii) if $\widehat{\bm W}_j$ depends on the first $m$ moments of data, generate the knockoff variables $\widehat{\bm X}$ using $m$-moments matching. The theoretical justification of such procedure is left for future research.

\spacingset{1.9}
\bibliographystyle{apalike}
\bibliography{mybib}

\newpage
\appendix
\setcounter{page}{1}
\setcounter{section}{0}
\renewcommand{\theequation}{A.\arabic{equation}}
\setcounter{equation}{0}

\begin{center}{\bf \Large  Supplementary Material to ``Asymptotic FDR Control with Model-X Knockoffs: Is Moments Matching Sufficient?''}

\bigskip

Yingying Fan, Lan Gao, Jinchi Lv and Xiaocong Xu
\end{center}

This Supplementary Material contains all the technical proofs and detailed derivations for the theoretical results. Section \ref{sec:proof_marginial_corr} presents the proofs for Section \ref{sec:marginial_corr} (marginal correlation), Section \ref{sec:proof_ols} presents the proofs for Section \ref{sec:rcd_ols} (ordinary least squares), and Section \ref{sec:proof_dl} presents the proofs for Section \ref{sec:debiased Lasso} (debiased Lasso), Sections \ref{new.sec3.4}, and Example \ref{exmp:discrete_covariate}. Section \ref{sec:mcks_correlated} contains the results for marginal correlation knockoff statistics with correlated features. Further, Section \ref{new.sec.addproandtechdet} includes additional proofs and some additional technical details.

\textit{Additional notation}. For any two integers $m,n$, let $[m:n]\equiv \{m,m+1,\ldots,n\}$. 
For $a \in \R$, let $a_\pm \equiv (\pm a)\vee 0$. For vectors $x = (x_i)$ and $y = (y_i) \in \R^n$, let $\iprod{x}{y}\equiv x^\top y =\sum_{i \in [n]} x_i y_i$. We use $\{e_k\}$ to denote the canonical basis of the Euclidean space, whose dimensionality should be self-clear from the context. For a random variable $X$, we use $\Prob_X,\E_X$ (respectively,  $\Prob^X,\E^X$) to indicate that the probability and expectation are taken with respect to $X$ (respectively, conditional on $X$).

\section{Proofs for Section \ref{sec:marginial_corr} (marginal correlation)} \label{sec:proof_marginial_corr}

Recall that in the case of marginal  correlation  knockoff statistics, for  any $j \in \mathcal{H}_0$ and $t \geq 0$, $\mathcal{P}_j(t) \equiv \Prob_{(\mathsf{G}_1, \mathsf{G}_2)^\top \sim \mathcal{N}(0, \sigma_j^2I_2   ) }(\abs{\mathsf{G}_1} - \abs{\mathsf{G}_2}  \geq t)$, and for any $x \in [0,1/2]$, $\mathcal{P}^{-1}(x) \equiv \sup \{t \geq 0 : \mathcal{P}(t) \geq x \}$ with $\mathcal{P}(t) \equiv \abs{\mathcal{H}_0}^{-1}\sum_{j \in \mathcal{H}_0} \mathcal{P}_j(t)$.

\subsection{Proof of Proposition \ref{prop:conc_knockoff}}


Note that from Lemma \ref{thm:conc_nonnegative}, we have 
\begin{align*}
	\Prob\Big( \mathcal{E}_1 \equiv \Big\{\pnorm{\bm{Y}}{}^2     \geq n\E Y^2/2  \Big\} \Big)  \geq 1 - \exp( -n \E^2 Y^2/C  ),
\end{align*}
where positive constant $C$ is a universal constant whose value may change from line to line hereafter. By the concentration of the norm for the sub-exponential random vector \cite[Proposition 2.2]{sambale2023some}, it holds that 
	\begin{align*}
		\Prob\Big( \mathcal{E}_2 \equiv \Big\{\bigabs{\pnorm{\bm{Y}}{} - \sqrt{n}\E^{1/2}Y^2  }  \leq C\log p \Big\} \Big)  \geq 1 - p^{-10}.
	\end{align*}
	Since $ \bm{X}_{i,j} \bm{Y}_i$ and $\widehat{\bm{X}}_{i,j} \bm{Y}_i $ are both $1/2$-sub-exponential random variables, we can again apply the concentration inequality for $1/2$-sub-exponential random variables, cf. \cite[Corollary 1.4]{gotze2021concentration}, to derive that for $\log p = \mathfrak{o}(\sqrt{n})$,
	\begin{align*}
		\Prob \Bigg( \mathcal{E}_3\equiv  \bigg\{ \bigabs{n^{-1}\bm{X}_{\cdot,j}^\top \bm{Y} - \E X_jY     } \vee  \bigabs{n^{-1}\widehat{\bm{X}}_{\cdot,j}^\top \bm{Y} - \E \widehat{X}_jY     }   \leq C\frac{\log p}{\sqrt{n} } \bigg\} \Bigg) \geq 1 - p^{-10}.
	\end{align*}
    
	Hence, on event $ \mathcal{E}_1 \cap \mathcal{E}_2 \cap \mathcal{E}_3 $, we can deduce that 
	\begin{align*}
		&\biggabs{\frac{\abs{ \bm{X}_{\cdot,j}^\top \bm{Y} }  }{\pnorm{\bm{Y}}{}} - \frac{\sqrt{n} \cdot \abs{\E X_j Y}  }{ \E^{1/2} Y^2 }} \\
		&=\biggabs{ \frac{n^{-1/2}\abs{ \bm{X}_{\cdot,j}^\top \bm{Y} } -   \sqrt{n} \cdot \abs{\E X_j Y}  }{ \E^{1/2} Y^2 }} + \abs{ \bm{X}_{\cdot,j}^\top \bm{Y} } \cdot  \frac{ \bigabs{  \pnorm{\bm{Y}}{} - \sqrt{n} \E^{1/2}Y^2   } }{\sqrt{n} \pnorm{\bm{Y}}{} \cdot \E^{1/2} Y^2 } \\
		&\leq  \frac{C\log p}{ \E^{1/2} Y^2} + (n\cdot \abs{\E X_jY} + C\sqrt{n}\log p) \cdot \frac{C\log p }{n \E Y^2 } \\
		&\leq C\log p \cdot \bigg( \frac{1}{\E^{1/2}Y^2} + \frac{\abs{\E X_jY} + \log p / \sqrt{n}}{\E Y^2} \bigg) \\
		&\leq C\log p \cdot \bigg( \frac{1}{\E^{1/2}Y^2} + \frac{\E^{1/2} X_j^2 \E^{1/2} Y^2 + \log p / \sqrt{n}}{\E Y^2} \bigg) \\
		&\leq C\log p \cdot \bigg( \frac{1}{\E^{1/2}Y^2} +  \frac{\log p}{\sqrt{n\E Y^2}}\bigg). 
	\end{align*}
    
	 Similarly, we can obtain that on event $ \mathcal{E}_1 \cap \mathcal{E}_2 \cap \mathcal{E}_3 $,
	\begin{align*}
	\frac{\abs{ \widehat{\bm{X}}_{\cdot,j}^\top \bm{Y} }  }{\pnorm{\bm{Y}}{}}  = 	\biggabs{\frac{\abs{ \widehat{\bm{X}}_{\cdot,j}^\top \bm{Y} }  }{\pnorm{\bm{Y}}{}} - \frac{\sqrt{n} \cdot \abs{\E \widehat{X}_j Y}  }{ \E^{1/2} Y^2 }}\leq C\log p \cdot \bigg( \frac{1}{\E^{1/2}Y^2} + \frac{\log p}{\sqrt{ n\E Y^2}}\bigg),
	\end{align*}
	where in the first step above, we have used the fact that for each $j \in \mathcal{H}_0$, $\E \widehat{X}_j Y = \E \widehat{X}_j \E Y =   0$. The conclusion now follows from an application of the triangle inequality, which completes the proof of Proposition \ref{prop:conc_knockoff}.

\subsection{Proof of Theorem \ref{thm:FDR_mcks}: light-tailed case}\label{sec:proof_FDR_mcks_light}

The results presented below are established under the assumption $\log p = \mathfrak{o}(n^c)$ for some $c \in (0,1/50)$, and this assumption will not be reiterated further. The proof of Theorem \ref{thm:FDR_mcks} for the light-tailed case relies on three key steps: the distributional approximation, the approximation for indicator functions, and the localization of $T_q$. These three steps verify the three conditions in Theorem \ref{thm:general_framwork}. We present the intermediate results for these steps here, while the proofs of these results are deferred to Section \ref{sec:proof_sec_FDR_mcks_light}.

\subsubsection*{Distributional approximation} 
We will first establish the asymptotic approximate symmetry property (i.e., condition (1) in Theorem \ref{thm:general_framwork}) of $\widehat{\bm{W}}_j$ for $j \in \mathcal{H}_0$.

\begin{theorem} \label{thm:moder_dis_mcks}
Assume that conditions \eqref{cond:C1} and \eqref{cond:C2} are satisfied. For any $ j \in \mathcal{H}_0 $, it holds that for $\mathsf{W} \in \{ \widehat{\bm{W}}, -\widehat{\bm{W}} \}$, there exists some constant $C = C(\mathsf{M}) > 0$ such that for $n \wedge p \geq C$ and $t \in (0,n^{1/7}/C )$,
    \begin{align*}
       \biggabs{ \frac{\Prob \big( \mathsf{W}_j \geq t \big)}{\mathcal{P}_j (t )  } - 1} \leq C \bigg( \frac{(1+t^3) \log^4(np)}{\sqrt{n}} + \frac{p^{-10}}{\mathcal{P}_j (t )  } \bigg).
    \end{align*}
\end{theorem}


\begin{remark}
	The probability estimate $p^{-10}$ above can be replaced with $p^{-D}$ for any $D > 0$. The cost will be a possibly enlarged constant $C$ that depends further on $D$.  This convention applies to other statements in the following sections where the probability estimates $p^{-10}$, $n^{-10}$ appear.
\end{remark}

An important consequence of the moderate deviation result in Theorem \ref{thm:moder_dis_mcks} above is that both $\Prob(\widehat{\bm{W}}_j \geq t )$ and $\Prob(\widehat{\bm{W}}_j \leq -t )$ are approximated by the same quantity $\mathcal{P}_j(t)$. This suggests that while $\widehat{\bm{W}}_j$ is not perfectly symmetric, it exhibits the asymptotic approximate symmetry over a broad range, as formalized in the corollary below.


\begin{corollary}\label{cor:ratio_G} 
Assume that conditions \eqref{cond:C1} and \eqref{cond:C2} hold. Let $\alpha_n \leq 1/2$ satisfy $\mathcal{P}^{-1}( \alpha_n ) = \mathfrak{o}(n^{1/7})$ and $p^{10}\alpha_n \to \infty$ as $n \wedge p \to \infty$. Then we have
		\begin{align*}
			\sup_{t \in (0, \mathcal{P}^{-1}( \alpha_n )  )}\biggabs{\frac{\sum_{j \in \mathcal{H}_0} \Prob \big( \widehat{\bm{W}}_j \geq t  \big)}{\sum_{j \in \mathcal{H}_0} \Prob \big( \widehat{\bm{W}}_j \leq -t   \big)} - 1}= \mathfrak{o}(1).
		\end{align*}
\end{corollary}

Corollary \ref{cor:ratio_G} above states that $\widehat{\bm{W}}_j$ admits the asymptotic approximate symmetry for $t \in  (0, \mathcal{P}^{-1}( \alpha_n ) )$. The upper bound for $t$ is carefully chosen to ensure that $p^{-10}/\mathcal{P}(t) = \mathfrak{o}(1)$ within this range. As will be clear later, under additional conditions \eqref{cond:C3} and \eqref{cond:C4}, it can be shown that $T_q$ is localized within this domain for some $\alpha_n \asymp  a_n /p$ with asymptotic probability $1$. Consequently, this range is sufficient for achieving the goal of asymptotic FDR control.

\subsubsection*{Approximation for indicator functions}
As outlined in the general framework for the asymptotic FDR control in Theorem \ref{thm:general_framwork}, we will approximate the indicator functions using the tail probabilities of the knockoff statistics. This approximation then enables the use of Corollary \ref{cor:ratio_G} to prove asymptotic FDR control. The lemma below establishes the validity of this approximation.

\begin{lemma}\label{lem:indicator_app}
Assume that conditions \eqref{cond:C1}--\eqref{cond:C3} hold. Let $\alpha_n \leq 1/2$ satisfy $\mathcal{P}^{-1}( \alpha_n ) = \mathfrak{o}(n^{1/7})$ and $p\alpha_n \to \infty$ as $n \wedge p \to \infty$. Then we have that for $\# \in \{ +,- \}$, 
    \begin{align*}
        \sup_{t \in (0, \mathcal{P}^{-1}( \alpha_n )  )}\biggabs{\frac{ \sum_{j \in \mathcal{H}_0}  \bm{1}\{ \# \widehat{\bm{W}}_j  \geq t\}  }{\sum _{j \in \mathcal{H}_0}  \Prob (\# \widehat{\bm{W}}_j   \geq t)  } - 1 } = \mathfrak{o}_\mathbf{P}(1).
    \end{align*}
\end{lemma}

The proof of Lemma \ref{lem:indicator_app} above builds upon the approach used in \cite[Lemma 3]{fan2023ark}, which focuses on bounding $\var (\sum_{j \in \mathcal{H}_0}  \bm{1}\{ \# \widehat{\bm{W}}_j  \geq t\})$.
At a high level, this involves exploring the weak dependency between $\widehat{\bm{W}}_j$ and $\widehat{\bm{W}}_\ell$ for $j \neq \ell$. A natural strategy is to approximate $\widehat{\bm{W}}_j$ with statistics that have explicit distributional properties, making the analysis more tractable. While \cite{fan2023ark} frequently employed coupling to replace $\Prob(\#\widehat{\bm{W}}_j \geq t)$ with the tail probability of the perfect knockoff statistic, a \textit{critical} technical modification here is the approximation of $\Prob(\#\widehat{\bm{W}}_j \geq t)$ by $\mathcal{P}_j(t)$. This shifts the focus from analyzing the weak dependency of $\widehat{\bm{W}}_j$ and $\widehat{\bm{W}}_\ell$ to studying the weak dependency between Gaussian random variables, which facilitates obtaining the desired result.

\subsubsection*{Localization of $T_q$}

\begin{lemma}\label{lem:local_T}
Assume that conditions \eqref{cond:C1}--\eqref{cond:C4} hold. Then for any $q \in (0,1)$, if $\log p = \mathfrak{o}(n^{c})$ for some small enough constant $c > 0$, we have that $ \mathcal{P}^{-1}\big(\frac{q  a_n }{2p} \big) = \mathcal{O}(\log p )$ and 
    \begin{align*}
         \Prob \bigg( {T}_q > \mathcal{P}^{-1} \bigg( \frac{q  a_n }{2p}\bigg)  \bigg)  = \mathfrak{o}(1).
    \end{align*}
\end{lemma}

We are now ready to prove Theorem \ref{thm:FDR_mcks} for the light-tailed case. 
    Building on the asymptotic approximate symmetry for $\widehat{\bm{W}}_j$ (cf. Corollary \ref{cor:ratio_G}), the approximation of indicator functions (cf. Lemma \ref{lem:indicator_app}), and the localization of $T_q$ (cf. Lemma \ref{lem:local_T}), the three conditions in Theorem \ref{thm:general_framwork} follow immediately, and thus the asymptotic FDR control is proved. This completes the proof of Theorem \ref{thm:FDR_mcks} for the light-tailed case.

\subsection{Proof of Theorem \ref{thm:FDR_mcks}: heavy-tailed case}\label{sec:proof_FDR_mcks_heavy}
As in the light-tailed case, we present the intermediate results for the distributional approximation, the approximation for indicator functions, and the localization of $T_q$ here. These results verify the three conditions in Theorem \ref{thm:general_framwork} and thus establish the asymptotic FDR control. The proofs of these intermediate results are deferred to Section \ref{sec:proof_sec_FDR_mcks_heavy}.

\subsubsection*{Distributional approximation}
\begin{theorem}\label{thm:dis_mcks_be}
Assume that conditions \eqref{cond:C1'} and \eqref{cond:C2} hold for some $\mathfrak{q} \geq 3$. For any $ j \in \mathcal{H}_0 $, it holds that for $\mathsf{W} \in \{ \widehat{\bm{W}}, -\widehat{\bm{W}} \}$, there exists some constants $C = C(\mathsf{M},\mathfrak{q}) > 0$ such that for $n \geq C$,
    \begin{align*}
        \sup_{t \in \R} \bigabs{ \Prob \big( \mathsf{W}_j \geq t \big) - \mathcal{P}_j(t)}  \leq C\frac{\log n}{\sqrt{n}} .
    \end{align*}
\end{theorem}

An immediate corollary is the following asymptotic symmetric property of $\widehat{\bm{W}}_j$'s.

\begin{corollary}\label{cor:ratio_G_be}
Assume that conditions \eqref{cond:C1'} and \eqref{cond:C2} hold for some $\mathfrak{q} \geq 3$. Then for any $ j \in \mathcal{H}_0 $, there exists some constant $C = C(\mathsf{M},\mathfrak{q}) > 0$ such that for $n \geq C$ and any $\alpha_n \in (0,1/2]$,
    \begin{align*}
        \sup_{t \in (0, \mathcal{P}^{-1}(\alpha_n) )} \biggabs{ \frac{\Prob (\widehat{\bm{W}}_j \geq t ) }{\Prob (\widehat{\bm{W}}_j \leq -t )} - 1 } \leq C\frac{\log n}{\sqrt{n}\alpha_n }.
    \end{align*}
\end{corollary}

\subsubsection*{Approximation for indicator functions and localization of $T_q$}
Recall that for any $j \in [p]$, $\widehat{\bm{w}}_j \equiv  \sqrt{n}\abs{\E X_j Y} /\E^{1/2}Y^2$. We now present a weaker concentration result for the knockoff statistics, in comparison to Proposition \ref{prop:conc_knockoff}.
\begin{proposition}\label{prop:conc_knockoff_heavy}
Assume that conditions \eqref{cond:C1'} and \eqref{cond:C2} hold for some $\mathfrak{q} \geq 3$. Then there exists some constant $C = C(\mathsf{M},\mathfrak{q}) > 0$ such that for any $j \in [p]$,
	    \begin{align*}
	    	\Prob \Big(  \bigabs{ \widehat{\bm{W}}_j - \widehat{\bm{w}}_j } \geq C\delta_{n;\mathfrak{q}} \Big) \leq   p^{-2} + \exp ( -n  \E^2 Y^2/C ),
	    \end{align*}
	    where $\delta_{n;\mathfrak{q}} \equiv \frac{p^{2/\mathfrak{q}}\log p}{\E^{1/2}Y^2} +  \frac{p^{2/\mathfrak{q}}\log^2 p}{\sqrt{n\E Y^2}} $.
\end{proposition}

As discussed in Section \ref{sec:general_framework}, achieving the FDR control relies critically on the localization of $T_q$. 
To maintain this localization, we need the signal strength to increase as the tails of the entries in the knockoffs data matrix become heavier, as formalized in condition \eqref{cond:C3'}.

\begin{lemma}\label{lem:T_qandIndicator_heavy}
Assume that conditions \eqref{cond:C1'} and \eqref{cond:C3'} hold for some $\mathfrak{q} \geq 3$, and conditions \eqref{cond:C2} and \eqref{cond:C4} hold. Then for any $q \in (0,1)$, if $p\log n = \mathfrak{o}(\sqrt{n})$ we have that 
    \begin{align}\label{ineq:T_local_heavy}
                  \Prob \bigg( {T}_q > \mathcal{P}^{-1} \bigg( \frac{q  a_n }{2p}\bigg)  \bigg)  = \mathfrak{o}(1),
    \end{align}
    and for $\# \in \{ +,- \}$, 
    \begin{align}\label{ineq:indicator_app_heavy}
        \sup_{t \in \big(0, \mathcal{P}^{-1} \big( \frac{q  a_n }{2p}\big)\big)}\biggabs{\frac{ \sum_{j \in \mathcal{H}_0}  \bm{1}\{ \# \widehat{\bm{W}}_j  \geq t\}  }{\sum _{j \in \mathcal{H}_0}  \Prob (\# \widehat{\bm{W}}_j   \geq t)  } - 1 } = \mathfrak{o}_\mathbf{P}(1).
    \end{align}
\end{lemma}

The proof of Lemma \ref{lem:T_qandIndicator_heavy} above follows similar arguments as for the light-tailed case. For brevity, we omit the full proof and instead provide a proof sketch in Section \ref{sec:proof_T_qandIndicator_haevy}.

Finally, we can prove Theorem \ref{thm:FDR_mcks} for the heavy-tailed case. 
 The asymptotic FDR control follows from Theorem \ref{thm:general_framwork}, upon noting that the asymptotic approximate symmetry for $\widehat{\bm{W}}_j$, the approximation of indicator functions, and the localization of $T_q$ are satisfied; see Corollary \ref{cor:ratio_G_be} and Lemma \ref{lem:T_qandIndicator_heavy}. This completes the proof of Theorem \ref{thm:FDR_mcks} for the heavy-tailed case.

\section{Proofs for Section \ref{sec:rcd_ols} (ordinary least squares)} \label{sec:proof_ols}

Recall that for knockoff statistics based on the differences in regression coefficients estimated using the OLS estimator, for  any $j \in \mathcal{H}_0$ and $t \geq 0$, 
\begin{align*}
    \mathcal{P}_j(t) \equiv \Prob_{(\mathsf{G}_1,\mathsf{G}_2)^\top \sim \mathcal{N} (0, \sigma_\xi^2\Sigma_n^{(j)})} \big(\abs{\mathsf{G}_1} - \abs{\mathsf{G}_2} \geq t\big),
\end{align*}
where $\Sigma_n^{(j)} \equiv (1 - 2p/n)^{-1}[\Sigma^{-1}]_{\{j, j+p\}, \{j, j+p\}}$.
For any $x \in [0,1/2]$, $\mathcal{P}^{-1}(x) \equiv \sup \{t \geq 0 : \mathcal{P}(t) \geq x \}$ with $\mathcal{P}(t) \equiv \abs{\mathcal{H}_0}^{-1}\sum_{j \in \mathcal{H}_0} \mathcal{P}_j(t)$.

For notational simplicity, let us write $\widehat{\mathsf{Z}} \equiv \widehat{\bm{Z}}/\sqrt{n}$,  $\widehat{\mathsf{Q}} \equiv \widehat{\bm{Q}}/\sqrt{n} $. By introducing the augmented coefficient vector $\beta_\ast^{\mathsf{au}} = (\beta_\ast^\top, 0_{p}^\top )^\top \in \R^{2p}$, we can rewrite that
\begin{align}\label{def:ols_alternative}
	\widehat{\beta}_j^{\mathsf{LS}} = \sqrt{n} \beta_{\ast,j}^{\mathsf{au}} + e_j^\top  (\widehat{\mathsf{Z}}^\top \widehat{\mathsf{Z}} )^{-1} \widehat{\mathsf{Z}}^\top \bm{\xi}, \quad j \in [2p].
\end{align}
We will mostly use this form in the technical analysis below.


\subsection{Proof of Proposition \ref{prop:conc_knockoff_ols}}
We will apply the following form of Hanson--Wright inequality, taken from \cite{Rudelson2013} and \cite{Adamczak2015}.
\begin{lemma}\label{lem:Hanson-Wright}
Assume that conditions \eqref{cond:O2}--\eqref{cond:O3} hold.  Then there exists a universal constant $C > 0$  such that for any matrix $A \in \R^{2p \times 2p}$, $i \in [n]$, and $t > 0$,
     \begin{align*}
         \Prob \Big( \bigabs{\widehat{\bm{Z}}_{i,\cdot}^\top A  \widehat{\bm{Z}}_{i,\cdot} - \tr ( \Sigma A ) } \geq t \Big) \leq 2\exp \bigg\{ -C\min\bigg( \frac{t^2}{\mathsf{c}_\mathsf{L}^4 \pnorm{\Sigma^{1/2}A\Sigma^{1/2} }{F}^2 }, \frac{t}{ \mathsf{c}_\mathsf{L}^2\pnorm{\Sigma^{1/2}A\Sigma^{1/2} }{\op}}   \bigg) \bigg\}
     \end{align*}
\end{lemma}

The following result is a direct consequence of \cite[Theorem 3.6]{Yinuniversality}.
\begin{lemma}\label{lem:rigidity}
Assume that conditions \eqref{cond:O1}--\eqref{cond:O3} hold. Then there exists some constant $C = C(\mathsf{M},\tau,\mathsf{c}_{\mathsf{L}}) > 0$ such that for any $n \geq C$, 
\begin{align*}
    \Prob \Big( C^{-1} \leq \lambda_{\min} (\widehat{\mathsf{Z}}^\top \widehat{\mathsf{Z}}  ) \leq \lambda_{\max} (\widehat{\mathsf{Z}}^\top \widehat{\mathsf{Z}}  ) \leq C  \Big) \geq 1 - Cn^{-10}.
\end{align*}
\end{lemma}

\begin{lemma}\label{lem:noise_ols}
Assume that conditions \eqref{cond:O1}--\eqref{cond:O3} hold.  Then there exists some constant $C = C(\mathsf{M},\tau,\mathsf{c}_{\mathsf{L}}) > 0$ such that for any $p \geq C$, 
    \begin{align*}
        \Prob\Big( \bigabs{ e_j^\top (\widehat{\mathsf{Z}}^\top \widehat{\mathsf{Z}})^{-1} \widehat{\mathsf{Z}}^\top \bm{\xi}   } \geq C\sqrt{\log p}  \Big) \leq Cp^{-10}, \quad \forall j \in [2p].
    \end{align*}
\end{lemma}

\noindent\textit{Proof of Lemma \ref{lem:noise_ols}}. 
Observe that by the general Hoeffding's inequality, it holds that for some $c = c(\mathsf{M}) > 0$,
    \begin{align*}
        \Prob^{\widehat{\mathsf{Z}}} \Big( \bigabs{ e_j^\top (\widehat{\mathsf{Z}}^\top \widehat{\mathsf{Z}})^{-1} \widehat{\mathsf{Z}}^\top \bm{\xi}   } \geq t  \Big) &\leq 2\exp \Bigg( -\frac{ct^2}{e_j^\top(\widehat{\mathsf{Z}}^\top \widehat{\mathsf{Z}})^{-1}e_j } \Bigg).
    \end{align*}
    An application of Lemma \ref{lem:rigidity} and the law of total probability leads to 
    \begin{align*}
         \Prob\Big( \bigabs{e_j^\top (\widehat{\mathsf{Z}}^\top \widehat{\mathsf{Z}})^{-1} \widehat{\mathsf{Z}}^\top \bm{\xi}   } \geq t \Big) &\leq \E_{\widehat{\mathsf{Z}}}\Big[ \Prob^{\widehat{\mathsf{Z}}} \Big( \bigabs{ e_j^\top (\widehat{\mathsf{Z}}^\top \widehat{\mathsf{Z}})^{-1} \widehat{\mathsf{Z}}^\top \bm{\xi}   } \geq t  \Big) \bm{1}\{ \lambda_{\min}(\widehat{\mathsf{Z}}^\top \widehat{\mathsf{Z}}) \geq C^{-1}  \} \Big]  \\
         &\quad + \E_{\widehat{\mathsf{Z}}}\Big[ \Prob^{\widehat{\mathsf{Z}}} \Big( \bigabs{ e_j^\top (\widehat{\mathsf{Z}}^\top \widehat{\mathsf{Z}})^{-1} \widehat{\mathsf{Z}}^\top \bm{\xi}   } \geq t  \Big) \bm{1}\{ \lambda_{\min}(\widehat{\mathsf{Z}}^\top \widehat{\mathsf{Z}}) < C^{-1}  \} \Big]\\
         &\leq 2\exp \bigg( -\frac{ct^2}{C} \bigg) + Cn^{-10},
    \end{align*}
    where $C$ is the constant in Lemma \ref{lem:rigidity}. The claim follows by setting $t = C'\sqrt{\log p}$ for some large enough $C' = C'(\mathsf{M},\tau,\mathsf{c}_{\mathsf{L}}) > 0$ and adjusting constants, which completes the proof of Lemma \ref{lem:noise_ols}.

    Recall from (\ref{def:ols_alternative}) that for $j \in [2p]$,
    \begin{align*}
    	\widehat{\beta}_j^{\mathsf{LS}} = \sqrt{n} \beta_{\ast,j}^{\mathsf{au}} + e_j^\top  (\widehat{\mathsf{Z}}^\top \widehat{\mathsf{Z}} )^{-1} \widehat{\mathsf{Z}}^\top \bm{\xi}.
    \end{align*}
    Using the facts that $\beta_{\ast,j}^{\mathsf{au}} =  \beta_{\ast, j}$ for $j \in [p]$ and $\beta_{\ast,j}^{\mathsf{au}} = 0$ for $j \in [p+1:2p]$, we can rewrite the knockoff statistics $\widehat{\bm{W}}_j$ as 
    \begin{align*}
    	\widehat{\bm{W}}_j =  \abs{\beta_{\ast, j} + e_j^\top  (\widehat{\mathsf{Z}}^\top \widehat{\mathsf{Z}} )^{-1} \widehat{\mathsf{Z}}^\top \bm{\xi} } - \abs{ e_{j+p}^\top  (\widehat{\mathsf{Z}}^\top \widehat{\mathsf{Z}} )^{-1} \widehat{\mathsf{Z}}^\top \bm{\xi}}.
    \end{align*}
    Therefore, we can obtain that 
    \begin{align*}
    	\bigabs{\widehat{\bm{W}}_j - \abs{\beta_{\ast, j}}  } \leq \abs{e_j^\top  (\widehat{\mathsf{Z}}^\top \widehat{\mathsf{Z}} )^{-1} \widehat{\mathsf{Z}}^\top \bm{\xi} } + \abs{e_{j+p}^\top  (\widehat{\mathsf{Z}}^\top \widehat{\mathsf{Z}} )^{-1} \widehat{\mathsf{Z}}^\top \bm{\xi} }.
    \end{align*}
    The conclusion follows immediately from Lemma \ref{lem:noise_ols}, which completes the proof of Proposition \ref{prop:conc_knockoff_ols}.

\subsection{Proof of Theorem \ref{thm:FDR_ols}}\label{sec:proof_sec_thm:FDR_ols}
We present the intermediate results for the distributional approximation, the approximation for indicator functions, and the localization of $T_q$ here, while the proofs of these results are deferred to Section \ref{sec:proof_sec_sec_thm:FDR_ols}.

\subsubsection*{Distributional approximation} 
Recall that for any $t >0$, $\mathcal{P}(t) = \abs{\mathcal{H}_0}^{-1}\sum_{j \in \mathcal{H}_0} \mathcal{P}_j(t)$,  and for any $x \in [0,1/2]$, its inverse function  $\mathcal{P}^{-1}(x) =\sup \{t \geq 0 : \mathcal{P}(t) \geq x \}$.

\begin{theorem}\label{thm:moder_dis_ols}
Assume that conditions \eqref{cond:O1}--\eqref{cond:O3} hold. For any $ j \in \mathcal{H}_0 $, it holds that for $\mathsf{W} \in \{ \widehat{\bm{W}}, -\widehat{\bm{W}} \}$, there exists some constants $C = C(\mathsf{M},\tau,\mathsf{c}_\mathsf{L})>0 $ such that for $n \geq C$ and $t \in (0,n^{1/7}/C)$,
    \begin{align*}
     \biggabs{ \frac{\Prob \big( \mathsf{W}_j \geq t  \big)}{\mathcal{P}_j(t ) } -1 } \leq C \bigg(\frac{(1+t^3)\log^3 n}{\sqrt{n} }+ \frac{p^{-10}}{\mathcal{P}_j(t )} \bigg).
\end{align*}
\end{theorem}

The asymptotic approximate symmetry for $\widehat{\bm{W}}_j$ follows directly from Theorem \ref{thm:moder_dis_ols} above.

\begin{corollary}\label{cor:ratio_G_ols} 
Assume that conditions \eqref{cond:O1}--\eqref{cond:O3} hold. Let $\alpha_n \leq 1/2$ satisfy $\mathcal{P}^{-1}( \alpha_n ) = \mathfrak{o}(n^{1/7})$ and $p^{10}\alpha_n \to \infty$ as $n \to \infty$. Then we have
		\begin{align*}
			\sup_{t \in (0, \mathcal{P}^{-1}( \alpha_n )  )}\biggabs{\frac{\sum_{j \in \mathcal{H}_0} \Prob \big( \widehat{\bm{W}}_j \geq t  \big)}{\sum_{j \in \mathcal{H}_0} \Prob \big( \widehat{\bm{W}}_j \leq -t   \big)} - 1}= \mathfrak{o}(1).
		\end{align*}
\end{corollary}

\subsubsection*{Approximation for indicator functions} 
\begin{lemma}\label{lem:indicator_app_ols}
Assume that conditions \eqref{cond:O1}--\eqref{cond:O3} hold. Let $\alpha_n \leq 1/2$ satisfy $\mathcal{P}^{-1}( \alpha_n ) = \mathfrak{o}(n^{1/7})$ and $p\alpha_n/m_n \to \infty$ as $n \to \infty$. Then we have that for $\# \in \{ +,- \}$, 
    \begin{align*}
        \sup_{t \in (0, \mathcal{P}^{-1}(\alpha_n ) )}\biggabs{\frac{ \sum_{j \in \mathcal{H}_0}  \bm{1}\{ \# \widehat{\bm{W}}_j  \geq t\}  }{\sum _{j \in \mathcal{H}_0}  \Prob (\# \widehat{\bm{W}}_j   \geq t)  } - 1 } = \mathfrak{o}_\mathbf{P}(1).
    \end{align*}
\end{lemma}

Under conditions \eqref{cond:O4} and \eqref{cond:O5} on the signal strength, as we will show later, $T_q$ can be localized in the interval $(0, \mathcal{P}^{-1}(\alpha_n ))$ with $\alpha_n \asymp  a_n /p$. In this case, condition $p\alpha_n/m_n \to \infty$ simplifies to $m_n / a_n \to 0$, which is ensured by condition \eqref{cond:O4}. 

\subsubsection*{Localization of $T_q$}

\begin{lemma}\label{lem:local_T_ols}
Assume that conditions \eqref{cond:O1}--\eqref{cond:O5} hold. Then for any $q \in (0,1)$, we have that $ \mathcal{P}^{-1}\big(\frac{q  a_n }{2p} \big) = \mathcal{O}(\log^{1/2} p )$ and 
    \begin{align*}
         \Prob \bigg( {T}_q > \mathcal{P}^{-1} \bigg( \frac{q  a_n }{2p}\bigg)  \bigg)  = \mathfrak{o}(1).
    \end{align*}
\end{lemma}

Finally, we are ready to prove Theorem \ref{thm:FDR_ols}. 
 The asymptotic FDR control follows from Theorem \ref{thm:general_framwork}, upon noting that the asymptotic approximate symmetry for $\widehat{\bm{W}}_j$ (cf. Corollary \ref{cor:ratio_G_ols}), the approximation of indicator functions (cf. Lemma \ref{lem:indicator_app_ols}), and the localization of $T_q$ (cf. Lemma \ref{lem:local_T_ols}) are satisfied. This completes the proof of Theorem \ref{thm:FDR_ols}.

\section{Proofs for Section \ref{sec:debiased Lasso} (debiased Lasso), Sections \ref{new.sec3.4}, and Example \ref{exmp:discrete_covariate}} \label{sec:proof_dl}

We present the intermediate results for the distributional approximation, the approximation for indicator functions, and the localization of $T_q$ here, while the proofs of these results are deferred to Section \ref{sec:proof_sec_sec_thm:FDR_dl}.

\subsection{Proof of Theorem \ref{thm:FDR_dl}}\label{sec:proof_thm:FDR_dl}

Recall that for knockoff statistics based on the differences in regression coefficients estimated using the debiased Lasso estimator, for any $j \in [p]$ and $t > 0$, 
\begin{align*}
    \mathcal{P}_j(t ) &\equiv \Prob_{(\mathsf{G}_1,\mathsf{G}_2)^\top \sim \mathcal{N} (0, \sigma_\xi^2\Sigma_n^{(j)}) }\big(\abs{\mathsf{G}_1} - \abs{\mathsf{G}_2} \geq t \big)
\end{align*}
with  $\Sigma^{(j)} \equiv [ \Sigma^{-1}]_{\{j, j+p\} , \{j, j+p\}} $. It follows from (\ref{eq:invar_swap}) that
\begin{align*}
	\mathcal{P}_j(t ) =\Prob_{(\mathsf{G}_1,\mathsf{G}_2)^\top \sim \mathcal{N} (0, \sigma_\xi^2\Sigma_n^{(j)}) }\big(\abs{\mathsf{G}_1} - \abs{\mathsf{G}_2} \leq -t \big).
\end{align*} 
Recall also that $a_n = \#\{j \in \mathcal{H}_1: \abs{\beta_{\ast,j}} \gg \sqrt{\log p}  \} $.

\subsubsection*{Distributional approximation} 
The same as for the case of the OLS estimator, the key to proving the asymptotic FDR control is the following moderate deviation result.

\begin{theorem}\label{thm:moder_dis_dl}
Assume that conditions \eqref{cond:O3} and \eqref{cond:L1} hold. For any $ j \in \mathcal{H}_0 $, it holds that for $\mathsf{W} \in \{ \widehat{\bm{W}}, -\widehat{\bm{W}} \}$, there exists some constant $C = C(\mathsf{M},\mathsf{C}_1,\mathsf{C}_2,\mathsf{C}_3)>0 $ such that for $n \wedge p \geq C$ and $t \in (0, n^{1/6}/(Cb_ns\log^4 (np) )^{1/3}  )$,
    \begin{align}\label{ineq:moder_dis_dl}
     \bigabs{ \Prob \big( \mathsf{W}_j \geq t  \big) -  \mathcal{P}_j(t ) }\leq C \bigg(\bigg(\frac{(1 + t^3)b_ns\log^3 (np) }{\sqrt{n}} \bigg)\cdot \mathcal{P}_j(t )   + p^{-8} \bigg).
\end{align}
    Consequently, for any $q \in (0,1)$, if 
   \begin{align*}
         \bigg(1 \vee \mathcal{P}^{-1}\bigg(\frac{qa_n}{2p}\bigg)\bigg)^3 \cdot \frac{b_ns\log^4 (np) }{\sqrt{n}} = \mathfrak{o}(1),
     \end{align*}
    we have
		\begin{align}\label{ineq:app_sym_dl}
			\sup_{t \in \big(0, \mathcal{P}^{-1}\big(\frac{qa_n}{2p}\big) \big)}\biggabs{\frac{\sum_{j \in \mathcal{H}_0} \Prob \big( \widehat{\bm{W}}_j \geq t  \big)}{\sum_{j \in \mathcal{H}_0} \Prob \big( \widehat{\bm{W}}_j \leq -t   \big)} - 1}= \mathfrak{o}(1).
		\end{align}
\end{theorem}

\subsubsection*{Approximation for indicator functions}
Recall that for any $j \in [p]$, 
\begin{align*}
    N(j) = \bigg\{\ell \in \mathcal{H}_0 \setminus\{j\}: \sum_{\mathfrak{j} \in \{j, j + p\} } \big(\abs{ [\Sigma^{-1}]_{\mathfrak{j}, \ell }} + \abs{ [\Sigma^{-1}]_{\mathfrak{j}, \ell+p }}\big) \neq 0  \bigg\}.
\end{align*}
It is easy to see that $\max_{j \in \mathcal{H}_0} \abs{N(j) } \leq b_n$. 

\begin{lemma}\label{lem:indicator_app_dl}
Assume that conditions \eqref{cond:O3} and \eqref{cond:L1} hold.  Then for any $q \in (0,1)$, if
     \begin{align*}
         \bigg(1 \vee \mathcal{P}^{-1}\bigg(\frac{qa_n}{2p}\bigg)\bigg)^3 \cdot \frac{b_ns\log^4 (np) }{\sqrt{n}} = \mathfrak{o}(1),
     \end{align*}
     we have that for $\# \in \{ +,- \}$, 
    \begin{align*}
        \sup_{t \in \big(0, \mathcal{P}^{-1}\big(\frac{qa_n}{2p}\big) \big)}\biggabs{\frac{ \sum_{j \in \mathcal{H}_0}  \bm{1}\{ \# \widehat{\bm{W}}_j  \geq t\}  }{\sum _{j \in \mathcal{H}_0}  \Prob (\# \widehat{\bm{W}}_j   \geq t)  } - 1 } = \mathfrak{o}_\mathbf{P}(1).
    \end{align*}
\end{lemma}

The proof of Lemma \ref{lem:indicator_app_dl} above follows similar arguments as those for Lemma \ref{lem:indicator_app_ols}, with a book-keeping calculation on the distributional approximation error. We will provide some details in Section \ref{sec:proof_lem:indicator_app_dl}.

\subsubsection*{Localization of $T_q$}
We first record the following concentration result for $\widehat{\bm{W}}_j$'s in \cite[Lemma 12]{fan2023ark}.

\begin{lemma}
Assume that conditions \eqref{cond:O3} and \eqref{cond:L1} hold. Then if $s \sqrt{b_n \log p / n} = \mathfrak{o}(1)$, there exists some constant $C = C(\mathsf{M},\mathsf{C}_1, \mathsf{C}_2, \mathsf{C}_3) > 0$ such that   
     \begin{align*}
       \sum_{j \in [p]} \Prob \Big(  \bigabs{ \widehat{\bm{W}}_j -  \abs{\beta_{\ast,j}}  } \geq C\sqrt{\log p} \Big) = \mathfrak{o}(1). 
    \end{align*}
\end{lemma}

With the above concentration on $\widehat{\bm{W}}_j$'s, an argument 
as in the proof of Lemma \ref{lem:local_T_ols} establishes the following localization of $T_q$.

\begin{lemma}\label{lem:local_T_dl}
Assume that conditions \eqref{cond:O3}--\eqref{cond:O5} and \eqref{cond:L1} hold. Then for any $q \in (0,1)$, if $s \sqrt{b_n \log p / n} = \mathfrak{o}(1)$, we have $ \mathcal{P}^{-1}\big(\frac{q  a_n }{2p} \big) = \mathcal{O}(\log^{1/2} p )$. Furthermore, under the additional assumption that $b_ns\log^7 (np)/\sqrt{n} = \mathfrak{o}(1) $,
     we have 
    \begin{align*}
         \Prob \bigg( {T}_q > \mathcal{P}^{-1} \bigg( \frac{q  a_n }{2p}\bigg)  \bigg)  = \mathfrak{o}(1).
    \end{align*}
\end{lemma}

Finally, we are able to prove Theorem \ref{thm:FDR_dl}. 
 The asymptotic FDR control follows from Theorem \ref{thm:general_framwork}, upon noting that the asymptotic approximate symmetry for $\widehat{\bm{W}}_j$ (cf. Theorem \ref{thm:moder_dis_dl}), the approximation of indicator functions (cf. Lemma \ref{lem:indicator_app_dl}), and the localization of $T_q$ (cf. Lemma \ref{lem:local_T_dl}) are satisfied. This completes the proof of Theorem \ref{thm:FDR_dl}.

\subsection{Proof of Proposition \ref{prop:couple}}
 For any $j \in [p]$, we can deduce that
    \begin{align*}
        \pnorm{ \widehat{\bm{X}}^{\mathsf{in}}_{\cdot,j} - \widehat{\bm{X}}_{\cdot,j}}{} &\leq  \pnorm{\bm{r}}{\infty} \cdot \pnorm{ \bm{X}( \widehat{\bm{\Omega}} - \Sigma_X^{-1} )e_j}{}  + \pnorm{\bm{\mathsf{Z}} ({\bm{S}}_Z^{1/2} -  \widehat{\bm{S}}_Z^{1/2} )e_j}{}.
    \end{align*}
    By the standard concentration inequality for chi-square random variables, it holds that for any $i \in [n]$,
    \begin{align*}
        \Prob_{\bm{\mathsf{Z}}}\big(  n^{-1}\pnorm{\bm{\mathsf{Z}} ({\bm{S}}_Z^{1/2} -  \widehat{\bm{S}}_Z^{1/2} )e_j}{}^2 \geq  3 \pnorm{({\bm{S}}_Z^{1/2} -  \widehat{\bm{S}}_Z^{1/2} )e_j  }{ }^2/2  \big) \leq  2\exp \bigg( -\frac{n}{32} \bigg).
    \end{align*}
    An application of Lemma \ref{lem:matrix_ineq1} yields that 
    \begin{align*}
       \pnorm{({\bm{S}}_Z^{1/2} -  \widehat{\bm{S}}_Z^{1/2} )e_j  }{ } &\leq \pnorm{ {\bm{S}}_Z^{1/2} -  \widehat{\bm{S}}_Z^{1/2} }{\op} = \pnorm{ {\bm{S}}_Z^{1/2}(\widehat{\bm{S}}_Z - {\bm{S}}_Z   )\widehat{\bm{S}}_Z^{-1} +  ({\bm{S}}_Z^{3/2} -  \widehat{\bm{S}}_Z^{3/2})\widehat{\bm{S}}_Z^{-1}  }{\op} \\
       &\leq \pnorm{ {\bm{S}}_Z^{1/2}(\widehat{\bm{S}}_Z - {\bm{S}}_Z )\widehat{\bm{S}}_Z^{-1}  }{\op} + \pnorm{({\bm{S}}_Z^{3/2} -  \widehat{\bm{S}}_Z^{3/2})\widehat{\bm{S}}_Z^{-1} }{\op} \\
       &\leq (\mathsf{c}_1^{-1/2}\mathsf{c}_2^{-1} +3\mathsf{c}_2^{-1} (\mathsf{c}_1^{-1/2} \vee \mathsf{c}_2^{-1/2} ))\pnorm{ \widehat{\bm{S}}_Z - {\bm{S}}_Z }{\op}  \\
       &\leq 4( \mathsf{c}_1^{-3/2} \vee \mathsf{c}_2^{-3/2})\pnorm{\bm{r}}{\infty}^2 \cdot \pnorm{ \widehat{\bm{\Omega}} - \Sigma_X^{-1}}{\op}. 
    \end{align*}
    Therefore, combining the above estimates completes the proof of Proposition \ref{prop:couple}.

\subsection{Proof of the claim in Example \ref{exmp:discrete_covariate}}
    Note that for any $\gamma \in \Gamma( \widehat{\mu}, \tilde{\mu})$ and any $\mathsf{c}_1 > 0$, 
    \begin{align}\label{ineq:lowerbound_1}
        &\Prob_{ \{(\bm{\widehat{X}}_{i,\cdot},\bm{\tilde{X}}_{i,\cdot} )\}_{i \in [n]} \overset{\text{i.i.d.}}{\sim} \gamma } \Big( \max_{j \in [p]} n^{-1/2} \pnorm{ \bm{\widehat{X}}_{\cdot,j} - \bm{\tilde{X}}_{\cdot,j} }{2}  \geq  \mathsf{c}_1 \Big) \nonumber \\
        &\geq \Prob_{ \{(\bm{\widehat{X}}_{i,\cdot},\bm{\tilde{X}}_{i,\cdot} )\}_{i \in [n]} \overset{\text{i.i.d.}}{\sim} \gamma } \Big( n^{-1} \pnorm{ \bm{\widehat{X}}_{\cdot,1} - \bm{\tilde{X}}_{\cdot,1} }{2}^2    \geq  \mathsf{c}_1^2 \Big).
    \end{align}
    Using the Paley--Zygmund inequality, we can show that 
    \begin{align}\label{ineq:lowerbound_4}
        &\Prob_{ \{(\bm{\widehat{X}}_{i,\cdot},\bm{\tilde{X}}_{i,\cdot} )\}_{i \in [n]} \overset{\text{i.i.d.}}{\sim} \gamma }  \Big( n^{-1} \pnorm{ \bm{\widehat{X}}_{\cdot,1} - \bm{\tilde{X}}_{\cdot,1} }{2}^2    \geq  \frac{1}{2n}\E_{ \{(\bm{\widehat{X}}_{i,\cdot},\bm{\tilde{X}}_{i,\cdot} )\}_{i \in [n]} \overset{\text{i.i.d.}}{\sim} \gamma }   \pnorm{ \bm{\widehat{X}}_{\cdot,1} - \bm{\tilde{X}}_{\cdot,1} }{2}^2 \Big)\nonumber \\
        &\geq \frac{ n^{-2}\E_{ \{(\bm{\widehat{X}}_{i,\cdot},\bm{\tilde{X}}_{i,\cdot} )\}_{i \in [n]} \overset{\text{i.i.d.}}{\sim} \gamma }^2  \pnorm{ \bm{\widehat{X}}_{\cdot,1} - \bm{\tilde{X}}_{\cdot,1} }{2}^2 }{4 n^{-2} \E_{ \{(\bm{\widehat{X}}_{i,\cdot},\bm{\tilde{X}}_{i,\cdot} )\}_{i \in [n]} \overset{\text{i.i.d.}}{\sim} \gamma }  \pnorm{ \bm{\widehat{X}}_{\cdot,1} - \bm{\tilde{X}}_{\cdot,1} }{2}^4 }.
    \end{align}
    We will provide a lower bound for $n^{-1} \E_{ \{(\bm{\widehat{X}}_{i,\cdot},\bm{\tilde{X}}_{i,\cdot} )\}_{i \in [n]} \overset{\text{i.i.d.}}{\sim} \gamma } \pnorm{ \bm{\widehat{X}}_{\cdot,1} - \bm{\tilde{X}}_{\cdot,1} }{2}^2$. Let $\mu_{\mathsf{g}}$ and $\mu_{\mathsf{r}}$ be the standard Gaussian distribution and Rademacher distribution, respectively. It holds that 
    \begin{align}\label{ineq:lowerbound_2}
      &n^{-1} \E_{ \{(\bm{\widehat{X}}_{i,\cdot},\bm{\tilde{X}}_{i,\cdot} )\}_{i \in [n]} \overset{\text{i.i.d.}}{\sim} \gamma }  \pnorm{ \bm{\widehat{X}}_{\cdot,1} - \bm{\tilde{X}}_{\cdot,1} }{2}^2 \nonumber\\
      &= \frac{1}{n} \sum_{i \in [n]} \E_{ (\bm{\widehat{X}}_{i,\cdot},\bm{\tilde{X}}_{i,\cdot} ){\sim} \gamma } (\bm{\widehat{X}}_{i,1} - \bm{\tilde{X}}_{i,1} )^2 \geq  \mathcal{W}_2^2( \mu_{\mathsf{g}}, \mu_{\mathsf{r}} ).
    \end{align}
    
    The Wasserstein distance between $\mu_{\mathsf{g}}$ and $\mu_{\mathsf{r}}$ can be computed explicitly. Let $F^{-1}_\mathsf{g}(\cdot)$ and $F^{-1}_\mathsf{r}(\cdot)$ be the  inverse cumulative distribution functions of $\mu_{\mathsf{g}}$ and $\mu_{\mathsf{r}}$, respectively. It holds that 
    \begin{align}\label{ineq:lowerbound_3}
        \mathcal{W}_2^2( \mu_{\mathsf{g}}, \mu_{\mathsf{r}} ) &= \int_0^1 \abs{F^{-1}_\mathsf{g}(\mathsf{q} ) - F^{-1}_\mathsf{r}(\mathsf{q} )  }  \mathrm{d} \mathsf{q}\nonumber \\
        &= \int_0^{1/2} \abs{F^{-1}_\mathsf{g}(\mathsf{q} ) + 1  }  \mathrm{d} \mathsf{q}+\int_{1/2}^{1} \abs{F^{-1}_\mathsf{g}(\mathsf{q} ) - 1  }  \mathrm{d} \mathsf{q} \nonumber\\
        &= \int_{-\infty}^0 \abs{\mathsf{x} + 1} \mu_{\mathsf{g}}( \mathrm{d}\mathsf{x} ) + \int_{0}^\infty \abs{\mathsf{x} - 1} \mu_{\mathsf{g}}( \mathrm{d}\mathsf{x} )  \geq \mathsf{c}_2
    \end{align}
    for some universal $\mathsf{c}_2 > 0$. Moreover, it is easy to show that for some universal $\mathsf{c}_3 > 0$, 
    \begin{align}\label{ineq:lowerbound_5}
         n^{-2} \E_{ \{(\bm{\widehat{X}}_{i,\cdot},\bm{\tilde{X}}_{i,\cdot} )\}_{i \in [n]} \overset{\text{i.i.d.}}{\sim} \gamma } \pnorm{ \bm{\widehat{X}}_{\cdot,1} - \bm{\tilde{X}}_{\cdot,1} }{2}^4 \leq \mathsf{c}_3.
    \end{align}
    
    Therefore, by setting $\mathsf{c}_1$ in \eqref{ineq:lowerbound_1} as  $\sqrt{\mathsf{c}_2/2}$, we can obtain that 
    \begin{align*}
         &\Prob_{ \{(\bm{\widehat{X}}_{i,\cdot},\bm{\tilde{X}}_{i,\cdot} )\}_{i \in [n]} \overset{\text{i.i.d.}}{\sim} \gamma } \Big( \max_{j \in [p]} n^{-1/2} \pnorm{ \bm{\widehat{X}}_{\cdot,j} - \bm{\tilde{X}}_{\cdot,j} }{2}  \geq  \sqrt{\mathsf{c}_2/2} \Big)\\
         &\overset{\eqref{ineq:lowerbound_1}}{\geq} \Prob_{ \{(\bm{\widehat{X}}_{i,\cdot},\bm{\tilde{X}}_{i,\cdot} )\}_{i \in [n]} \overset{\text{i.i.d.}}{\sim} \gamma }\Big( n^{-1} \pnorm{ \bm{\widehat{X}}_{\cdot,1} - \bm{\tilde{X}}_{\cdot,1} }{2}^2    \geq  \mathsf{c}_2/2 \Big) \\
        &\overset{\eqref{ineq:lowerbound_2} \& \eqref{ineq:lowerbound_3}} {\geq}  \Prob_{ \{(\bm{\widehat{X}}_{i,\cdot},\bm{\tilde{X}}_{i,\cdot} )\}_{i \in [n]} \overset{\text{i.i.d.}}{\sim} \gamma } \Big( n^{-1} \pnorm{ \bm{\widehat{X}}_{\cdot,1} - \bm{\tilde{X}}_{\cdot,1} }{2}^2    \geq  \frac{1}{2n}\E_{ \{(\bm{\widehat{X}}_{i,\cdot},\bm{\tilde{X}}_{i,\cdot} )\}_{i \in [n]} \overset{\text{i.i.d.}}{\sim} \gamma }  \pnorm{ \bm{\widehat{X}}_{\cdot,1} - \bm{\tilde{X}}_{\cdot,1} }{2}^2 \Big) \\
        &\overset{\eqref{ineq:lowerbound_4}-\eqref{ineq:lowerbound_5} }{\geq} \frac{\mathsf{c}_2^2}{4\mathsf{c}_3}.
    \end{align*}
    The claim follows by adjusting the constants, which completes the proof of the claim in Example \ref{exmp:discrete_covariate}. 

\section{Marginal correlation knockoff statistics with correlated features}\label{sec:mcks_correlated}

In this section, we present our results on the asymptotic FDR control for marginal correlation knockoff statistics with correlated features. The proofs for this section closely align with those for other knockoff statistics discussed in this paper; additional details are provided in Section \ref{sec:proof_sec_mcks_correlated}.

For correlated features, recall that the knockoff statistic based on the marginal correlation difference is defined as 
\begin{align*}
        \widehat{\bm{W}}_j \equiv  \frac{\abs{ \bm{X}_{\cdot,j}^\top  \bm{Y} } }{\pnorm{\bm{Y}}{}} - \frac{\abs{ \widehat{\bm{X}}_{\cdot,j}^\top  \bm{Y} } }{\pnorm{\bm{Y}}{}}  , \quad j \in [p],
\end{align*}
where the approximate knockoffs data matrix is generated via
\begin{align*}
	\widehat{\bm{X}} = \bm{X}(I_p - \Sigma_X^{-1}\mathrm{diag}(\bm{r})) + \bm{\mathsf{Z}}(2\mathrm{diag}(\bm{r}) - \mathrm{diag}(\bm{r})\Sigma_X^{-1}\mathrm{diag}(\bm{r}))^{1/2}
\end{align*}
with $\bm{\mathsf{Z}} \in \R^{n \times p}$ consisting of i.i.d. $\mathcal{N}(0,1)$ entries that are independent of $\bm{X}$ and $\bm{\xi}$. 
Recall also that 
\begin{align*}
	\Sigma = \cov \big[ (X^\top, \widehat{X}^\top )^\top \big] = \begin{pmatrix}
		\Sigma_X & \Sigma_X - \mathrm{diag}(\bm{r}) \\
		\Sigma_X - \mathrm{diag}(\bm{r}) & \Sigma_X
	\end{pmatrix}.
\end{align*}

We impose the following moment conditions:
\begin{enumerate}[label=(D\arabic*), ref=D\arabic*]
    \item \label{cond:D1} The components of $\bm{Y}$ are i.i.d. and uniformly sub-exponential, and matrix $\widehat{\bm{Z}} = [\bm{X}, \widehat{\bm{X}} ] = \widehat{\bm{Q}} \Sigma^{1/2}$, where the rows of $\widehat{\bm{Q}}$ are i.i.d. with mean-zero, unit-variance, and uniformly sub-exponential entries. 
    \item \label{cond:D2}  $\pnorm{\Sigma}{\op} \vee \pnorm{\Sigma^{-1}}{\op} \leq \mathsf{M}$ for some constant $\mathsf{M} > 0$.
\end{enumerate}


The proposition below, analogous to Proposition \ref{prop:conc_knockoff}, establishes a concentration result for the approximate knockoff statistics and quantifies the signal strength.

\begin{proposition}\label{prop:conc_knockoff_mc_c}
Assume that conditions \eqref{cond:D1} and \eqref{cond:D2} hold. Then if $\log p = \mathfrak{o}(\sqrt{n})$, there exists some constant $C = C(\mathsf{M}) > 0$ such that for any $j \in [p]$,
	    \begin{align*}
	    	\Prob \Big(  \bigabs{ \widehat{\bm{W}}_j - \widehat{\bm{w}}_j } \geq C\delta_n \Big) \leq   p^{-10} + \exp ( -n  \E^2 Y^2/C ),
	    \end{align*}
        	    where $\widehat{\bm{w}}_j \equiv  \sqrt{n}(\abs{\E X_j Y} - \abs{\E \widehat{X}_j Y} )/\E^{1/2}Y^2$ and $\delta_n \equiv  \frac{\log p}{\E^{1/2}Y^2} + \frac{\log^2 p}{\sqrt{ n\E Y^2}} $.
\end{proposition}

With slight abuse of notation, we define for any $t \geq 0$ and $j \in [p]$, 
\begin{align*}
	\mathcal{P}_j (t ) &\equiv \Prob_{(\mathsf{G}_1, \mathsf{G}_2)^\top \sim \mathcal{N}(0,  \Sigma_{\{j,j+p \}, \{j,j+p \} }  ) }(\abs{\mathsf{G}_1} - \abs{\mathsf{G}_2}  \geq t) \\
    &=\Prob_{(\mathsf{G}_1, \mathsf{G}_2)^\top \sim \mathcal{N}(0,  \Sigma_{\{j,j+p \}, \{j,j+p \} }  ) }(\abs{\mathsf{G}_1} - \abs{\mathsf{G}_2}  \leq -t).
\end{align*}
Recall that for any $t >0$, $\mathcal{P}(t) = \abs{\mathcal{H}_0}^{-1}\sum_{j \in \mathcal{H}_0} \mathcal{P}_j(t)$,  and for any $x \in [0,1/2]$, its inverse function  $\mathcal{P}^{-1}(x) =\sup \{t \geq 0 : \mathcal{P}(t) \geq x \}$.  To ensure robust performance, we impose the following conditions on the class of approximate knockoff statistics. 

\begin{enumerate}[label=(D\arabic*), ref=D\arabic*, start = 3]
	\item \label{cond:D3} For any constant $C >0$ independent of $n$, $a_n \equiv \#\{j \in \mathcal{H}_1: \widehat{\bm{w}}_j \geq C{\delta}_n  \} \to \infty $. Additionally,  $\mathfrak{m}_n/a_n \to 0$. Here, $\mathfrak{m}_n \equiv  \max_{j \in \mathcal{H}_0} \abs{M(j)}$ with 
$
    M(j) \equiv \big\{\ell \in \mathcal{H}_0 \setminus\{j\}: \sum_{\mathfrak{j} \in \{j, j + p\} } \big(\abs{ \Sigma_{\mathfrak{j}, \ell }} + \abs{ \Sigma_{\mathfrak{j}, \ell+p }}\big) \neq 0  \big\}.
$
	\item \label{cond:D4} For any $q \in (0,1)$, $(qa_n)^{-1} \sum_{j \in \mathcal{H}_1}\Prob ( \widehat{\bm{W}}_j <- \mathcal{P}^{-1} ( \frac{q  a_n }{2p}) ) \to 0$. 
\end{enumerate}

We are now ready to state our result on the asymptotic FDR control for marginal correlation knockoff statistics with correlated features.

\begin{theorem}\label{thm:FDR_mc_c}
Assume that conditions \eqref{cond:D1}--\eqref{cond:D4} hold. Then for any $q \in (0,1)$, if  $\log p = \mathfrak{o}(n^c)$ for some small enough constant $c > 0 $, conditions (1)--(3) in Theorem \ref{thm:general_framwork} are satisfied for $\alpha_n = \mathcal{P}^{-1} \big( \frac{q a_n}{2p} \big)$. Consequently, we have 
\begin{align*} 
\limsup_{n,p \to \infty} \mathrm{FDR} \leq q. 
\end{align*}
\end{theorem}

\section{Additional proofs and technical details} \label{new.sec.addproandtechdet}

\subsection{Proofs for Section \ref{sec:proof_FDR_mcks_light}}\label{sec:proof_sec_FDR_mcks_light}

\subsubsection{Proof of Theorem \ref{thm:moder_dis_mcks}}
We first adapt the proof of the moderate deviation result from \cite[Theorem 4.2]{fang2023p} to address a special class of sets in two dimensions (i.e., $d = 2$ therein). Specifically, let us consider sets 
\begin{align}\label{def:knock_sets}
    \mathcal{S}_{t,\pm} \equiv \{(x,y): \pm(\abs{x} - \abs{y}) \geq t \}
\end{align}
for any $t > 0$.

\begin{theorem}\label{thm:moderate_deviation}
    Let $W = n^{-1/2}\sum_{i \in [n]} \mathsf{X}_i \in \R^2$, where $\{\mathsf{X}_1,\ldots,\mathsf{X}_n\}$ are independent, $\E \mathsf{X}_i = 0$ for all $i \in [n]$, and $\var W = \Sigma_W$. Assume that  $\pnorm{\mathsf{X}_i}{\psi_1} \leq b$ for all $i \in [n]$ and $\pnorm{\Sigma_W^{-1}}{\op} \vee \pnorm{\Sigma_W}{\op} \leq K$ for some constant $K$. Then there exist some constants $C_1,C_2,C_3 > 0$ which only depend on $K$ such that if
    $
        C_1 b^2 (1 + t^3)\log n \leq \sqrt{n},
   $
    we have that for $\# \in \{+,- \}$ and $t \in (0, n^{1/7} /C_2)$,
    \begin{align*}
       \biggabs{ \frac{\Prob \big( W \in\mathcal{S}_{t,\#}  \big)  }{\Prob \big( \mathsf{G} \in  \mathcal{S}_{t,\#} \big)} - 1} \leq  C_3b^2 \cdot \frac{ (1 + t^3)\log n}{\sqrt{n}},
    \end{align*}
    where $\mathsf{G} \sim \mathcal{N}(0,\Sigma_W)$.
\end{theorem}

\noindent\textit{Proof of Theorem \ref{thm:moderate_deviation}}. 
We only show the proof for the case when $\# = +$, and the proof for $\# = -$ is similar so we omit repetitive details. The constant $C$ below may depend on $K$ which may change from line to line hereafter. For $t \leq 4$, the conclusion simply follows from the multi-dimensional Berry--Esseen Theorem with the fact that $\Prob \big( \mathsf{G} \in  \mathcal{S}_{t,\#} \big) > C$. As such, we will focus on the case when $t > 4$.

First note by \cite[Theorem 4.1]{fang2023p} that for any $\mathsf{q} \geq 2$,
\begin{align*}
    \mathcal{W}_\mathsf{q}(\Sigma_W^{-1/2}W, \Sigma_W^{-1/2}\mathsf{G}) \leq C\bigg( \frac{\mathsf{q}}{\sqrt{n}} + \frac{\mathsf{q}^{5/2}}{n} \bigg) b^2 \equiv \mathsf{B}_{n,\mathsf{q}},
\end{align*}
where $\mathcal{W}_\mathsf{q}(\cdot,\cdot)$ denotes the $\mathsf{q}$-Wasserstein distance. Hence, we can couple $W$ and $\mathsf{G}$ such that $\E\pnorm{ \Sigma_W^{-1/2} W- \Sigma_W^{-1/2} \mathsf{G}}{}^\mathsf{q} \leq (\mathsf{B}_{n,q})^\mathsf{q}$. Then for any $t > 0$ and $\epsilon \in (0,t/2)$, it holds that 
\begin{align*}
    &\Prob (W \in \mathcal{S}_{t} ) =  \Prob (\abs{W_1} - \abs{W_2} \geq t) \\
    &\leq \Prob (\abs{\mathsf{G}_1} - \abs{\mathsf{G}_2} \geq t - \epsilon) + \Prob \big(\abs{\abs{W_1} - \abs{W_2} - \abs{\mathsf{G}_1} + \abs{\mathsf{G}_2}   } > \epsilon  \big) \\
    &= \Prob (\abs{\mathsf{G}_1} - \abs{\mathsf{G}_2} \geq t  ) + \Prob (t-\epsilon\leq \abs{\mathsf{G}_1} - \abs{\mathsf{G}_2} \leq t ) \\
    &\quad + \Prob \big(\abs{\abs{W_1} - \abs{W_2} - \abs{\mathsf{G}_1} + \abs{\mathsf{G}_2}   } > \epsilon  \big). 
\end{align*}
We will bound the last two terms above separately. 

\medskip

\emph{Term $ \Prob \big(\abs{\abs{W_1} - \abs{W_2} - \abs{\mathsf{G}_1} + \abs{\mathsf{G}_2}   } > \epsilon  \big) $}: Since
\begin{align*}
    \abs{\abs{W_1} - \abs{W_2} - \abs{\mathsf{G}_1} + \abs{\mathsf{G}_2}   } \leq \abs{\abs{W_1}-  \abs{\mathsf{G}_1} } + \abs{ \abs{W_2} -\abs{\mathsf{G}_2} } \leq K^{1/2}\cdot  \pnorm{ \Sigma_W^{-1/2} W-  \Sigma_W^{-1/2} \mathsf{G}}{},
\end{align*}
it holds that for any $\mathsf{q} \geq 2$,
\begin{align*}
    \Prob \big(\abs{\abs{W_1} - \abs{W_2} - \abs{\mathsf{G}_1} + \abs{\mathsf{G}_2}   } > \epsilon  \big) \leq \Prob \big( \pnorm{ \Sigma_W^{-1/2} W-  \Sigma_W^{-1/2} \mathsf{G}}{} > \epsilon/K^{1/2}  \big) \\
    \leq (\epsilon/K^{1/2})^{-\mathsf{q}} \E\pnorm{ \Sigma_W^{-1/2}W- \Sigma_W^{-1/2}\mathsf{G}}{}^\mathsf{q}.
\end{align*}
Hence, we have 
\begin{align*}
    \frac{\Prob \big(\abs{\abs{W_1} - \abs{W_2} - \abs{\mathsf{G}_1} + \abs{\mathsf{G}_2}   } > \epsilon  \big) }{ \Prob (\abs{\mathsf{G}_1} - \abs{\mathsf{G}_2} \geq t  )} \leq \frac{(\epsilon/K^{1/2})^{-\mathsf{q}}(\mathsf{B}_{n,\mathsf{q}})^\mathsf{q}}{\Prob (\abs{\mathsf{G}_1} - \abs{\mathsf{G}_2} \geq t  ) }.
\end{align*}
Setting $ \epsilon = K^{1/2}e \mathsf{B}_{n,\mathsf{q}} $ and $\mathsf{q} = \abs{\log \Prob (\abs{\mathsf{G}_1} - \abs{\mathsf{G}_2} \geq t  ) } + \abs{\log \Delta}$ for some $\Delta > 0$, the bound above becomes
\begin{align}\label{ineq:moder_est_2}
     \frac{\Prob \big(\abs{\abs{W_1} - \abs{W_2} - \abs{\mathsf{G}_1} + \abs{\mathsf{G}_2}   } > \epsilon  \big) }{ \Prob (\abs{\mathsf{G}_1} - \abs{\mathsf{G}_2} \geq t  )} \leq \Delta.
\end{align}
This completes the bound for $ \Prob \big(\abs{\abs{W_1} - \abs{W_2} - \abs{\mathsf{G}_1} + \abs{\mathsf{G}_2}   } > \epsilon  \big) $.

\medskip

\emph{Term $\Prob \big( t - \epsilon <  \abs{\mathsf{G}_1} - \abs{\mathsf{G}_2} \leq t \big) $}: We first decompose $\Prob \big( t - \epsilon <  \abs{\mathsf{G}_1} - \abs{\mathsf{G}_2} \leq t \big) $ as 
\begin{align*}
    \Prob \big( t - \epsilon <  \abs{\mathsf{G}_1} - \abs{\mathsf{G}_2} \leq t \big)= \sum_{\star, \ast \in \{+,- \} }\Prob \big( t - \epsilon <  \abs{\mathsf{G}_1} - \abs{\mathsf{G}_2} \leq t,  \star\mathsf{G}_1 \geq 0, \ast \mathsf{G}_2 \geq 0 \big).
\end{align*}
We will only show how to bound the case when $\star = +, \ast = +$, the other cases can be handled similarly. Let us define $\Lambda \equiv \Sigma_W^{-1}$ and 
\begin{align*}
    f(\epsilon) \equiv \frac{\Prob \big( t - \epsilon <  \abs{\mathsf{G}_1} - \abs{\mathsf{G}_2} \leq t, \mathsf{G}_1 \geq 0, \mathsf{G}_2 \geq 0  \big)}{\Prob (\abs{\mathsf{G}_1} - \abs{\mathsf{G}_2} \geq t , \mathsf{G}_1 \geq 0, \mathsf{G}_2 \geq 0  )  } = \frac{\int_{0}^\infty \int_{\bm{x}_1+t-\epsilon}^{\bm{x}_1+t} \exp \bigg(-\frac{\bm{x}^\top \Lambda\bm{x}  }{2} \bigg) \mathrm{d} \bm{x}_2\mathrm{d} \bm{x}_1}{\int_{0}^\infty \int_{\bm{x}_1+t}^{\infty} \exp \bigg(-\frac{\bm{x}^\top \Lambda\bm{x}  }{2} \bigg) \mathrm{d} \bm{x}_2\mathrm{d} \bm{x}_1}.
\end{align*}
    Obviously, $f(0) = 0$. Taking the derivative yields that 
    \begin{align*}
        f'(\epsilon) &= \frac{\int_{0}^\infty  \exp \Big( - \big( (\bm{x}_1 + t -\epsilon)^2 \Lambda_{11} + 2(\bm{x}_1 + t -\epsilon)\bm{x}_1 \Lambda_{12} + \bm{x}_1^2\Lambda_{22} \big)/2 \Big) \mathrm{d} \bm{x}_1}{\int_{0}^\infty \int_{\bm{x}_1+t}^{\infty} \exp \bigg(-\frac{\bm{x}^\top \Lambda\bm{x}  }{2} \bigg) \mathrm{d} \bm{x}_2\mathrm{d} \bm{x}_1}\\
        &\leq \frac{t\int_{0}^\infty  \exp \Big( - \big( (\bm{x}_1 + t -\epsilon)^2 \Lambda_{11} + 2(\bm{x}_1 + t -\epsilon)\bm{x}_1 \Lambda_{12} + \bm{x}_1^2\Lambda_{22} \big) /2\Big) \mathrm{d} \bm{x}_1}{\int_{0}^\infty  \exp \Big( - \big( (\bm{x}_1 + t + t^{-1})^2 \Lambda_{11} + 2(\bm{x}_1 + t +t^{-1})\bm{x}_1 \Lambda_{12} + \bm{x}_1^2\Lambda_{22} \big) /2\Big)\mathrm{d} \bm{x}_1}.
    \end{align*}
    
Using the facts that 
    \begin{enumerate}
        \item[ $\bullet$ ] $\exp \big( - \big( (\bm{x}_1 + t -\epsilon)^2 \Lambda_{11} + 2(\bm{x}_1 + t -\epsilon)\bm{x}_1 \Lambda_{12} + \bm{x}_1^2\Lambda_{22} \big) /2\big)\leq  \exp \big( -\lambda_{\min} (\Lambda ) ((\bm{x}_1 + t -\epsilon)^2 +\bm{x}_1^2) /2  \big)$ and 
        \item[ $\bullet$ ] $\exp \big( - \big( (\bm{x}_1 + t +t^{-1})^2 \Lambda_{11} + 2(\bm{x}_1 + t +t^{-1})\bm{x}_1 \Lambda_{12} + \bm{x}_1^2\Lambda_{22} \big)/2 \big) \geq \exp \big( -\lambda_{\max} (\Lambda ) ((\bm{x}_1 + t +t^{-1})^2 +\bm{x}_1^2)  /2 \big)$,
    \end{enumerate}
    we can show that for $\mathsf{z}_1$ to be chosen later,
    \begin{align*}
        f'(\epsilon) &\leq \frac{t\int_{0}^{\mathsf{z}_1t}  \exp \Big( - \big( (\bm{x}_1 + t -\epsilon)^2 \Lambda_{11} + 2(\bm{x}_1 + t -\epsilon)\bm{x}_1 \Lambda_{12} + \bm{x}_1^2\Lambda_{22} \big)/2 \Big) \mathrm{d} \bm{x}_1}{\int_{0}^{\mathsf{z}_1t} \exp \Big( - \big( (\bm{x}_1 + t +t^{-1})^2 \Lambda_{11} + 2(\bm{x}_1 + t +t^{-1})\bm{x}_1 \Lambda_{12} + \bm{x}_1^2\Lambda_{22} \big) /2\Big)\mathrm{d} \bm{x}_1} \\
        &+\frac{t\int_{\mathsf{z}_1t}^{\infty}   \exp \big( -\lambda_{\min} (\Lambda ) ((\bm{x}_1 + t -\epsilon)^2 +\bm{x}_1^2) /2  \big) \mathrm{d} \bm{x}_1}{ \int_{0}^{t} \exp \big( -\lambda_{\max} (\Lambda ) ((\bm{x}_1 + t +t^{-1})^2 +\bm{x}_1^2)  /2 \big)\mathrm{d} \bm{x}_1} \equiv f_1 + f_2.
    \end{align*}
    For $f_2$, choosing  $\mathsf{z}_1$ such that $\mathsf{z}_1 \geq \sqrt{10 \lambda_{\max} (\Lambda ) / \lambda_{\min} (\Lambda )}$, with some calculations we can obtain that 
    \begin{align*}
        f_2 &\leq \frac{ t\exp \big( -\lambda_{\min} (\Lambda ) (\mathsf{z}_1 t)^2/2\big) \int_{\mathsf{z}_1t}^{\infty}   \exp \big( -\lambda_{\min} (\Lambda )(\bm{x}_1 + t -\epsilon)^2 /2  \big) \mathrm{d} \bm{x}_1}{ \exp \big( -\lambda_{\max} (\Lambda ) ((2 t +t^{-1})^2 +t^2)  /2 \big)}\\
        &\leq C\frac{ t\exp \big( -\lambda_{\min} (\Lambda ) (\mathsf{z}_1 t)^2/2\big) }{ \exp \big( -\lambda_{\max} (\Lambda ) ((2 t +t^{-1})^2 +t^2)  /2 \big)} \leq Ct.
    \end{align*}
    
     For $f_1$, let us consider two cases: $\Lambda_{12} > 0$ and $\Lambda_{12} \leq 0$. If $\Lambda_{12} > 0$, it holds that 
     \begin{align*}
        f_1 \leq  \frac{t\exp \big(\epsilon (1+\mathsf{z}_1)t\Lambda_{11} + \epsilon\mathsf{z}_1 t\Lambda_{12} \big)}{ \exp \big(-(2+\mathsf{z}_1)\Lambda_{11} - \mathsf{z}_1 \Lambda_{12}\big)} \leq t \exp \big((\epsilon t + 1)[(2+\mathsf{z}_1)\Lambda_{11}+\mathsf{z}_1\Lambda_{12}] \big).
    \end{align*}
    If $\Lambda_{12} \leq 0$, we have that 
    \begin{align*}
        f_1 \leq  \frac{t\exp \big(\epsilon (1+\mathsf{z}_1)t\Lambda_{11} \big) }{ \exp \big(-(2+\mathsf{z}_1)\Lambda_{11} \big)} \leq t \exp \big((\epsilon t+1)(2+\mathsf{z}_1)\Lambda_{11} \big).
    \end{align*}
        Hence, an application of the fact that $\pnorm{\Lambda}{\max} \leq \lambda_{\max}(\Lambda)$ leads to 
    \begin{align*}
        f_1 \leq t \exp \big(4\lambda_{\max}(\Lambda)(\epsilon t + 1) (1 + \mathsf{z}_1)  \big).
    \end{align*}
    Combining the estimates for $f_1$ and $f_2$ above, we can obtain by the first-order Taylor expansion that 
    \begin{align*}
     &\Prob \big( t - \epsilon <  \abs{\mathsf{G}_1} - \abs{\mathsf{G}_2} \leq t, \mathsf{G}_1 \geq 0, \mathsf{G}_2 \geq 0  \big) = \Prob (\abs{\mathsf{G}_1} - \abs{\mathsf{G}_2} \geq t,\mathsf{G}_1 \geq 0, \mathsf{G}_2 \geq 0  ) \cdot f(\epsilon) \nonumber\\
     &\leq C\epsilon t  \big(\exp \big(C(\epsilon t + 1) \big) \cdot \Prob (\abs{\mathsf{G}_1} - \abs{\mathsf{G}_2} \geq t, \mathsf{G}_1 \geq 0, \mathsf{G}_2 \geq 0  ).
\end{align*}
Similar bounds hold for other cases when $\star, \ast \in \{+,-\}$. Combining, we have that
\begin{align}\label{ineq:moder_est_1}
     &\Prob \big( t - \epsilon <  \abs{\mathsf{G}_1} - \abs{\mathsf{G}_2} \leq t  \big) = \Prob (\abs{\mathsf{G}_1} - \abs{\mathsf{G}_2} \geq t  ) \cdot f(\epsilon) \nonumber\\
     &\leq C\epsilon t  \big(\exp \big(C(\epsilon t + 1) \big) \cdot \Prob (\abs{\mathsf{G}_1} - \abs{\mathsf{G}_2} \geq t ).
\end{align}
This completes the bound for $\Prob \big( t - \epsilon <  \abs{\mathsf{G}_1} - \abs{\mathsf{G}_2} \leq t \big)$.

Moreover, we have
\begin{align*}
    &\Prob (\abs{\mathsf{G}_1} - \abs{\mathsf{G}_2} \geq t  ) = \frac{1}{2\pi \sqrt{\det (\Sigma_W)}} \int_{0}^\infty \int_{\bm{x}_1+t}^{\infty} \exp \bigg(-\frac{\bm{x}^\top \Lambda\bm{x}  }{2} \bigg) \mathrm{d} \bm{x}_2\mathrm{d} \bm{x}_1 \\
    &\geq \frac{1}{2\pi \sqrt{\det (\Sigma_W)}} \int_{0}^\infty  \exp \bigg(-  \frac{\lambda_{\max}(\Lambda)\bm{x}_1^2  }{2}  \bigg) \int_{\bm{x}_1+t}^{\infty} \exp \bigg(-  \frac{\lambda_{\max}(\Lambda)\bm{x}_2^2  }{2} \bigg) \mathrm{d} \bm{x}_2\mathrm{d} \bm{x}_1 \\
    &\geq \frac{1}{2\pi \lambda_{\max}(\Lambda) \sqrt{\det (\Sigma_W)}} \int_{0}^\infty (\bm{x}_1 + t)^{-1} \exp \bigg(-  \frac{\lambda_{\max}(\Lambda)(\bm{x}_1^2 +(\bm{x}_1+t)^2) }{2}  \bigg)   \mathrm{d} \bm{x}_1 \\
    & \geq \frac{1}{2\pi \lambda_{\max}(\Lambda) \sqrt{\det (\Sigma_W)}} \int_{0}^t (\bm{x}_1 + t)^{-1} \exp \bigg(-  \frac{\lambda_{\max}(\Lambda)(\bm{x}_1^2 +(\bm{x}_1+t)^2) }{2}  \bigg)   \mathrm{d} \bm{x}_1\\
    &\geq \frac{1}{4\pi \lambda_{\max}(\Lambda) \sqrt{\det (\Sigma_W)}}\exp \bigg(-  \frac{5\lambda_{\max}(\Lambda)t^2 }{2}  \bigg).
\end{align*}

Now with $\Delta = n^{-1}$, we have 
\begin{align*}
  &\mathsf{q} =   \abs{\log \Prob (\abs{\mathsf{G}_1} - \abs{\mathsf{G}_2} \geq t  ) }+ \log n \\
  &\leq \biggabs{\log \frac{1}{4\pi \lambda_{\max}(\Lambda) \sqrt{\det (\Sigma_W)}}\exp \bigg(-  \frac{5\lambda_{\max}(\Lambda)t^2 }{2}  \bigg)} + \log n\leq C(\log n +  t^2).
\end{align*}
Then as long as 
\begin{align*}
    \epsilon t &= C e\bigg( \frac{\mathsf{q}}{\sqrt{n}} + \frac{\mathsf{q}^{5/2}}{n} \bigg) b^2 t \leq C b^2 t \bigg( \frac{t^2 +\log n }{\sqrt{n}} + \frac{t^5 +\log^{5/2} n}{n} \bigg) \\
    &\leq Cb^2 \cdot \frac{ (1 + t^3)\log n}{\sqrt{n}} \leq 1,
\end{align*}
we can combine (\ref{ineq:moder_est_2}) and (\ref{ineq:moder_est_1}) above to derive that
\begin{align*}
    &\frac{ \Prob (\abs{W_1} - \abs{W_2} \geq t )}{ \Prob (\abs{\mathsf{G}_1} - \abs{\mathsf{G}_2} \geq t  )} \\
    & \leq 1 + \frac{\Prob (t-\epsilon\leq \abs{\mathsf{G}_1} - \abs{\mathsf{G}_2} \leq t ) + \Prob \big(\abs{\abs{W_1} - \abs{W_2} - \abs{\mathsf{G}_1} + \abs{\mathsf{G}_2}   } > \epsilon  \big)}{ \Prob (\abs{\mathsf{G}_1} - \abs{\mathsf{G}_2} \geq t  )} \\
        &\leq 1 +  C \epsilon t e^{C\epsilon t} + \Delta \leq 1 + Cb^2 \cdot \frac{ (1 + t^3)\log n}{\sqrt{n}}.
\end{align*}
A similar argument establishes the converse, which completes the proof of Theorem \ref{thm:moderate_deviation}. 

	     We focus on the case when $\mathsf{W} = \widehat{\bm{W}}$, since the other cases can be treated similarly. 
     The constant $C$ below can depend on $\mathsf{M}$, whose value may change from line to line. First note that by the standard sub-exponential concentration, we have 
     \begin{align}\label{ineq:def_of_E_Y}
         \Prob \bigg(\mathcal{E}_Y \equiv \bigg\{ \frac{\max_{i \in [n]} \abs{\bm{Y}_i}}{\pnorm{\bm{Y}}{} } \leq C \frac{\log n p } {\sqrt{n}}  \bigg\} \bigg) \geq 1 - p^{-10}.
     \end{align}   
     Since $\bm{X}_{\cdot,j}$, $\widehat{\bm{X}}_{\cdot,j}$, and $\bm{Y}$ are independent for any $j \in \mathcal{H}_0$, we can apply Theorem \ref{thm:moderate_deviation} by letting $\mathsf{X}_i$ therein be 
 $( \sqrt{n}{\bm{X}}_{i,j}\bm{Y}_i/(\sigma_j \pnorm{\bm{Y}}{}),  \sqrt{n}\widehat{\bm{X}}_{i,j}\bm{Y}_i/(\sigma_j\pnorm{\bm{Y}}{}))^\top$ for any $i \in [n]$. This yields that on event $\mathcal{E}_Y$, for any $j \in \mathcal{H}_0$ and $t \in (0, n^{1/7}/C)$,
      \begin{align}\label{ineq:est_conditionY}
         \biggabs{ \frac{\Prob^{\bm{Y}} \big( \widehat{\bm{W}}_j \geq t \big)}{\mathcal{P}_j(t) } - 1} \leq C \frac{(1+ t^3)\log^3 (np) }{\sqrt{n}}.
     \end{align}
     Using the law of total probability, we can deduce that 
    \begin{align*}
       &\biggabs{ \frac{\Prob \big( \widehat{\bm{W}}_j \geq t \big)}{\mathcal{P}_j(t) } - 1} = \biggabs{ \E_{\bm{Y}}\bigg[ \frac{\Prob^{\bm{Y}} \big( \widehat{\bm{W}}_j \geq t \big)}{\mathcal{P}_j(t) } \bigg]- 1} \\
        &\leq   \biggabs{ \E_{\bm{Y}}\bigg[ \frac{\Prob^{\bm{Y}} \big( \widehat{\bm{W}}_j \geq t \big) \bm{1}\{ \mathcal{E}_Y \} }{\mathcal{P}_j(t) } \bigg]- 1}   + \biggabs{ \E_{\bm{Y}}\bigg[ \frac{\Prob^{\bm{Y}} \big( \widehat{\bm{W}}_j \geq t \big)\bm{1}\{ \mathcal{E}_Y^c \}}{\mathcal{P}_j(t) } \bigg]} \\
        &\leq  \E_{\bm{Y}}\bigg[\biggabs{ \frac{\Prob^{\bm{Y}} \big( \widehat{\bm{W}}_j \geq t \big)  }{\mathcal{P}_j(t) } - 1} \cdot \bm{1}\{ \mathcal{E}_Y \}  \bigg]  + \biggabs{ \E_{\bm{Y}}\bigg[ \frac{\Prob^{\bm{Y}} \big( \widehat{\bm{W}}_j \geq t \big)\bm{1}\{ \mathcal{E}_Y^c \}}{\mathcal{P}_j(t) } \bigg]}+ \E_{\bm{Y}}\bm{1}\{ \mathcal{E}_Y^c \}  \\
        &\leq C \frac{(1+ t^3)\log^3 (np) }{\sqrt{n}} + \frac{p^{-10}}{ \mathcal{P}_j(t)} + p^{-10}.
    \end{align*}     
     The conclusion follows, which completes the proof of Theorem \ref{thm:moder_dis_mcks}. 

\subsubsection{Proof of Corollary \ref{cor:ratio_G}}
	For any $t \in (0, \mathcal{P}^{-1}(\alpha_n))$, applying Theorem \ref{thm:moder_dis_mcks} yields that 
	\begin{align*}
		\biggabs{\frac{\abs{\mathcal{H}_0 }^{-1} \sum_{j \in \mathcal{H}_0} \Prob (\widehat{\bm{W}}_j \geq t )}{\mathcal{P}(t)  } - 1} \vee \biggabs{\frac{\abs{\mathcal{H}_0 }^{-1} \sum_{j \in \mathcal{H}_0} \Prob (\widehat{\bm{W}}_j \leq -t )}{\mathcal{P}(t)  } - 1}  = \mathfrak{o}(1).
	\end{align*}
	The conclusion follows by some simple algebra, which completes the proof of Corollary \ref{cor:ratio_G}. 

\subsubsection{Proof of Lemma \ref{lem:indicator_app}}

\begin{lemma}\label{lem:var_bound}
Assume that conditions \eqref{cond:C1} and \eqref{cond:C2} hold. Then for $\mathsf{W} \in \{ \widehat{\bm{W}}, -\widehat{\bm{W}} \}$, there exists some  constant $C= C(\mathsf{M}) > 0$ such that for $n \geq C$ and  $t \in (0,n^{1/7}/C )$,
    \begin{align*}
    \var \Big( \sum_{ j \in \mathcal{H}_0}\bm{1} \{ \mathsf{W}_j \geq t \} \Big)\leq C \bigg\{\frac{(1+t^3)\log^3(np)}{\sqrt{n}} \bigg(\sum_{j \in \mathcal{H}_0} \Prob  \Big( \mathsf{W}_j \geq t   \Big)\bigg)^2 + p^{-10}\bigg\}.
    \end{align*}
\end{lemma}

\noindent\textit{Proof of Lemma \ref{lem:var_bound}}. 
We show only the bound for $\mathsf{W} = \widehat{\bm{W}}$, while the other cases can be dealt with similarly. Observe that by the law of total variance, we have 
    \begin{align*}
         \var \Big( \sum_{ j \in \mathcal{H}_0}\bm{1} \{ \widehat{\bm{W}}_j \geq t \} \Big) &= \E \Big[ \var \Big(  \sum_{ j \in \mathcal{H}_0}\bm{1} \{ \widehat{\bm{W}}_j \geq t \}  \,\big| \,\bm{Y}  \Big)  \Big] \\
         &+ \var \Big[ \E \Big(\sum_{ j \in \mathcal{H}_0}\bm{1} \{ \widehat{\bm{W}}_j \geq t \} \, \big| \,  \bm{Y}    \Big)  \Big] \equiv V_1 + V_2.
    \end{align*}
    We will bound terms $V_1$ and $V_2$ above separately. For term $V_1$, it holds that 
    \begin{align*}
       V_1 &= \E_{\bm{Y}} \bigg[ \E^{\bm{Y}} \Big(  \sum_{ j \in \mathcal{H}_0}\bm{1} \{ \widehat{\bm{W}}_j \geq t \}    \Big)^2 - \bigg(\E^{\bm{Y}}  \Big(  \sum_{ j \in \mathcal{H}_0}\bm{1} \{ \widehat{\bm{W}}_j \geq t \}   \Big)\bigg)^2 \bigg] \\
       &= \sum_{ j,\ell \in \mathcal{H}_0} \E_{\bm{Y} } \bigg[   \E^{\bm{Y} } \Big( \bm{1} \{ \widehat{\bm{W}}_j \geq t \}\bm{1} \{ \widehat{\bm{W}}_\ell \geq t \}  \Big) \\
       &\quad \quad \quad \quad - \E^{\bm{Y} } \Big(  \bm{1} \{ \widehat{\bm{W}}_j \geq t \}  \Big)\E^{\bm{Y} } \Big(  \bm{1} \{ \widehat{\bm{W}}_\ell \geq t \}  \Big) \bigg]\leq \sum_{\ell \in \mathcal{H}_0}\Prob (\widehat{\bm{W}}_\ell \geq t ).
    \end{align*}
    For term $V_2$, we have that 
    \begin{align*}
        V_2 &= \E_{\bm{Y}} \bigg[  \bigg(\E^{\bm{Y}}  \Big(  \sum_{ j \in \mathcal{H}_0}\bm{1} \{ \widehat{\bm{W}}_j \geq t \}   \Big)\bigg)^2 - \E^2 \Big( \sum_{ j \in \mathcal{H}_0}  \bm{1} \{ \widehat{\bm{W}}_j \geq t \}\Big)  \bigg] \\
        &= \sum_{ j,\ell \in \mathcal{H}_0}\E_{\bm{Y}} \bigg[  \E^{\bm{Y}}  \Big( \bm{1} \{ \widehat{\bm{W}}_j \geq t \}   \Big)\E^{\bm{Y}}  \Big( \bm{1} \{ \widehat{\bm{W}}_\ell \geq t \} \Big)- \E \Big( \bm{1} \{ \widehat{\bm{W}}_j \geq t \}\Big)\E \Big( \bm{1} \{ \widehat{\bm{W}}_\ell \geq t \}\Big)  \bigg].
    \end{align*}
    If $(X_j, \widehat{X}_j)$ follows a normal distribution, the distribution of $\widehat{\bm{W}}_j$ becomes independent of $\bm{Y}$, which entails that $V_2 = 0$ in this case. 
    
    For the general case, we can apply (\ref{ineq:est_conditionY}) in the proof of Theorem \ref{thm:moder_dis_mcks} to derive that
    \begin{align*}
        V_2 &\leq \sum_{ j,\ell \in \mathcal{H}_0}\E_{\bm{Y}} \bigg[  \mathcal{P}_j(t)  \Prob^{\bm{Y}}  \Big(\widehat{\bm{W}}_\ell \geq t \Big)\bm{1}\{ \mathcal{E}_Y\}  - \Prob \Big( \widehat{\bm{W}}_j \geq t \Big)\Prob \Big(\widehat{\bm{W}}_\ell \geq t \Big)  \bigg] \\
        &\quad + C\frac{(1+t^3)\log^3(np)}{\sqrt{n}} \sum_{j,\ell \in \mathcal{H}_0}\E_{\bm{Y}} \bigg[  \mathcal{P}_j(t)\Prob^{\bm{Y}} \Big(\widehat{\bm{W}}_\ell \geq t   \Big)\bigg] + Cp^{-10} \\
        &\leq \sum_{\ell \in \mathcal{H}_0} \Prob (\widehat{\bm{W}}_\ell \geq t )  \cdot \sum_{j \in \mathcal{H}_0}\bigabs{ \mathcal{P}_j(t)-  \Prob (\widehat{\bm{W}}_j \geq t )  }\\
        &\quad + C\frac{(1+t^3)\log^3(np)}{\sqrt{n}} \sum_{j,\ell \in \mathcal{H}_0}\E_{\bm{Y}} \bigg[  \mathcal{P}_j(t)\Prob^{\bm{Y}} \Big(\widehat{\bm{W}}_\ell \geq t   \Big)\bigg] + Cp^{-10}\\
        &\leq  C\frac{(1+t^3)\log^3(np)}{\sqrt{n}}\bigg(\sum_{j \in \mathcal{H}_0} \Prob  \Big(\widehat{\bm{W}}_j \geq t   \Big)\bigg)^2 + Cp^{-10}.
    \end{align*} 
    Here, recall $\mathcal{E}_Y$ defined in (\ref{ineq:def_of_E_Y}).
    Thus, combining the estimates for $V_1$ and $V_2$ above completes the proof of Lemma \ref{lem:var_bound}. 

\noindent\textit{Proof of Lemma \ref{lem:indicator_app}}. 
The proof largely adapts from that of \cite[Lemma 3]{fan2023ark}; we spell out some details below. We focus on the case when $\# = +$, as the case $\# = -$ can be analyzed in a similar fashion. Let $\widehat{G}_+(t) \equiv \abs{\mathcal{H}_0}^{-1}\sum_{j \in \mathcal{H}_0} \Prob(\widehat{\bm{W}}_j \geq t )$ for simplicity. We define a sequence $  0 < z_0 < z_1 < \cdots < z_{l_n} = \mathcal{P}(0)  $ and 
\begin{align*}
    t_i =\mathcal{P}^{-1} (z_i),
\end{align*}
where
\begin{align*}
   l_n = \bigg( \log \Big(\frac{\mathcal{P}(0)p}{h_n} -  \frac{\alpha_n}{h_n} \Big) \bigg)^{1/\gamma},\quad  z_0 = \alpha_n , \quad z_i = z_0 + \frac{h_n e^{i ^\gamma}}{p}\quad \text{for} \quad i \in [l_n] 
\end{align*}
with $0 < \gamma < 1$ and $h_n = (p\alpha_n)^\eta$ for some $\eta  \in (0, 1)$. 
For $t \in (0, \mathcal{P}^{-1}(\alpha_n) )$, we can find some $i \in [\ell_n - 1]$ such that $t \in [t_{i+1}, t_i]$ and 
\begin{align*}
    &\biggabs{\frac{ \sum_{j \in \mathcal{H}_0}  \bm{1}\{\widehat{\bm{W}}_j \geq t\}  }{\abs{\mathcal{H}_0} \cdot \widehat{G}_+(t) } - 1 }
    \leq  \Bigg\{\biggabs{\frac{ \sum_{j \in \mathcal{H}_0}  \bm{1}\{\widehat{\bm{W}}_j \geq t_{i+1} \} }{\abs{\mathcal{H}_0} \cdot \widehat{G}_+(t_i) } - 1 } \vee \biggabs{\frac{ \sum_{j \in \mathcal{H}_0}  \bm{1}\{\widehat{\bm{W}}_j  \geq t_{i}\} }{\abs{\mathcal{H}_0} \cdot \widehat{G}_+(t_{i+1}) } - 1 }\Bigg\}.
\end{align*}

Since $\mathcal{P}(t)$ is continuous and strictly decreasing, it holds that 
\begin{align}\label{eq:P_continuity}
    \sup_{i \in [l_n-1]} \abs{\mathcal{P} (t_i) /\mathcal{P} (t_{i+1}) -1 } =  \sup_{i \in [l_n-1]} \abs{z_i /z_{i+1} -1 } \to 0.
\end{align}
In view of Theorem \ref{thm:moder_dis_mcks}, we have that 
\begin{align}\label{eq:dis_approx}
    \biggabs{\frac{\widehat{G}_+(t_i)}{\mathcal{P} (t_i) } - 1} \vee \biggabs{\frac{\widehat{G}_+(t_{i+1})}{\mathcal{P} (t_{i+1}) } - 1} = \mathfrak{o}(1).
\end{align}
Combining (\ref{eq:P_continuity}) and (\ref{eq:dis_approx}) above, some straightforward algebra yields that 
\begin{align*}
    &\Bigg\{\biggabs{\frac{ \sum_{j \in \mathcal{H}_0}  \bm{1}\{\widehat{\bm{W}}_j \geq t_{i+1} \} }{\abs{\mathcal{H}_0} \cdot \widehat{G}_+(t_i) } - 1 } \vee \biggabs{\frac{ \sum_{j \in \mathcal{H}_0}  \bm{1}\{\widehat{\bm{W}}_j  \geq t_{i}\} }{\abs{\mathcal{H}_0} \cdot \widehat{G}_+(t_{i+1}) } - 1 }\Bigg\} \\
    &\leq (1 + \mathfrak{o}(1)) \cdot \Bigg\{\biggabs{\frac{ \sum_{j \in \mathcal{H}_0}  \bm{1}\{\widehat{\bm{W}}_j \geq t_{i+1} \} }{\abs{\mathcal{H}_0} \cdot \widehat{G}_+(t_{i+1}) } - 1 } \vee \biggabs{\frac{ \sum_{j \in \mathcal{H}_0}  \bm{1}\{\widehat{\bm{W}}_j  \geq t_{i}\} }{\abs{\mathcal{H}_0} \cdot \widehat{G}_+(t_{i}) } - 1 }\Bigg\} + \mathfrak{o}(1).
\end{align*}
Thus to prove the desired result, it is sufficient to show that
\begin{align*}
    D_n \equiv \sup_{0 \leq i \leq l_n} \biggabs{ \frac { \sum_{j \in \mathcal{H}_0 } \bm{1} (\widehat{\bm{W}}_j \geq t_i) } { \abs{\mathcal{H}_0} \cdot  \widehat{G}_+(t_i) } - 1}  = \mathfrak{o}_\mathbf{P}(1).
\end{align*}

For any $\epsilon > 0$, it follows from Lemma \ref{lem:var_bound} that 
\begin{align*}
    \Prob  ( D_n \geq \epsilon ) &\leq \sum_{i = 0}^{l_n} \Prob \bigg(\biggabs{ \frac { \sum_{j \in \mathcal{H}_0 } \big(\bm{1} (\widehat{\bm{W}}_j \geq t_i) - \Prob ( \widehat{\bm{W}}_j \geq t_i) \big) } { \abs{\mathcal{H}_0} \cdot  \widehat{G}_+(t_i) } } \geq \epsilon \bigg) \\
    &\leq \sum_{i = 0}^{l_n} \frac{\var \big(\sum_{j \in \mathcal{H}_0 } \big(\bm{1} (\widehat{\bm{W}}_j \geq t_i) \big)   }{ \abs{\mathcal{H}_0}^2 \cdot  \widehat{G}_+^2(t_i) \cdot \epsilon^2}\\
    &\leq C\frac{l_n (1 + t_{0}^3) \log^3 (np)}{\epsilon^2 \sqrt{n}} + \frac{p^{-10}}{\abs{\mathcal{H}_0}^2 \cdot \epsilon^2}\cdot \sum_{i = 0}^{l_n}\frac{1}{ \widehat{G}^2_+(t_i) }. 
\end{align*}
Moreover, by Theorem \ref{thm:moder_dis_mcks}, it holds that
\begin{align*}
	\sum_{i = 0}^{l_n}\frac{1}{\widehat{G}^2_+(t_i) } &= \sum_{i = 0}^{l_n}\frac{1}{ \mathcal{P}^2(t_i)  }  + \sum_{i = 0}^{l_n}\frac{\mathcal{P}^2(t_i) - \widehat{G}^2_+(t_i) }{\widehat{G}^2_+(t_i)\mathcal{P}^2(t_i)} \\
	&=\sum_{i = 0}^{l_n}\frac{1}{z_i^2 } + \sum_{i = 0}^{l_n}\frac{\mathcal{P}^2(t_i)/\widehat{G}^2_+(t_i) - 1 }{\mathcal{P}^2(t_i)} \\
	&= \sum_{i = 0}^{l_n}\frac{1}{z_i^2 } + \sum_{i = 0}^{l_n}\frac{\mathcal{P}^2(t_i)/\widehat{G}^2_+(t_i) - 1 }{z_i^2} \overset{(\ref{eq:dis_approx})}{=}  \sum_{i = 0}^{l_n}\frac{1 + \mathfrak{o}(1)}{z_i^2 } = \mathcal{O}( \alpha_n^{-2} l_n  ).
\end{align*}
Therefore, combining the estimates in the above two expressions concludes the proof of Lemma \ref{lem:indicator_app}. 

\subsubsection{Proof of Lemma \ref{lem:local_T}}

Let $C$ be the positive constant given in Proposition \ref{prop:conc_knockoff}. Denote by $\mathcal{A}_n \equiv \{j \in \mathcal{H}_1: \widehat{\bm{w}}_j \geq 3C{\delta}_n  \} $. It follows from Proposition \ref{prop:conc_knockoff} that 
\begin{align*}
    &\Prob \Big(  \widehat{\bm{W}}_j \leq 2C\delta_n \, \text{for some } j \in \mathcal{A}_n \Big) \leq  \sum_{j \in \mathcal{A}_n}\Prob \Big(  \widehat{\bm{W}}_j - \widehat{\bm{w}}_j \leq 2C\delta_n - \widehat{\bm{w}}_j  \Big)\\
    &\leq  \sum_{j \in \mathcal{A}_n}\Prob \Big(  \abs{\widehat{\bm{W}}_j - \widehat{\bm{w}}_j} \geq C\delta_n \Big) \leq \sum_{j \in [p]} \Prob\Big(\abs{\widehat{\bm{W}}_j - \widehat{\bm{w}}_j } \geq C\delta_n \Big) = \mathfrak{o}(1).
\end{align*}
This implies that event $\mathcal{E}_1 \equiv  \{ \widehat{\bm{W}}_j > 2C\delta_n  \, \text{for all } j \in \mathcal{A}_n \} $  holds with asymptotic probability $1$. Hence, on event $\mathcal{E}_1$, it holds that 
\begin{align*}
    \sum_{j \in [p]} \bm{1}\{\widehat{\bm{W}}_j > 2C\delta_n \} \geq a_n.
\end{align*}
Furthermore, as $\widehat{\bm{W}}_j \geq 0 >  -C\delta_n$ for all $j \in [p]$ (by definition), we can derive that
\begin{align*}
   \sum_{j \in [p]} \Prob \Big( \widehat{\bm{W}}_j < -2C\delta_n \Big) &=\sum_{j \in [p]} \Prob \Big( \widehat{\bm{W}}_j - \widehat{\bm{w}}_j < -2C\delta_n - \widehat{\bm{w}}_j\Big) \\
   &\leq \sum_{j \in [p]} \Prob \Big( \abs{\widehat{\bm{W}}_j - \widehat{\bm{w}}_j } > C\delta_n \Big)   = \mathfrak{o}(1).
\end{align*}

With an application of Theorem \ref{thm:moder_dis_mcks}, we can show that for large enough $n$,
\begin{align*}
    \mathcal{P}(2C\delta_n) 
    &= \frac{1}{\mathcal{H}_0} \sum_{j \in \mathcal{H}_0 }\mathcal{P}_j(2C\delta_n)   = \frac{1 + \mathfrak{o}(1)}{\abs{\mathcal{H}_0}}\sum_{j \in \mathcal{H}_0} \Prob \big( \widehat{\bm{W}}_j \leq -2C\delta_n  \big) \leq \frac{q a_n }{2p},
\end{align*}
which entails that $\mathcal{P}^{-1} \big( \frac{q  a_n }{2p}\big) \leq 2C\delta_n = \mathcal{O}(\log p)$. Thus, on event $\mathcal{E}_1$, it holds that 
\begin{align}\label{ineq:localT_est_1}
     \sum_{j \in [p]} \bm{1} \bigg\{\widehat{\bm{W}}_j > \mathcal{P}^{-1} \bigg( \frac{q  a_n }{2p}\bigg)\bigg\} \geq a_n.
\end{align}
On the other hand, as $\mathcal{P}^{-1} \big( \frac{q  a_n }{2p}\big) \leq 2C\delta_n = \mathfrak{o}(n^{1/7})$ for large enough $n$, we can resort to Theorem \ref{thm:moder_dis_mcks}  and Lemma \ref{lem:indicator_app} to obtain that 
\begin{align*}
     \frac{q  a_n }{2p} &= \mathcal{P}\bigg( \mathcal{P}^{-1} \bigg( \frac{q  a_n }{2p}\bigg)\bigg)=\frac{1 + \mathfrak{o}(1)}{\abs{\mathcal{H}_0}} \sum_{j \in \mathcal{H}_0}\Prob \bigg( \widehat{\bm{W}}_j \leq - \mathcal{P}^{-1} \bigg( \frac{q  a_n }{2p}\bigg)\bigg) \\
    &= \big(1 + \mathfrak{o}_\mathbf{P}(1)  \big)\cdot \frac{1}{\abs{\mathcal{H}_0}}\sum_{j \in \mathcal{H}_0}\bm{1} \bigg\{\widehat{\bm{W}}_j <-  \mathcal{P}^{-1} \bigg( \frac{q  a_n }{2p}\bigg)\bigg\}. 
\end{align*}
This implies that for some constant $c$ satisfying $0 < c < 1$, with asymptotic probability $1$ we have 
\begin{align*}
    1 + \sum_{j \in \mathcal{H}_0}\bm{1} \bigg\{\widehat{\bm{W}}_j <-  \mathcal{P}^{-1} \bigg( \frac{q  a_n }{2p}\bigg) \bigg\} \leq cq a_n.
\end{align*}

Moreover, under condition \eqref{cond:C4}, it follows that 
\begin{align*}
    &\Prob \bigg( \sum_{j \in \mathcal{H}_1} \bm{1} \bigg\{\widehat{\bm{W}}_j <- \mathcal{P}^{-1} \bigg( \frac{q  a_n }{2p}\bigg) \bigg\} \leq(1 - c)q a_n \bigg)  \\
    &\leq \frac{1}{(1 - c)q a_n} \sum_{j \in \mathcal{H}_1}\Prob \bigg( \widehat{\bm{W}}_j <- \mathcal{P}^{-1} \bigg( \frac{q  a_n }{2p}\bigg) \bigg) \to 0.
\end{align*}
A combination of the estimates in the above two expressions yields that with asymptotic probability $1$, 
\begin{align}\label{ineq:localT_est_2}
    1 +  \sum_{j \in [p]}  \bm{1} \bigg\{ \widehat{\bm{W}}_j < -   \mathcal{P}^{-1} \bigg( \frac{q  a_n }{2p}\bigg)\bigg\} \leq qa_n.
\end{align}
Hence, combining (\ref{ineq:localT_est_1}) and (\ref{ineq:localT_est_2}), we have that with asymptotic probability $1$,
\begin{align*}
    1 +  \sum_{j \in [p]}  \bm{1} \bigg\{ \widehat{\bm{W}}_j < -   \mathcal{P}^{-1} \bigg( \frac{q  a_n }{2p}\bigg)\bigg\} \overset{(\ref{ineq:localT_est_2})}{\leq} q a_n \overset{(\ref{ineq:localT_est_1})}{\leq} q \sum_{j \in [p]}  \bm{1} \bigg\{ \widehat{\bm{W}}_j \geq  \mathcal{P}^{-1} \bigg( \frac{q  a_n }{2p}\bigg) \bigg\}.
\end{align*}
 The conclusion follows from the definition of $T_q$, which completes the proof of Lemma \ref{lem:local_T}. 

\subsection{Proofs for Section \ref{sec:proof_FDR_mcks_heavy}}\label{sec:proof_sec_FDR_mcks_heavy}

\subsubsection{Proof of Theorem \ref{thm:dis_mcks_be}}
	We will focus on the case when $\mathsf{W} = \widehat{\bm{W}}$, since the other cases can be treated similarly. Notice that for any $j \in \mathcal{H}_0$, $\bm{X}_{\cdot,j}$, $\widehat{\bm{X}}_{\cdot,j}$, and $\bm{Y}$ are independent. For $i \in [n]$, denote by $\zeta^{(j)}_i \equiv ( \bm{X}_{i,j}\bm{Y}_i/ \sigma_j\pnorm{\bm{Y}}{},  \widehat{\bm{X}}_{i,j}\bm{Y}_i/\sigma_j\pnorm{\bm{Y}}{})^\top$. Using the  multidimensional version of the Berry--Esseen theorem, we have that for any $j \in \mathcal{H}_0$,
    \begin{align}\label{ineq:Y_1}
      \sup_{U \in \mathcal{C}}  \biggabs{ \Prob^{\bm{Y}} \bigg( \sum_{i \in [n]} {\zeta}^{(j)}_i \in U \bigg)-  \Prob \bigg( \mathsf{G}^{(j)}  \in U \bigg)} \leq C \sum_{i \in [n]} \E^{\bm{Y}} \pnorm{ {\zeta}^{(j)}_i }{}^3. 
    \end{align}
    Here, $\mathsf{G}^{(j)}  \sim \mathcal{N}(0, \sigma_j^2 I_2)$ and $\mathcal{C}$ is the collection of all (measurable) convex sets in $\R^2$. Further, we can deduce that 
    \begin{align}\label{ineq:Y_2}
        \sum_{i \in [n]} \E^{\bm{Y}} \pnorm{ {\zeta}^{(j)}_i }{}^3 &= \sum_{i \in [n]} \E^{\bm{Y}}  \Bigg(\frac{\bm{X}_{i,j}^2 \bm{Y}_i^2}{\pnorm{\bm{Y}}{}^2} +\frac{\widehat{\bm{X}}_{i,j}^2 \bm{Y}_i^2}{\pnorm{\bm{Y}}{}^2} \Bigg)^{3/2}\nonumber \\
        &\leq 4  \sum_{i \in [n]} \Bigg[ \E^{\bm{Y}}  \Bigg(\frac{\bm{X}_{i,j}^3 \bm{Y}_i^3}{\pnorm{\bm{Y}}{}^3}\Bigg) +  \E^{\bm{Y}}  \Bigg(\frac{\widehat{\bm{X}}_{i,j}^3 \bm{Y}_i^3}{\pnorm{\bm{Y}}{}^3}\Bigg)\Bigg]\leq   \frac{8\max_{i \in [n]} \abs{\bm{Y}_i}}{\pnorm{\bm{Y}}{} }\cdot \E X_j^3 .
    \end{align}
    
The conclusion follows from the sub-exponential tail of $\bm{Y}_i$'s and the following bound using the law of total probability 
\begin{align*}
     &\sup_{U \in \mathcal{C}}  \biggabs{\Prob \bigg( \sum_{i \in [n]} {\zeta}^{(j)}_i \in U \bigg)-  \Prob \bigg( \mathsf{G}^{(j)}  \in U \bigg)} \\
     &= \sup_{U \in \mathcal{C}}  \biggabs{ \E_{\bm{Y}} \bigg[ \Prob^{\bm{Y}} \bigg( \sum_{i \in [n]} {\zeta}^{(j)}_i \in U \bigg) -\Prob \bigg( \mathsf{G}^{(j)}  \in U \bigg)\bigg] } \lesssim  \E_{\bm{Y}} \bigg[ \frac{\max_{i \in [n]} \abs{\bm{Y}_i}}{\pnorm{\bm{Y}}{} } \bigg] . 
\end{align*}
The proof is concluded by observing that sets $\mathcal{S}_{t,\pm}$ defined in \eqref{def:knock_sets} are unions of disjoint convex sets, which completes the proof of Theorem \ref{thm:dis_mcks_be}. 

\subsubsection{Proof of Proposition \ref{prop:conc_knockoff_heavy}}
The positive constant $C$ may depend on $\mathsf{M}$ and $\mathfrak{q}$ which may change from line to line hereafter. Note that from Lemma \ref{thm:conc_nonnegative}, we have 
\begin{align*}
	\Prob\Big( \mathcal{E}_1 \equiv \Big\{\pnorm{\bm{Y}}{}^2     \geq n\E Y^2/2  \Big\} \Big)  \geq 1 - \exp( -n \E^2 Y^2/C  ).
\end{align*}
It follows from the concentration of the norm for the sub-exponential random vector \cite[Proposition 2.2]{sambale2023some} that 
	\begin{align*}
		\Prob\Big( \mathcal{E}_2 \equiv \Big\{\bigabs{\pnorm{\bm{Y}}{} - \sqrt{n}\E^{1/2}Y^2  }  \leq C\log p \Big\} \Big)  \geq 1 - p^{-10}.
	\end{align*}
    
	We next consider the concentration for $ \bm{X}_{\cdot,j}^\top \bm{Y}$.  Let us rewrite
    \begin{align*}
     \bm{X}_{\cdot,j}^\top \bm{Y}- \E  \bm{X}_{\cdot,j}^\top \bm{Y} = \sum_{i \in [n]} \big( \bm{X}_{i,j} \bm{Y}_i - \E  \bm{X}_{i,j} \bm{Y}_i\big) \equiv  \sum_{i \in [n]} \Delta_i^{(j)},
    \end{align*}
	and thus, $\Delta_1^{(j)}, \ldots, \Delta_n^{(j)}$ are i.i.d. mean $0$ random variables. Using Rosenthal's inequality, cf. Lemma \ref{thm:rosenthal}, we can show that for any $\mathfrak{q} \geq 2$,
    \begin{align*}
        \E \biggabs{\sum\nolimits_{i \in [n]} \Delta_i^{(j)}}^\mathfrak{q} \leq C \sum\nolimits_{i \in [n]} \E \abs{\Delta_i^{(j)}}^\mathfrak{q} + \bigg( \sum\nolimits_{i \in [n]} \E \abs{\Delta_i^{(j)}}^2 \bigg)^{\mathfrak{q}/2}\leq C (n + n^{\mathfrak{q}/2}) \leq C n^{\mathfrak{q}/2}.
    \end{align*}
    Applying the Markov inequality yields that for any $t \geq 0$,
    \begin{align*}
        \Prob \bigg(\bigabs{ \abs{\bm{X}_{\cdot,j}^\top \bm{Y}} - \abs{\E \bm{X}_{\cdot,j}^\top \bm{Y}}  } \geq \sqrt{n}t \bigg) \leq \Prob \bigg( \bigabs{\sum\nolimits_{i \in [n]}\Delta_i^{(j)}   } \geq \sqrt{n}t \bigg) \leq \frac{\E \bigabs{\sum\nolimits_{i \in [n]} \Delta_i^{(j)}}^\mathfrak{q}}{(\sqrt{n}t)^\mathfrak{q}} \leq \frac{C}{t^\mathfrak{q}}.
    \end{align*}
    Similarly, for any $t \geq 0$, it holds that 
    \begin{align*}
    	 \Prob \bigg(\bigabs{ \abs{\widehat{\bm{X}}_{\cdot,j}^\top \bm{Y}} - \abs{\E \widehat{\bm{X}}_{\cdot,j}^\top \bm{Y}}  } \geq \sqrt{n}t \bigg) \leq \frac{C}{t^\mathfrak{q}}.
    \end{align*}
    Choosing $t = p^{2/\mathfrak{q}} \log p$, we can obtain that 
    \begin{align*}
    	\Prob \Bigg( \mathcal{E}_3\equiv  \bigg\{ \bigabs{n^{-1}\bm{X}_{\cdot,j}^\top \bm{Y} - \E X_jY     } \vee  \bigabs{n^{-1}\widehat{\bm{X}}_{\cdot,j}^\top \bm{Y} - \E \widehat{X}_jY     }   \leq C\frac{ p^{2/\mathfrak{q}}\log p }{\sqrt{n} } \bigg\} \Bigg) \geq 1 - \mathfrak{o}(p^{-2}).
    \end{align*}
	
    Therefore, on event $ \mathcal{E}_1 \cap \mathcal{E}_2 \cap \mathcal{E}_3 $, we can deduce that 
	\begin{align*}
		&\biggabs{\frac{\abs{ \bm{X}_{\cdot,j}^\top \bm{Y} }  }{\pnorm{\bm{Y}}{}} - \frac{\sqrt{n} \cdot \abs{\E X_j Y}  }{ \E^{1/2} Y^2 }} \\
		&=\biggabs{ \frac{n^{-1/2}\abs{ \bm{X}_{\cdot,j}^\top \bm{Y} } -   \sqrt{n} \cdot \abs{\E X_j Y}  }{ \E^{1/2} Y^2 }} + \abs{ \bm{X}_{\cdot,j}^\top \bm{Y} } \cdot  \frac{ \bigabs{  \pnorm{\bm{Y}}{} - \sqrt{n} \E^{1/2}Y^2   } }{\sqrt{n} \pnorm{\bm{Y}}{} \cdot \E^{1/2} Y^2 } \\
		&\leq  \frac{Cp^{2/\mathfrak{q}}\log p}{ \E^{1/2} Y^2} + (n\cdot \abs{\E X_jY} + C\sqrt{n}p^{2/\mathfrak{q}}\log p) \cdot \frac{C\log p }{n \E Y^2 } \\
		&\leq Cp^{2/\mathfrak{q}} \log p \cdot \bigg( \frac{1}{\E^{1/2}Y^2} + \frac{\abs{\E X_jY} + \log p / \sqrt{n}}{\E Y^2} \bigg) \\
		&\leq Cp^{2/\mathfrak{q}}\log p \cdot \bigg( \frac{1}{\E^{1/2}Y^2} + \frac{\E^{1/2} X_j^2 \E^{1/2} Y^2 + \log p / \sqrt{n}}{\E Y^2} \bigg) \\
		&\leq Cp^{2/\mathfrak{q}}\log p \cdot \bigg( \frac{1}{\E^{1/2}Y^2} +  \frac{\log p}{\sqrt{n\E Y^2}}\bigg). 
	\end{align*}
	Similarly, it holds that on event $ \mathcal{E}_1 \cap \mathcal{E}_2 \cap \mathcal{E}_3 $,
	\begin{align*}
		\frac{\abs{ \widehat{\bm{X}}_{\cdot,j}^\top \bm{Y} }  }{\pnorm{\bm{Y}}{}}  = \biggabs{\frac{\abs{ \widehat{\bm{X}}_{\cdot,j}^\top \bm{Y} }  }{\pnorm{\bm{Y}}{}} - \frac{\sqrt{n} \cdot \abs{\E \widehat{X}_j Y}  }{ \E^{1/2} Y^2 }}\leq  Cp^{2/\mathfrak{q}}\log p \cdot \bigg( \frac{1}{\E^{1/2}Y^2} +  \frac{\log p}{\sqrt{n\E Y^2}}\bigg).
	\end{align*}
	The conclusion follows from an application of the triangle inequality, which completes the proof of Proposition \ref{prop:conc_knockoff_heavy}. 

\subsubsection{Proof of Lemma \ref{lem:T_qandIndicator_heavy}}\label{sec:proof_T_qandIndicator_haevy}
	The proof follows similar arguments as those for Lemmas \ref{lem:indicator_app} and \ref{lem:local_T}, so we only sketch the differences here. We will first prove the following variance bound: for $\mathsf{W} \in \{\widehat{\bm{W}},-\widehat{\bm{W}} \}$ and any $t \geq 0$, there exists some universal constant $C > 0$ such that 
	\begin{align}\label{ineq:variance_b_heavy}
	   \frac{ \var \Big( \sum_{ j \in \mathcal{H}_0}\bm{1} \{ \mathsf{W}_j \geq t \} \Big)}{\sum_{j \in \mathcal{H}_0} \Prob \big(\mathsf{W}_j \geq t \big) }\leq1 + C \abs{\mathcal{H}_0} \cdot \frac{\log n}{\sqrt{n} }.	
	\end{align}
	The proof of the above variance bound is similar to that of Lemma \ref{lem:var_bound}, where we will invoke Theorem \ref{thm:dis_mcks_be} to bound $V_2$ therein as 
	 \begin{align*}
        V_2 &\overset{(\ref{ineq:Y_1}) \&(\ref{ineq:Y_2})}{\leq} \sum_{ j,\ell \in \mathcal{H}_0}\E_{\bm{Y}} \bigg[ \mathcal{P}(t)  \Prob^{\bm{Y}}  \Big(\widehat{\bm{W}}_\ell \geq t   \Big)  - \Prob \Big( \widehat{\bm{W}}_j \geq t \Big)\Prob \Big(\widehat{\bm{W}}_\ell \geq t \Big)  \bigg] \\
        &\quad +C\abs{\mathcal{H}_0}\cdot \sum_{\ell \in \mathcal{H}_0}\E_{\bm{Y}} \bigg[  \frac{\max_{i \in [n]} \abs{\bm{Y}_i}}{\pnorm{\bm{Y}}{} } \cdot  \Prob^{\bm{Y} } \Big(\widehat{\bm{W}}_\ell \geq t   \Big)\bigg] \\
        &\leq \sum_{\ell \in \mathcal{H}_0} \Prob (\widehat{\bm{W}}_\ell \geq t )  \cdot \sum_{j \in \mathcal{H}_0}\bigabs{ \mathcal{P}(t) -  \Prob (\widehat{\bm{W}}_j \geq t )  }\\
        &\quad +C\abs{\mathcal{H}_0}\cdot \sum_{\ell \in \mathcal{H}_0}\E_{ \bm{Y} } \bigg[  \frac{\max_{i \in [n]} \abs{\bm{Y}_i}}{\pnorm{\bm{Y}}{} } \cdot \Prob^{ \bm{Y} } \Big(\widehat{\bm{W}}_\ell \geq t    \Big)\bigg]\\
        &\leq C\abs{\mathcal{H}_0} \cdot \sqrt{\frac{\log n}{n} } \cdot \sum_{\ell \in \mathcal{H}_0} \Prob (\widehat{\bm{W}}_\ell \geq t ). 
    \end{align*} 
    
    With the variance bound above, we can then follow the proof of Lemma \ref{lem:indicator_app} to prove the indicator approximation in (\ref{ineq:indicator_app_heavy}). The main difference is that we will bound $\Prob(D_n \geq \epsilon )$ therein as follows. Using the variance bound in (\ref{ineq:variance_b_heavy}), we have that 
    \begin{align*}
    	 \Prob  ( D_n \geq \epsilon ) &\leq \sum_{i = 0}^{l_n} \Prob \bigg(\biggabs{ \frac { \sum_{j \in \mathcal{H}_0 } \big(\bm{1} (\widehat{\bm{W}}_j \geq t_i) - \Prob ( \widehat{\bm{W}}_j \geq t_i) \big) } { \abs{\mathcal{H}_0} \cdot  \widehat{G}_+(t_i) } } \geq \epsilon \bigg) \\
    &\leq \sum_{i = 0}^{l_n} \frac{\var \big(\sum_{j \in \mathcal{H}_0 } \big(\bm{1} (\widehat{\bm{W}}_j \geq t_i) \big)   }{ \abs{\mathcal{H}_0}^2 \cdot  \widehat{G}_+^2(t_i) \cdot \epsilon^2}\\
    &\leq \epsilon^{-2}\cdot \bigg(1 + C \abs{\mathcal{H}_0}\frac{\log n}{ \sqrt{n} }\bigg) \cdot  \sum_{i = 0}^{l_n} \frac{1}{ \abs{\mathcal{H}_0}\cdot \widehat{G}_+(t_i) },
    \end{align*}
    and from Corollary \ref{cor:ratio_G_be}, it follows that 
\begin{align*}
    &\sum_{i = 0}^{l_n} \frac{1}{ \abs{\mathcal{H}_0}\cdot \widehat{G}_+(t_i) } = \sum_{i = 0}^{l_n} \frac{1}{ \abs{\mathcal{H}_0}\cdot \mathcal{P}(t_i)  } + \sum_{i = 0}^{l_n} \frac{\mathcal{P}(t_i) - \widehat{G}_+(t_i)}{ \abs{\mathcal{H}_0}\cdot\mathcal{P}(t_i) \cdot \widehat{G}_+(t_i) } \\
    &= \frac{1}{\abs{\mathcal{H}_0}} \sum_{i = 0}^{l_n} \frac{1 +\mathfrak{o}(1) }{  z_i}   = \frac{p}{\abs{\mathcal{H}_0}}\sum_{i = 0}^{l_n} \frac{1}{ qa_n/2 + h_n e^{i ^\gamma} }\cdot (1+\mathfrak{o}(1) ) \leq Ch_n^{-1} \to 0.
\end{align*}
Finally, with (\ref{ineq:indicator_app_heavy}) established, the localization of $T_q$ can be derived using the proof of Lemma \ref{lem:local_T}, substituting $\delta_n$ in that proof with $\delta_{n;\mathfrak{q}}$, and utilizing Proposition \ref{prop:conc_knockoff_heavy}. This completes the proof of Lemma \ref{lem:T_qandIndicator_heavy}. 

\subsection{Proofs for Section \ref{sec:proof_sec_thm:FDR_ols}}\label{sec:proof_sec_sec_thm:FDR_ols}

\subsubsection{Some a priori estimates}
\begin{lemma}\label{lem:quad_conc}
Assume that conditions \eqref{cond:O1}--\eqref{cond:O3} hold. Then for any deterministic unit vectors $\mathfrak{u}, \mathfrak{v} \in \R^{2p}$, there exists some constant $C = C(\mathsf{M},\tau,\mathsf{c}_{\mathsf{L}}) > 0$ such that for any $n \geq C$, 
\begin{align*}
    \Prob \bigg( \bigabs{\iprod{\mathfrak{u} }{ \big(\widehat{\mathsf{Z}}^\top \widehat{\mathsf{Z}} + \eta I_{2p} \big)^{-1} \mathfrak{v}}- \E \iprod{\mathfrak{u} }{ \big(\widehat{\mathsf{Z}}^\top \widehat{\mathsf{Z}} + \eta I_{2p} \big)^{-1} \mathfrak{v}} }\geq C \frac{\log^2 n}{\sqrt{n}} \bigg) \leq Cn^{-10}.
\end{align*}
\end{lemma}

\noindent\textit{Proof of Lemma \ref{lem:quad_conc}}. 
We will focus on the proof for $\widehat{\mathsf{Z}}$ with $\eta = 0$, while the proofs for $\widehat{\mathsf{Z}}$ and the general case of $\eta$ are similar and thus omitted. Let us define for any $j \in [n]$,
\begin{align*}
    \mathsf{D}_j \equiv \E \big[ \iprod{\mathfrak{u} }{ \big(\widehat{\mathsf{Z}}^\top \widehat{\mathsf{Z}} \big)^{-1} \mathfrak{v}}  \,|\, \widehat{\mathsf{Z}}_{1,\cdot},\ldots, \widehat{\mathsf{Z}}_{j,\cdot}  \big] - \E \big[ \iprod{\mathfrak{u} }{ \big(\widehat{\mathsf{Z}}^\top \widehat{\mathsf{Z}} \big)^{-1} \mathfrak{v}}  \,|\, \widehat{\mathsf{Z}}_{1,\cdot},\ldots, \widehat{\mathsf{Z}}_{j-1,\cdot}  \big].
\end{align*}
Denote by $\widehat{\mathsf{Z}}_{-j,\cdot} \in \R^{(n-1) \times 2p }$ the matrix obtained by removing the $j$th row of $\widehat{\mathsf{Z}}$. Then it holds that 
\begin{align*}
     \mathsf{D}_j &= \E \big[ \iprod{\mathfrak{u} }{ \big(\widehat{\mathsf{Z}}^\top \widehat{\mathsf{Z}} \big)^{-1} \mathfrak{v}} -\iprod{\mathfrak{u} }{ \big(\widehat{\mathsf{Z}}_{-j,\cdot}^\top \widehat{\mathsf{Z}}_{-j,\cdot} \big)^{-1} \mathfrak{v}}  \,|\, \widehat{\mathsf{Z}}_{1,\cdot},\ldots, \widehat{\mathsf{Z}}_{j,\cdot}  \big] \\
     &\quad + \E \big[ \iprod{\mathfrak{u} }{ \big(\widehat{\mathsf{Z}}_{-j,\cdot}^\top \widehat{\mathsf{Z}}_{-j,\cdot} \big)^{-1} \mathfrak{v}} -\iprod{\mathfrak{u} }{ \big(\widehat{\mathsf{Z}}^\top \widehat{\mathsf{Z}} \big)^{-1} \mathfrak{v}}  \,|\, \widehat{\mathsf{Z}}_{1,\cdot},\ldots, \widehat{\mathsf{Z}}_{j-1,\cdot}  \big] \equiv \mathsf{D}_{j,1} +  \mathsf{D}_{j,2}.
\end{align*}
An application of the Sherman--Morrison formula gives that  
\begin{align*}
\abs{\mathsf{D}_{j,1}} &\leq \E \big[ \abs{\iprod{\mathfrak{u} }{ \big(\widehat{\mathsf{Z}}^\top \widehat{\mathsf{Z}} \big)^{-1} \mathfrak{v}} -\iprod{\mathfrak{u} }{ \big(\widehat{\mathsf{Z}}_{-j,\cdot}^\top \widehat{\mathsf{Z}}_{-j,\cdot} \big)^{-1} \mathfrak{v}}}  \,|\, \widehat{\mathsf{Z}}_{1,\cdot},\ldots, \widehat{\mathsf{Z}}_{j,\cdot}  \big]\\
    &= \E \Bigg[ \biggabs{ \frac{\iprod{\mathfrak{u} }{\big(\widehat{\mathsf{Z}}_{-j,\cdot}^\top \widehat{\mathsf{Z}}_{-j,\cdot} \big)^{-1} \widehat{\mathsf{Z}}_{j,\cdot}   } \iprod{\mathfrak{v} }{\big(\widehat{\mathsf{Z}}_{-j,\cdot}^\top \widehat{\mathsf{Z}}_{-j,\cdot} \big)^{-1} \widehat{\mathsf{Z}}_{j,\cdot}   } }{1 + \iprod{\widehat{\mathsf{Z}}_{j,\cdot} }{\big(\widehat{\mathsf{Z}}_{-j,\cdot}^\top \widehat{\mathsf{Z}}_{-j,\cdot} \big)^{-1} \widehat{\mathsf{Z}}_{j,\cdot}   } } }
 \, \bigg| \,  \widehat{\mathsf{Z}}_{1,\cdot},\ldots, \widehat{\mathsf{Z}}_{j,\cdot}\Bigg] \\
 &\leq  \E \big[ \abs{  \iprod{\mathfrak{u} }{\big(\widehat{\mathsf{Z}}_{-j,\cdot}^\top \widehat{\mathsf{Z}}_{-j,\cdot} \big)^{-1} \widehat{\mathsf{Z}}_{j,\cdot}   } \iprod{\mathfrak{v} }{\big(\widehat{\mathsf{Z}}_{-j,\cdot}^\top \widehat{\mathsf{Z}}_{-j,\cdot} \big)^{-1} \widehat{\mathsf{Z}}_{j,\cdot}   }    }
 \, | \,  \widehat{\mathsf{Z}}_{1,\cdot},\ldots, \widehat{\mathsf{Z}}_{j,\cdot}\big].
\end{align*}

By the convex concentration in condition \eqref{cond:O2}, for some  $c = c(\mathsf{c}_{\mathsf{L}}) > 0$ we have 
\begin{align*}
    \Prob^{ \widehat{\mathsf{Z}}_{-j,\cdot}}  \Big(  \abs{\iprod{\mathfrak{u} }{\big(\widehat{\mathsf{Z}}_{-j,\cdot}^\top \widehat{\mathsf{Z}}_{-j,\cdot} \big)^{-1} \widehat{\mathsf{Z}}_{j,\cdot}   }} \geq t  \Big) &=     \Prob^{ \widehat{\mathsf{Z}}_{-j,\cdot}}  \Big(  \abs{\iprod{\mathfrak{u} }{\big(\widehat{\mathsf{Z}}_{-j,\cdot}^\top \widehat{\mathsf{Z}}_{-j,\cdot} \big)^{-1} \Sigma^{1/2} \widehat{\mathsf{Q}}_{j,\cdot}   }} \geq t  \Big) \\
    &\leq \exp \bigg(- \frac{cnt^2}{\pnorm{\Sigma^{1/2}\big(\widehat{\mathsf{Z}}_{-j,\cdot}^\top \widehat{\mathsf{Z}}_{-j,\cdot} \big)^{-1} \mathfrak{u}  }{}^2  } \bigg).
\end{align*}
Since $\lambda_{\min}(\widehat{\mathsf{Z}}_{-j,\cdot}^\top \widehat{\mathsf{Z}}_{-j,\cdot}  ) \geq \lambda_{\min}(\widehat{\mathsf{Z}}^\top \widehat{\mathsf{Z}})$, it follows that 
$\pnorm{\Sigma^{1/2}\big(\widehat{\mathsf{Z}}_{-j,\cdot}^\top \widehat{\mathsf{Z}}_{-j,\cdot} \big)^{-1} \mathfrak{u}  }{}^2 \leq \mathsf{M} \cdot \pnorm{\big(\widehat{\mathsf{Z}}^\top \widehat{\mathsf{Z}} \big)^{-1} }{}^2$.
Hence, using Lemma \ref{lem:rigidity} and the law of total probability, we can show that there exists some constant $C_2 = C_2(\mathsf{M},\tau, \mathsf{c}_\mathsf{L}) > 0$ such that for large enough $n$,
\begin{align}\label{ineq:liner_conc1}
    \Prob\bigg(  \abs{\iprod{\mathfrak{u} }{\big(\widehat{\mathsf{Z}}_{-j,\cdot}^\top \widehat{\mathsf{Z}}_{-j,\cdot} \big)^{-1} \widehat{\mathsf{Z}}_{j,\cdot}   }} \geq C_2\sqrt{\frac{\log n}{n}}  \bigg) \leq C_2n^{-10} .
\end{align}
Similarly, it holds that 
\begin{align}\label{ineq:liner_conc2}
    \Prob\bigg(  \abs{\iprod{\mathfrak{v} }{\big(\widehat{\mathsf{Z}}_{-j,\cdot}^\top \widehat{\mathsf{Z}}_{-j,\cdot} \big)^{-1} \widehat{\mathsf{Z}}_{j,\cdot}   }} \geq C_2\sqrt{\frac{\log n}{n}}  \bigg) \leq C_2n^{-10}.
\end{align}

Combining the above probability estimates, we can deduce that 
\begin{align*}
    &\Prob \bigg( \E \big[ \abs{  \iprod{\mathfrak{u} }{\big(\widehat{\mathsf{Z}}_{-j,\cdot}^\top \widehat{\mathsf{Z}}_{-j,\cdot} \big)^{-1} \widehat{\mathsf{Z}}_{j,\cdot}   } \iprod{\mathfrak{v} }{\big(\widehat{\mathsf{Z}}_{-j,\cdot}^\top \widehat{\mathsf{Z}}_{-j,\cdot} \big)^{-1} \widehat{\mathsf{Z}}_{j,\cdot}   }    }
 \, | \,  \widehat{\mathsf{Z}}_{1,\cdot},\ldots, \widehat{\mathsf{Z}}_{j,\cdot}\big] \leq C_2^2\cdot \frac{\log n}{n} \bigg)\\
 &\geq  \Prob \bigg( \abs{  \iprod{\mathfrak{u} }{\big(\widehat{\mathsf{Z}}_{-j,\cdot}^\top \widehat{\mathsf{Z}}_{-j,\cdot} \big)^{-1} \widehat{\mathsf{Z}}_{j,\cdot}   } \iprod{\mathfrak{v} }{\big(\widehat{\mathsf{Z}}_{-j,\cdot}^\top \widehat{\mathsf{Z}}_{-j,\cdot} \big)^{-1} \widehat{\mathsf{Z}}_{j,\cdot}   }    }
 \leq  C_2^2\cdot \frac{\log n}{n} \bigg) \\
 &\geq  \Prob\bigg(  \abs{\iprod{\mathfrak{u} }{\big(\widehat{\mathsf{Z}}_{-j,\cdot}^\top \widehat{\mathsf{Z}}_{-j,\cdot} \big)^{-1} \widehat{\mathsf{Z}}_{j,\cdot}   }} \leq C_2\sqrt{\frac{\log n}{n}}  \, , \, \abs{\iprod{\mathfrak{v} }{\big(\widehat{\mathsf{Z}}_{-j,\cdot}^\top \widehat{\mathsf{Z}}_{-j,\cdot} \big)^{-1} \widehat{\mathsf{Z}}_{j,\cdot}   }} \leq C_2\sqrt{\frac{\log n}{n}}  \bigg)  \\
 &\geq 1- 2C_2n^{-10}. 
\end{align*}
This implies that 
\begin{align*}
    \Prob \bigg( \abs{\mathsf{D}_{j,1}} \geq C_2^2\cdot \frac{\log n}{n}  \bigg) \leq 2C_2n^{-10}.
\end{align*}

Similarly, there exists some constant $C_3 = C_3(\mathsf{M},\tau,\mathsf{c}_{\mathsf{L}}) > 0$ such that for large enough $n$,
\begin{align*}
    \Prob \bigg( \abs{\mathsf{D}_{j,2}} \geq C_3^2\cdot \frac{\log n}{n}  \bigg) \leq 2C_2n^{-10}.
\end{align*}
Then a direct application of Proposition \ref{prop:generalize_azuma} yields that for some constant $C_4 = C_4(\mathsf{M},\tau,\mathsf{c}_{\mathsf{L}}) > 0$ and large enough $n$,
\begin{align*}
        \Prob \Big( \bigabs{\iprod{\mathfrak{u} }{ \big(\widehat{\mathsf{Z}}^\top \widehat{\mathsf{Z}} \big)^{-1} \mathfrak{v}}- \E \iprod{\mathfrak{u} }{ \big(\widehat{\mathsf{Z}}^\top \widehat{\mathsf{Z}} \big)^{-1} \mathfrak{v}} }\geq C_4\frac{\log^2 n}{\sqrt{n}}   \Big) \leq C_4n^{-10}.
\end{align*}
This completes the proof of Lemma \ref{lem:quad_conc}. 

\begin{lemma}\label{lem:expectation_bound}
Assume that conditions \eqref{cond:O1}--\eqref{cond:O3} hold. Then for any deterministic unit vectors $\mathfrak{u}, \mathfrak{v} \in \R^{2p}$, there exists some constant $C = C(\mathsf{M},\tau,\mathsf{c}_{\mathsf{L}}) > 0$ such that for any $\eta > 0$ and any $n\geq C$, 
    \begin{align*}
       \bigabs{ \E \iprod{\mathfrak{u} }{ \big({\mathsf{Z}}^\top {\mathsf{Z}} + \eta I_{2p} \big)^{-1} \mathfrak{v}} - \E \iprod{\mathfrak{u} }{ \big(\widehat{\mathsf{Z}}^\top \widehat{\mathsf{Z}}  + \eta I_{2p}\big)^{-1} \mathfrak{v}}} \leq C \bigg(  \frac{\log^{3/2} n}{\sqrt{n}} + \eta^{-1}n^{-10} \bigg).
    \end{align*}
    Here, ${\mathsf{Z}} \equiv ({\mathsf{Z}}_{1,\cdot},\ldots, {\mathsf{Z}}_{n,\cdot})^\top \in \R^{n \times 2p}$ with $\{\sqrt{n} {\mathsf{Z}}_{i,\cdot} \}_{i \in [n]} \overset{i.i.d.}{\sim} \mathcal{N}(0,\Sigma)$. 
\end{lemma}

\noindent\textit{Proof of Lemma \ref{lem:expectation_bound}}. 
    Let $  \mathsf{Z}^{(i)} \equiv ({\mathsf{Z}}_{1,\cdot},  \ldots, {\mathsf{Z}}_{i,\cdot}, \widehat{\mathsf{Z}}_{i+1,\cdot},\ldots, \widehat{\mathsf{Z}}_{n,\cdot} )^\top$, $\mathcal{G}^{(i)} \equiv  ((\mathsf{Z}^{(i)})^\top  \mathsf{Z}^{(i)} + \eta I_{2p})^{-1}$, and  $\mathcal{G}_{-j}^{(i)} \equiv  ((\mathsf{Z}_{-j,\cdot}^{(i)})^\top  \mathsf{Z}_{-j,\cdot}^{(i)}  + \eta I_{2p})^{-1}$ for simplicity. Let us first rewrite our target into the following telescoping sum
    \begin{align*}
         &\bigabs{ \E \iprod{\mathfrak{u} }{ \big({\mathsf{Z}}^\top {\mathsf{Z}} + \eta I_{2p} \big)^{-1} \mathfrak{v}} - \E \iprod{\mathfrak{u} }{ \big(\widehat{\mathsf{Z}}^\top \widehat{\mathsf{Z}} + \eta I_{2p} \big)^{-1} \mathfrak{v}}}=\biggabs{  \sum_{i \in [n]}  \E \iprod{\mathfrak{u} }{\mathcal{G}^{(i)} \mathfrak{v}} - \E \iprod{\mathfrak{u} }{\mathcal{G}^{(i-1)} \mathfrak{v}}}.
    \end{align*}
    With the aid of the Sherman--Morrison formula, we can deduce that 
    \begin{align*}
        &\iprod{\mathfrak{u} }{ \mathcal{G}^{(i)} \mathfrak{v}} - \iprod{\mathfrak{u} }{\mathcal{G}^{(i-1)}\mathfrak{v}} \\
        &= \iprod{\mathfrak{u} }{\mathcal{G}_{-i}^{(i)}  \mathfrak{v}} - \frac{\iprod{\mathfrak{u} }{ \mathcal{G}_{-i}^{(i)}{\mathsf{Z}}_{i,\cdot}}   \iprod{\mathfrak{v} }{ \mathcal{G}_{-i}^{(i)} {\mathsf{Z}}_{i,\cdot}}  }{1 + \iprod{{\mathsf{Z}}_{i,\cdot} }{ \mathcal{G}_{-i}^{(i)} {\mathsf{Z}}_{i,\cdot}} }  - \iprod{\mathfrak{u} }{ \mathcal{G}_{-i}^{(i-1)} \mathfrak{v}} +  \frac{\iprod{\mathfrak{u} }{ \mathcal{G}_{-i}^{(i-1)} \widehat{\mathsf{Z}}_{i,\cdot}}   \iprod{\mathfrak{v} }{ \mathcal{G}_{-i}^{(i-1)} \widehat{\mathsf{Z}}_{i,\cdot}}   }{1 + \iprod{\widehat{\mathsf{Z}}_{i,\cdot} }{ \mathcal{G}_{-i}^{(i-1)}   \widehat{\mathsf{Z}}_{i,\cdot}} }\\
        &= \frac{\iprod{\mathfrak{u} }{\mathcal{G}_{-i}^{(i)}  \widehat{\mathsf{Z}}_{i,\cdot}}   \iprod{\mathfrak{v} }{ \mathcal{G}_{-i}^{(i)}  \widehat{\mathsf{Z}}_{i,\cdot}}   }{1 + \iprod{\widehat{\mathsf{Z}}_{i,\cdot} }{ \mathcal{G}_{-i}^{(i)}    \widehat{\mathsf{Z}}_{i,\cdot}} } -\frac{\iprod{\mathfrak{u} }{ \mathcal{G}_{-i}^{(i)}{\mathsf{Z}}_{i,\cdot}}   \iprod{\mathfrak{v} }{ \mathcal{G}_{-i}^{(i)}{\mathsf{Z}}_{i,\cdot}}  }{1 + \iprod{{\mathsf{Z}}_{i,\cdot} }{ \mathcal{G}_{-i}^{(i)}  {\mathsf{Z}}_{i,\cdot}} } \\
        &=  \frac{\iprod{\mathfrak{u} }{\mathcal{G}_{-i}^{(i)} \widehat{\mathsf{Z}}_{i,\cdot}}   \iprod{\mathfrak{v} }{ \mathcal{G}_{-i}^{(i)}\widehat{\mathsf{Z}}_{i,\cdot}}   }{1 +\E_{\widehat{\mathsf{Z}}_{i,\cdot}}  \iprod{\widehat{\mathsf{Z}}_{i,\cdot} }{ \mathcal{G}_{-i}^{(i)}   \widehat{\mathsf{Z}}_{i,\cdot}} } -\frac{\iprod{\mathfrak{u} }{\mathcal{G}_{-i}^{(i)}{\mathsf{Z}}_{i,\cdot}}   \iprod{\mathfrak{v} }{ \mathcal{G}_{-i}^{(i)}{\mathsf{Z}}_{i,\cdot}}  }{1 + \E_{{\mathsf{Z}}_{i,\cdot} } \iprod{{\mathsf{Z}}_{i,\cdot} }{\mathcal{G}_{-i}^{(i)}  {\mathsf{Z}}_{i,\cdot}} } +  \widehat{\mathsf{Err}}_i - {\mathsf{Err}}_i\\
        &\equiv \widehat{I}_{i} - {I}_i +  \widehat{\mathsf{Err}}_i - {\mathsf{Err}}_i,
    \end{align*}
    where
    \begin{align*}
        \widehat{\mathsf{Err}}_i &\equiv \frac{\iprod{\mathfrak{u} }{\mathcal{G}_{-i}^{(i)} \widehat{\mathsf{Z}}_{i,\cdot}}   \iprod{\mathfrak{v} }{\mathcal{G}_{-i}^{(i)}\widehat{\mathsf{Z}}_{i,\cdot}}   }{1 + \iprod{\widehat{\mathsf{Z}}_{i,\cdot} }{ \mathcal{G}_{-i}^{(i)}   \widehat{\mathsf{Z}}_{i,\cdot}} } - \frac{\iprod{\mathfrak{u} }{\mathcal{G}_{-i}^{(i)} \widehat{\mathsf{Z}}_{i,\cdot}}   \iprod{\mathfrak{v} }{ \mathcal{G}_{-i}^{(i)} \widehat{\mathsf{Z}}_{i,\cdot}}   }{1 +\E_{\widehat{\mathsf{Z}}_{i,\cdot}}  \iprod{\widehat{\mathsf{Z}}_{i,\cdot} }{\mathcal{G}_{-i}^{(i)} \widehat{\mathsf{Z}}_{i,\cdot}} },  \\
        {\mathsf{Err}}_i&\equiv  \frac{\iprod{\mathfrak{u} }{\mathcal{G}_{-i}^{(i)}{\mathsf{Z}}_{i,\cdot}}   \iprod{\mathfrak{v} }{ \mathcal{G}_{-i}^{(i)}{\mathsf{Z}}_{i,\cdot}}  }{1 + \iprod{{\mathsf{Z}}_{i,\cdot} }{ \mathcal{G}_{-i}^{(i)}  {\mathsf{Z}}_{i,\cdot}} }- \frac{\iprod{\mathfrak{u} }{ \mathcal{G}_{-i}^{(i)} {\mathsf{Z}}_{i,\cdot}}   \iprod{\mathfrak{v} }{ \mathcal{G}_{-i}^{(i)}{\mathsf{Z}}_{i,\cdot}}  }{1 + \E_{{\mathsf{Z}}_{i,\cdot} } \iprod{{\mathsf{Z}}_{i,\cdot} }{ \mathcal{G}_{-i}^{(i)}   {\mathsf{Z}}_{i,\cdot}} } .
    \end{align*}
    
    Using the facts that
    \begin{itemize}
        \item $ \E_{{\mathsf{Z}}_{i,\cdot} } \iprod{{\mathsf{Z}}_{i,\cdot} }{\mathcal{G}_{-i}^{(i)}  {\mathsf{Z}}_{i,\cdot}} = \E_{\widehat{\mathsf{Z}}_{i,\cdot} } \iprod{\widehat{\mathsf{Z}}_{i,\cdot} }{ \mathcal{G}_{-i}^{(i)}   \widehat{\mathsf{Z}}_{i,\cdot}}$ and
        \item $\E_{{\mathsf{Z}}_{i,\cdot} } \iprod{\mathfrak{u} }{\mathcal{G}_{-i}^{(i)}{\mathsf{Z}}_{i,\cdot}}   \iprod{\mathfrak{v} }{ \mathcal{G}_{-i}^{(i)}{\mathsf{Z}}_{i,\cdot}} =\E_{\widehat{\mathsf{Z}}_{i,\cdot} } \iprod{\mathfrak{u} }{ \mathcal{G}_{-i}^{(i)} \widehat{\mathsf{Z}}_{i,\cdot}}   \iprod{\mathfrak{v} }{ \mathcal{G}_{-i}^{(i)} \widehat{\mathsf{Z}}_{i,\cdot}} $,
    \end{itemize}
    we can show that 
    \begin{align*}
        \E (\widehat{I}_{i} - {I}_i) = \E \Big( \E_{\widehat{\mathsf{Z}}_{i,\cdot} } \widehat{I}_i - \E_{{\mathsf{Z}}_{i,\cdot} } {I}_i  \Big) = 0.
    \end{align*}
    Thus, it remains to bound $\widehat{\mathsf{Err}}_i$ and ${\mathsf{Err}}_i$. 
    
    Observe that
    \begin{align*}
        \bigabs{\widehat{\mathsf{Err}}_i} &=\abs{ \iprod{\mathfrak{u} }{\mathcal{G}_{-i}^{(i)}\widehat{\mathsf{Z}}_{i,\cdot}}   \iprod{\mathfrak{v} }{ \mathcal{G}_{-i}^{(i)}\widehat{\mathsf{Z}}_{i,\cdot}}}\cdot \bigg(  \frac{\abs{  \iprod{{\mathsf{Z}}_{i,\cdot} }{ \mathcal{G}_{-i}^{(i)}   {\mathsf{Z}}_{i,\cdot}} -\E_{{\mathsf{Z}}_{i,\cdot} } \iprod{{\mathsf{Z}}_{i,\cdot} }{ \mathcal{G}_{-i}^{(i)}   {\mathsf{Z}}_{i,\cdot}} }  }{ \abs{ 1 + \iprod{{\mathsf{Z}}_{i,\cdot} }{ \mathcal{G}_{-i}^{(i)}   {\mathsf{Z}}_{i,\cdot}}  } \cdot \abs{1 + \E_{{\mathsf{Z}}_{i,\cdot}}\iprod{{\mathsf{Z}}_{i,\cdot} }{ \mathcal{G}_{-i}^{(i)}   {\mathsf{Z}}_{i,\cdot}} } } \bigg) \\
        &\leq \abs{ \iprod{\mathfrak{u} }{\mathcal{G}_{-i}^{(i)}\widehat{\mathsf{Z}}_{i,\cdot}} }\cdot \abs{ \iprod{\mathfrak{v} }{ \mathcal{G}_{-i}^{(i)}\widehat{\mathsf{Z}}_{i,\cdot}}}\cdot \abs{  \iprod{{\mathsf{Z}}_{i,\cdot} }{ \mathcal{G}_{-i}^{(i)}   {\mathsf{Z}}_{i,\cdot}} -\E_{{\mathsf{Z}}_{i,\cdot} } \iprod{{\mathsf{Z}}_{i,\cdot} }{ \mathcal{G}_{-i}^{(i)}   {\mathsf{Z}}_{i,\cdot}} }.
    \end{align*}
    Similar to (\ref{ineq:liner_conc1}) and (\ref{ineq:liner_conc2}), there exists some constant $C_1 = C_1(\mathsf{M},\tau,\mathsf{c}_{\mathsf{L}}) > 0$ such that
    \begin{align}\label{ineq:prob_est_12}
        \Prob \bigg(\abs{ \iprod{\mathfrak{u} }{\mathcal{G}_{-i}^{(i)}\widehat{\mathsf{Z}}_{i,\cdot}} } \geq C_1 \sqrt{\frac{\log n}{n} }  \bigg) \vee    \Prob \bigg(\abs{ \iprod{\mathfrak{v} }{\mathcal{G}_{-i}^{(i)}\widehat{\mathsf{Z}}_{i,\cdot}} } \geq C_1 \sqrt{\frac{\log n}{n} }  \bigg)\leq C_1n^{-10}.
    \end{align}
    Moreover, it follows from the Hanson--Wright inequality, cf. Lemma \ref{lem:Hanson-Wright}, that there exists some constant $C_2 = C_2(\mathsf{M},\tau,\mathsf{c}_{\mathsf{L}}) > 0$ such that for large enough $n$,
    \begin{align*}
        \Prob_{{\mathsf{Z}}_{i,\cdot}} \bigg(\abs{  \iprod{{\mathsf{Z}}_{i,\cdot} }{ \mathcal{G}_{-i}^{(i)}   {\mathsf{Z}}_{i,\cdot}}& -\E_{{\mathsf{Z}}_{i,\cdot} } \iprod{{\mathsf{Z}}_{i,\cdot} }{ \mathcal{G}_{-i}^{(i)}   {\mathsf{Z}}_{i,\cdot}} }\\
        &\geq C_2 \frac{\log^{1/2} n}{n} \cdot \bigg(\sum_{k,\ell \in [2p]}[\Sigma^{1/2}\mathcal{G}_{-i}^{(i)} \Sigma^{1/2}]_{k\ell}^2   \bigg)^{1/2}  \bigg) \leq C_2n^{-10}.
    \end{align*}
    
    In light of Lemma \ref{lem:rigidity}, there exists some constant $C_3 = C_3(\mathsf{M},\tau, \mathsf{c}_\mathsf{L})> 0$ such that $\Prob(\pnorm{ \Sigma^{1/2}\mathcal{G}_{-i}^{(i)} \Sigma^{1/2} }{\op} \leq C_3 ) \geq 1 - C_3 p^{-10}$. Hence, using the law of total probability, it holds for some (large) constant $C_4 = C_4(\mathsf{M},\tau, \mathsf{c}_\mathsf{L}) > 0$ that for large enough $n$,
    \begin{align}\label{ineq:prob_est_3}
        \Prob \bigg(\abs{  \iprod{{\mathsf{Z}}_{i,\cdot} }{ \mathcal{G}_{-i}^{(i)}   {\mathsf{Z}}_{i,\cdot}}& -\E_{{\mathsf{Z}}_{i,\cdot} } \iprod{{\mathsf{Z}}_{i,\cdot} }{ \mathcal{G}_{-i}^{(i)}   {\mathsf{Z}}_{i,\cdot}} } \geq C_4 \sqrt{ \frac{ \log n}{n}} \bigg) \leq C_4n^{-10}.
    \end{align}
    By (\ref{ineq:prob_est_12}) and (\ref{ineq:prob_est_3}), there exists some (large) constant $C_5 = C_5(\mathsf{M},\tau, \mathsf{c}_\mathsf{L}) > 0$ such that for large enough $n$,
    \begin{align*}
        \Prob \bigg(  \bigabs{\widehat{\mathsf{Err}}_i} \geq C_5 \bigg( \frac{\log n}{n} \bigg)^{3/2} \bigg) \leq C_5n^{-10}
    \end{align*}
   and similarly,
    \begin{align*}
        \Prob \bigg(  \bigabs{{\mathsf{Err}}_i} \geq C_5 \bigg( \frac{\log n}{n} \bigg)^{3/2} \bigg) \leq C_5n^{-10}.
    \end{align*}
    
    Therefore, using the deterministic bound of $ \abs{{\mathsf{Err}}_i} \vee \abs{\widehat{\mathsf{Err}}_i} \leq \eta^{-1}$, we see that there exists some (large) constant $C_6 = C_6(\mathsf{M},\tau, \mathsf{c}_\mathsf{L}) > 0$ such that for large enough $p$,
    \begin{align*}
        \E\abs{{\mathsf{Err}}_i} \vee \E\abs{\widehat{\mathsf{Err}}_i} \leq C_6 \bigg( \bigg( \frac{\log n}{n} \bigg)^{3/2} + \eta^{-1}n^{-10} \bigg).
    \end{align*}
    The conclusion then follows by the triangle inequality, which completes the proof of Lemma \ref{lem:expectation_bound}. 

\begin{lemma}\label{lem:expectaion_asymp}
Assume that conditions \eqref{cond:O1}--\eqref{cond:O3} hold. Then for any deterministic unit vectors $\mathfrak{u}, \mathfrak{v} \in \R^{2p}$, there exists some constant $C = C(\mathsf{M},\tau, \mathsf{c}_\mathsf{L}) > 0$ such that for any $\eta > 0$ and any $n \geq C$ , 
    \begin{align*}
        \biggabs{\E \iprod{\mathfrak{u} }{ \big(\widehat{\mathsf{Z}}^\top \widehat{\mathsf{Z}} + \eta I_{2p} \big)^{-1}\mathfrak{v}} - \frac{\mathfrak{u}^\top \Sigma^{-1} \mathfrak{v}}{1 - 2p/n}} \leq  C\bigg(  \frac{\log^{3/2} n}{\sqrt{n}} + \eta^{-1}n^{-10}  + \eta  \bigg).
    \end{align*}
\end{lemma}

\noindent\textit{Proof of Lemma \ref{lem:expectaion_asymp}}. 
    From Lemma \ref{lem:expectation_bound}, it suffices to calculate $ \E \iprod{\mathfrak{u} }{ \big({\mathsf{Z}}^\top {\mathsf{Z}} + \eta I_{2p} \big)^{-1} \mathfrak{v}}$ for $\mathsf{Z} = \mathsf{Q}\Sigma^{1/2}$ and $\{\mathsf{Q}_{ij}\}_{i \in [n],j\in [2p]} \overset{i.i.d.}{\sim} \mathcal{N}(0,n^{-1})$. Note that with $\tilde{\mathfrak{u}} \equiv \Sigma^{-1/2}\mathfrak{u}$ and $\tilde{\mathfrak{v}} \equiv \Sigma^{-1/2}\mathfrak{v}$,
    \begin{align*}
   &\tilde{\mathfrak{u}}^\top \tilde{\mathfrak{v}} - \eta \E \iprod{\tilde{\mathfrak{u}} }{ \Sigma^{-1}\big({\mathsf{Q}}^\top {\mathsf{Q}} + \eta\Sigma^{-1} \big)^{-1} \tilde{\mathfrak{v}}}=\E \iprod{\tilde{\mathfrak{u}} }{ {\mathsf{Q}}^\top {\mathsf{Q}} \big({\mathsf{Q}}^\top {\mathsf{Q}} + \eta \Sigma^{-1}  \big)^{-1} \tilde{\mathfrak{v}}} \\
      &=\sum_{i \in [n], j,k\in [2p] } \tilde{\mathfrak{u}}_j \tilde{\mathfrak{v}}_k \cdot  \E  \mathsf{Q}_{ij} [\mathsf{Q}\big({\mathsf{Q}}^\top {\mathsf{Q}} + \eta \Sigma^{-1}  \big)^{-1}   ]_{ik}. 
    \end{align*}
    By Stein's lemma, it holds that 
    \begin{align*}
        &\E  \mathsf{Q}_{ij} [\mathsf{Q}\big({\mathsf{Q}}^\top {\mathsf{Q}} + \eta \Sigma^{-1} \big)^{-1}   ]_{ik} = n^{-1} \E \partial_{ \mathsf{Q}_{ij}} \Big\{ [\mathsf{Q}\big({\mathsf{Q}}^\top {\mathsf{Q}} + \eta \Sigma^{-1}  \big)^{-1}   ]_{ik} \Big\} \\
        &=- n^{-1} \E \Big\{ [\mathsf{Q}\big({\mathsf{Q}}^\top {\mathsf{Q}} + \eta \Sigma^{-1}  \big)^{-1} ]_{ij}
[{\mathsf{Q}}  \big({\mathsf{Q}}^\top {\mathsf{Q}} + \eta \Sigma^{-1}  \big)^{-1}   ]_{ik} \Big\}\\
&\quad- n^{-1} \E \Big\{ [\mathsf{Q}\big({\mathsf{Q}}^\top {\mathsf{Q}} + \eta \Sigma^{-1}  \big)^{-1}{\mathsf{Q}}^\top ]_{ii}
[\big({\mathsf{Q}}^\top {\mathsf{Q}} + \eta \Sigma^{-1} \big)^{-1}   ]_{jk} \Big\}\\
        &\quad+n^{-1} \E \Big\{ [\big({\mathsf{Q}}^\top {\mathsf{Q}} + \eta \Sigma^{-1} \big)^{-1}   ]_{jk} \Big\}. 
    \end{align*}
    
    Combining the equations in the above two expressions yields that 
    \begin{align*}
        \tilde{\mathfrak{u}}^\top \tilde{\mathfrak{v}} &= \E \iprod{\tilde{\mathfrak{u}}  }{\big({\mathsf{Q}}^\top {\mathsf{Q}} + \eta \Sigma^{-1} \big)^{-1} \tilde{\mathfrak{v}} } \\
        &- n^{-1}\E \tr {\mathsf{Q}}  \big({\mathsf{Q}}^\top {\mathsf{Q}} + \eta \Sigma^{-1}  \big)^{-1} {\mathsf{Q}}^\top \cdot  \iprod{\tilde{\mathfrak{u}}  }{\big({\mathsf{Q}}^\top {\mathsf{Q}} + \eta \Sigma^{-1} \big)^{-1} \tilde{\mathfrak{v}} } \\
        &  - n^{-1} \E  \iprod{\tilde{\mathfrak{u}}  }{\big({\mathsf{Q}}^\top {\mathsf{Q}} + \eta \Sigma^{-1} \big)^{-1}{\mathsf{Q}}^\top {\mathsf{Q}}\big({\mathsf{Q}}^\top {\mathsf{Q}} + \eta \Sigma^{-1} \big)^{-1}  \tilde{\mathfrak{v}} } \\
        &+\eta \E \iprod{\tilde{\mathfrak{u}} }{ \Sigma^{-1}\big({\mathsf{Q}}^\top {\mathsf{Q}} + \eta\Sigma^{-1} \big)^{-1} \tilde{\mathfrak{v}}} \equiv \E \iprod{\tilde{\mathfrak{u}}  }{\big({\mathsf{Q}}^\top {\mathsf{Q}} + \eta \Sigma^{-1} \big)^{-1} \tilde{\mathfrak{v}} }+ I_1 + I_2 + I_3.
    \end{align*}
    For term $I_1$, we have that 
    \begin{align*}
        I_1  &= - n^{-1}\E \tr {\mathsf{Q}}^\top {\mathsf{Q}}  \big({\mathsf{Q}}^\top {\mathsf{Q}} + \eta \Sigma^{-1}  \big)^{-1}  \cdot  \iprod{\tilde{\mathfrak{u}}  }{\big({\mathsf{Q}}^\top {\mathsf{Q}} + \eta \Sigma^{-1} \big)^{-1} \tilde{\mathfrak{v}} } \\
        &= -\frac{2p}{n}\E \iprod{\tilde{\mathfrak{u}}  }{\big({\mathsf{Q}}^\top {\mathsf{Q}} + \eta \Sigma^{-1} \big)^{-1} \tilde{\mathfrak{v}} }  + \frac{\eta}{n} \E \tr \Sigma^{-1}\big({\mathsf{Q}}^\top {\mathsf{Q}} + \eta \Sigma^{-1}  \big)^{-1}\cdot \iprod{\tilde{\mathfrak{u}}  }{\big({\mathsf{Q}}^\top {\mathsf{Q}} + \eta \Sigma^{-1} \big)^{-1} \tilde{\mathfrak{v}} } \\
        &= -\frac{2p}{n}\E \iprod{\tilde{\mathfrak{u}}  }{\big({\mathsf{Q}}^\top {\mathsf{Q}} + \eta \Sigma^{-1} \big)^{-1} \tilde{\mathfrak{v}} } + \mathcal{O}\big( \eta \E \pnorm{\big({\mathsf{Q}}^\top {\mathsf{Q}} + \eta\Sigma^{-1} \big)^{-1} }{\op}^2 \big).
    \end{align*}
    For term $I_2$, it follows that 
    \begin{align*}
        I_2 &= -n^{-1} \E \iprod{\tilde{\mathfrak{u}}  }{\big({\mathsf{Q}}^\top {\mathsf{Q}} + \eta \Sigma^{-1} \big)^{-1} \tilde{\mathfrak{v}} } + \frac{\eta}{n} \E \iprod{\tilde{\mathfrak{u}}  }{\big({\mathsf{Q}}^\top {\mathsf{Q}} + \eta \Sigma^{-1} \big)^{-1} \Sigma^{-1} \big({\mathsf{Q}}^\top {\mathsf{Q}} + \eta \Sigma^{-1} \big)^{-1}  \tilde{\mathfrak{v}} } \\
        &= \mathcal{O}\big( n^{-1} \E \pnorm{\big({\mathsf{Q}}^\top {\mathsf{Q}} + \eta\Sigma^{-1} \big)^{-1} }{\op} + \eta n^{-1}\E \pnorm{\big({\mathsf{Q}}^\top {\mathsf{Q}} + \eta\Sigma^{-1} \big)^{-1} }{\op}^2  \big).
    \end{align*}
    For term $I_3$, it is easy to show that
    \begin{align*}
        I_3 = \mathcal{O}\big( \eta \E \pnorm{\big({\mathsf{Q}}^\top {\mathsf{Q}} + \eta\Sigma^{-1} \big)^{-1} }{\op} \big).
    \end{align*}
    
    Furthermore, we can estimate by \cite[Theorem 1.1]{rudelson2009smallest} that for $\ell \in [2]$ and sufficiently small constant $c > 0$,
    \begin{align*}
        \E \pnorm{\big({\mathsf{Q}}^\top {\mathsf{Q}} + \eta\Sigma^{-1} \big)^{-1} }{\op}^\ell &= \E \pnorm{\big({\mathsf{Q}}^\top {\mathsf{Q}} + \eta\Sigma^{-1} \big)^{-1} }{\op}^\ell \cdot \bm{1}\{ \lambda_{\min}({\mathsf{Q}}^\top {\mathsf{Q}}) \geq c) \} \\
        &\quad + \E \pnorm{\big({\mathsf{Q}}^\top {\mathsf{Q}} + \eta\Sigma^{-1} \big)^{-1} }{\op}^\ell \cdot \bm{1}\{ \lambda_{\min}({\mathsf{Q}}^\top {\mathsf{Q}}) < c) \} \\
        &\leq c^{-\ell} + \eta^{-\ell}\Prob \Big( \lambda_{\min}({\mathsf{Q}}^\top {\mathsf{Q}}) < c \Big) = \mathcal{O}(1 + \eta^{-\ell}n^{-10} ).
    \end{align*}
    Combining the above results leads to 
    \begin{align*}
       \tilde{\mathfrak{u}}^\top \tilde{\mathfrak{v}} &= \bigg(1 - \frac{2p}{n}\bigg) \cdot \E \iprod{\tilde{\mathfrak{u}} }{ \big({\mathsf{Q}}^\top {\mathsf{Q}} + \eta\Sigma^{-1} \big)^{-1} \tilde{\mathfrak{v}}} + \mathcal{O}\big(\eta + \eta^{-1}n^{-10} + n^{-1} \big).
    \end{align*} 
    The conclusion follows from rearranging the terms and applying Lemma \ref{lem:expectation_bound}, which completes the proof of Lemma \ref{lem:expectaion_asymp}. 

\subsubsection{Proof of Theorem \ref{thm:moder_dis_ols}}
We need the following aprior estimates to prove Theorem \ref{thm:moder_dis_ols}.  Recall that for $j \in [2p]$,
    \begin{align*}
    	\widehat{\beta}_j^{\mathsf{LS}} = \sqrt{n} \beta_{\ast,j}^{\mathsf{au}} + e_j^\top  (\widehat{\mathsf{Z}}^\top \widehat{\mathsf{Z}} )^{-1} \widehat{\mathsf{Z}}^\top \bm{\xi}.
    \end{align*}
\begin{lemma}\label{lem:apriorest}
    Assume that conditions \eqref{cond:O1}--\eqref{cond:O3} hold. Then there exists some constant $C = C(\mathsf{M},\tau,\mathsf{c}_{\mathsf{L}}) > 0$ such that for $n \geq C$,
    \begin{align*}
        \Prob \Big( \max_{i \in [n], j \in [2p]} \abs{e_j^\top  (\widehat{\mathsf{Z}}^\top \widehat{\mathsf{Z}} )^{-1} \widehat{\mathsf{Z}}^\top e_i} \geq C \sqrt{\frac{\log n}{n}}  \Big) \leq Cn^{-10}.
    \end{align*}
\end{lemma}
\begin{proof}
Using the Sherman-Morrison formula, for any $i \in [n]$ and $j \in [2p]$, we have that 
\begin{align*}
    \abs{e_j^\top  (\widehat{\mathsf{Z}}^\top \widehat{\mathsf{Z}} )^{-1} \widehat{\mathsf{Z}}^\top e_i} &\leq \abs{e_j^\top  (\widehat{\mathsf{Z}}_{-i,\cdot}^\top \widehat{\mathsf{Z}}_{-i,\cdot} )^{-1} \widehat{\mathsf{Z}}_{i,\cdot}}  + \biggabs{ \frac{  e_j^\top  (\widehat{\mathsf{Z}}_{-i,\cdot}^\top \widehat{\mathsf{Z}}_{-i,\cdot} )^{-1}\widehat{\mathsf{Z}}_{i,\cdot}\widehat{\mathsf{Z}}_{i,\cdot}^\top (\widehat{\mathsf{Z}}_{-i,\cdot}^\top \widehat{\mathsf{Z}}_{-i,\cdot} )^{-1}\widehat{\mathsf{Z}}_{i,\cdot}  }{1 + \widehat{\mathsf{Z}}_{i,\cdot}^\top(\widehat{\mathsf{Z}}_{-i,\cdot}^\top \widehat{\mathsf{Z}}_{-i,\cdot} )^{-1}\widehat{\mathsf{Z}}_{i,\cdot} }} \\
    &\leq 2 \abs{e_j^\top  (\widehat{\mathsf{Z}}_{-i,\cdot}^\top \widehat{\mathsf{Z}}_{-i,\cdot} )^{-1} \widehat{\mathsf{Z}}_{i,\cdot}}.
\end{align*}
The conclusion then follows by \eqref{ineq:liner_conc1} and an union bound argument, which completes the proof of Lemma \ref{lem:apriorest}.
\end{proof}

    We will focus on the case when $\mathsf{W} = \widehat{W}$. The constant $C$ below may depend on $\mathsf{M},\tau,\mathsf{c}_{\mathsf{L}}$, which may change from line to line.
    Let us define $\bar{j} \equiv j + p$ for any $j \in [p]$. Denote by $\widehat{\mathsf{G}}^{(j)} \sim \mathcal{N}(0,\sigma_\xi^2\cdot \widehat{\Sigma}_{n}^{(j)})$ with 
    \begin{align*}
        \widehat{\Sigma}_{n}^{(j)} \equiv 
        \begin{bmatrix}
        [(\widehat{\mathsf{Z}}^\top \widehat{\mathsf{Z}})^{-1} ]_{jj} & [(\widehat{\mathsf{Z}}^\top \widehat{\mathsf{Z}})^{-1} ]_{j\bar{j}}\\
        [(\widehat{\mathsf{Z}}^\top \widehat{\mathsf{Z}})^{-1}  ]_{j\bar{j}} & [(\widehat{\mathsf{Z}}^\top \widehat{\mathsf{Z}})^{-1} ]_{\bar{j}\bar{j}}
        \end{bmatrix}.
    \end{align*}
Applying Lemmas \ref{lem:quad_conc}--\ref{lem:expectaion_asymp} with $\eta$ therein chosen to be sufficiently small, i.e., $\eta = n^{-5}$, we can show that for large enough $n$,
\begin{align*}
    \Prob \Bigg( \mathcal{E}_j \equiv \Bigg\{ \biggabs{\iprod{\mathfrak{u} }{(\widehat{\mathsf{Z}}^\top \widehat{\mathsf{Z}})^{-1}   \mathfrak{v}} - \frac{\mathfrak{u} \Sigma^{-1} \mathfrak{v} }{1-2p/n } \leq C\frac{\log^2 n}{\sqrt{n}}  }, \;  \mathfrak{u}, \mathfrak{v} \in \{e_j, e_{\bar{j}} \} \Bigg\} \Bigg) \geq 1 - Cn^{-10}
\end{align*}
and on event $\mathcal{E}_j$, $\pnorm{ \widehat{\Sigma}_{n}^{(j)} -  {\Sigma}_{n}^{(j)} }{\max} \leq C \log^2 n/\sqrt{n}$. 
Since $\widehat{\Sigma}_{n}^{(j)}$ is a principal submatrix of $(\widehat{\mathsf{Z}}^\top \widehat{\mathsf{Z}})^{-1}$, an application of the Cauchy interlacing theorem gives that 
\begin{align*}
    &\lambda_{\max} ( \widehat{\Sigma}_{n}^{(j)}) \leq \lambda_{\max} ((\widehat{\mathsf{Z}}^\top \widehat{\mathsf{Z}})^{-1} ) = \frac{1}{\lambda_{\min} (\widehat{\mathsf{Z}}^\top \widehat{\mathsf{Z}}) }, \\
    &\lambda_{\min} ( \widehat{\Sigma}_{n}^{(j)}) \geq \lambda_{\min} ((\widehat{\mathsf{Z}}^\top \widehat{\mathsf{Z}})^{-1} ) = \frac{1}{\lambda_{\max} (\widehat{\mathsf{Z}}^\top \widehat{\mathsf{Z}}) }.
\end{align*}
Thus, it follows from Lemma \ref{lem:rigidity} that 
\begin{align}\label{ineq:empirical_cov_spec}
    \Prob\bigg(\mathscr{E}_j \equiv \bigg\{ C^{-1} \leq  \lambda_{\min}( \widehat{\Sigma}_{n}^{(j)}) \leq  \lambda_{\max}( \widehat{\Sigma}_{n}^{(j)})  \leq C \bigg\} \bigg) \geq 1-Cn^{-10}.
\end{align}
Moreover, by Lemma \ref{lem:apriorest}, we have that
\begin{align*}
    \Prob\bigg(\mathfrak{E}_j \equiv \bigg\{  \max_{i \in [n], \mathfrak{j}\in \{j,\bar{j} \}} \abs{e_\mathfrak{j}^\top  (\widehat{\mathsf{Z}}^\top \widehat{\mathsf{Z}} )^{-1} \widehat{\mathsf{Z}}^\top e_i} \geq C \sqrt{\frac{\log n}{n}} \bigg\}  \bigg) \geq 1 -  Cn^{-10}.
\end{align*}

In the sequel, we will work on event $\mathcal{E}_j \cap \mathscr{E}_j \cap \mathfrak{E}_j$. Recall that for $j \in \mathcal{H}_0$, $ \beta_{\ast,j} = \beta_{\ast,\bar{j}} = 0$. By invoking Theorem \ref{thm:moderate_deviation} and setting $\mathsf{X}_i$ therein be $(\sqrt{n} e_j^\top (\widehat{\mathsf{Z}}^\top \widehat{\mathsf{Z}} )^{-1} \widehat{\mathsf{Z}}^\top e_i  \bm{\xi}_i,  \sqrt{n}e_{\bar{j}}^\top (\widehat{\mathsf{Z}}^\top \widehat{\mathsf{Z}} )^{-1} \widehat{\mathsf{Z}}^\top e_i \bm{\xi}_i)^\top$ for any $i \in [n]$, we have that 
\begin{align}\label{ineq:moder_1}
    &\Prob^{\widehat{\mathsf{Z}}} \big( \abs{\widehat{\beta}^{\mathsf{LS}}_j} -  \abs{\widehat{\beta}^{\mathsf{LS}}_{\bar{j}}} \geq t \big) =  \bigg( 1+ \mathcal{O}\bigg(\frac{(1 + t^3)\log^2 n }{\sqrt{n}} \bigg)\bigg) \cdot  \Prob^{\widehat{\mathsf{Z}}} \big( \abs{\widehat{\mathsf{G}}^{(j)}_1} -  \abs{\widehat{\mathsf{G}}^{(j)}_2} \geq t\big).
\end{align}

We next proceed with the covariance approximation. From the definition of the Gaussian density function, it holds that 
\begin{align*}
    &\Prob^{\widehat{\mathsf{Z}}} \big( \abs{\widehat{\mathsf{G}}^{(j)}_1} -  \abs{\widehat{\mathsf{G}}^{(j)}_2} \geq t \big) - \mathcal{P}_j(t) = \mathcal{H}({\Sigma}_{n}^{(j)}, {\Sigma}_{n}^{(j)}) - \mathcal{H}(\widehat{\Sigma}_{n}^{(j)}, \widehat{\Sigma}_{n}^{(j)}),
\end{align*}
where for $\Sigma_1, \Sigma_2 \in \R^{2 \times 2}$, $\mathcal{H}(\Sigma_1, \Sigma_2)$ is defined as 
\begin{align*}
    \mathcal{H}(\Sigma_1, \Sigma_2) \equiv \frac{1}{2\pi \sqrt{ \det (\Sigma_1) }}\int_{-\infty}^\infty \int_{-(\abs{y}+t )}^{\abs{y}+t} \rho((x, y)^\top;  \Sigma_2^{-1})   \mathrm{d}x \mathrm{d} y
\end{align*}
with 
\begin{align*}
  \rho((x, y)^\top; \Sigma_2^{-1} ) \equiv  \exp \bigg( -\frac{(x,y)\,\Sigma_2^{-1} \, (x,y)^\top}{2\sigma_\xi^2} \bigg). 
\end{align*}
It follows from the triangle inequality that 
\begin{align*}
   &\abs{ \mathcal{H}(\widehat{\Sigma}_{n}^{(j)}, \widehat{\Sigma}_{n}^{(j)}) - \mathcal{H}({\Sigma}_{n}^{(j)}, {\Sigma}_{n}^{(j)}) }\\
   &\leq  \abs{ \mathcal{H}(\widehat{\Sigma}_{n}^{(j)}, \widehat{\Sigma}_{n}^{(j)}) - \mathcal{H}(\widehat{\Sigma}_{n}^{(j)}, {\Sigma}_{n}^{(j)}) } + \abs{ \mathcal{H}(\widehat{\Sigma}_{n}^{(j)}, {\Sigma}_{n}^{(j)}) - \mathcal{H}({\Sigma}_{n}^{(j)}, {\Sigma}_{n}^{(j)}) }.
\end{align*}
We will bound these two terms above separately.

\medskip
\emph{Term $\abs{ \mathcal{H}(\widehat{\Sigma}_{n}^{(j)}, \widehat{\Sigma}_{n}^{(j)}) - \mathcal{H}(\widehat{\Sigma}_{n}^{(j)}, {\Sigma}_{n}^{(j)}) }$}:
Let $\bm{x} \equiv (x,y)^{\top}$ for simplicity. Using the fact that
\begin{align*}
    &\rho(\bm{x}; {\Sigma}_{n}^{(j);-1}) -   \rho(\bm{x}; \widehat{\Sigma}_{n}^{(j);-1}) \\
    &= \rho(\bm{x}; {\Sigma}_{n}^{(j);-1}) \cdot (1 - \rho(\bm{x}; \widehat{\Sigma}_{n}^{(j);-1}- {\Sigma}_{n}^{(j);-1} ) ) \cdot  \bm{1}\{\bm{x}^\top \, ({\Sigma}_{n}^{(j);-1} -  \widehat{\Sigma}_{n}^{(j);-1})  \, \bm{x}   < 0 \} \\
    &\quad + \rho(\bm{x}; \widehat{\Sigma}_{n}^{(j);-1}) \cdot (1 - \rho(\bm{x}; {\Sigma}_{n}^{(j);-1}- \widehat{\Sigma}_{n}^{(j);-1} ) ) \cdot \bm{1}\{\bm{x}^\top \, ({\Sigma}_{n}^{(j);-1} -  \widehat{\Sigma}_{n}^{(j);-1})  \, \bm{x}   \geq 0 \} \\
    & \leq \frac{\pnorm{\bm{x} }{}^2}{2\sigma_\xi^2 } \cdot  \pnorm{  \widehat{\Sigma}_{n}^{(j);-1} - {\Sigma}_{n}^{(j);-1}}{\op} \cdot  \Big(\rho(\bm{x}; {\Sigma}_{n}^{(j);-1}) \vee \rho(\bm{x}; \widehat{\Sigma}_{n}^{(j);-1}) \Big),
\end{align*}
we can derive that 
\begin{align*}
     &\abs{ \mathcal{H}(\widehat{\Sigma}_{n}^{(j)}, \widehat{\Sigma}_{n}^{(j)}) - \mathcal{H}(\widehat{\Sigma}_{n}^{(j)}, {\Sigma}_{n}^{(j)}) } \leq \frac{1}{2\pi \sqrt{ \det (\widehat{\Sigma}_{n}^{(j)}) }}\int_{-\infty}^\infty \int_{-(\abs{y}+t )}^{\abs{y}+t} \bigg\{ \frac{(x^2 + y^2 )}{2\sigma_\xi^2} \\
     &\qquad\quad \times \pnorm{  \widehat{\Sigma}_{n}^{(j);-1} - {\Sigma}_{n}^{(j);-1}}{\op} \cdot  \Big(\rho((x,y)^\top ; {\Sigma}_{n}^{(j);-1}) + \rho((x,y)^\top; \widehat{\Sigma}_{n}^{(j);-1}) \Big) \bigg\} \mathrm{d}x \mathrm{d} y\\
     &\leq \frac{1}{2\pi \sqrt{ \det (\widehat{\Sigma}_{n}^{(j)}) }}\int_{-\infty}^\infty \int_{-(\abs{y}+t )}^{\abs{y}+t} \bigg\{ \frac{(x^2 + y^2 )}{2\sigma_\xi^2} \big(\bm{1}\{ (x^2+y^2) \leq C\log n \}  + \bm{1}\{ (x^2+y^2) >  C\log n \}\big ) \\
     &\qquad\quad \times \pnorm{  \widehat{\Sigma}_{n}^{(j);-1} - {\Sigma}_{n}^{(j);-1}}{\op} \cdot  \Big(\rho((x,y)^\top ; {\Sigma}_{n}^{(j);-1}) + \rho((x,y)^\top; \widehat{\Sigma}_{n}^{(j);-1}) \Big) \bigg\} \mathrm{d}x \mathrm{d} y.
\end{align*}
Here, $C > 0$ is a constant chosen to be sufficiently large. 

For the term related to $\bm{1}\{ (x^2+y^2) \leq C\log n \}$, which we denote as $I_1$, can be bounded as 
\begin{align*}
    I_1 \leq C\frac{\log^3 n}{\sqrt{n}} \cdot \Big(  \mathcal{H}(\widehat{\Sigma}_{n}^{(j)}, \widehat{\Sigma}_{n}^{(j)}) + \mathcal{H}(\widehat{\Sigma}_{n}^{(j)}, {\Sigma}_{n}^{(j)})\Big).
\end{align*}
 For the term related to $\bm{1}\{ (x^2+y^2) > C\log n \}$, which we denote as $I_2$, can be bounded using the Gaussian tail bound. By increasing $C$ if necessary, we can obtain that for sufficiently large $n$, $I_2 \leq n^{-10}$. Combining the above estimates yields that 
 \begin{align*}
     \abs{ \mathcal{H}(\widehat{\Sigma}_{n}^{(j)}, \widehat{\Sigma}_{n}^{(j)}) - \mathcal{H}(\widehat{\Sigma}_{n}^{(j)}, {\Sigma}_{n}^{(j)}) } \leq C\frac{\log^3 n}{\sqrt{n}} \cdot \Big(  \mathcal{H}(\widehat{\Sigma}_{n}^{(j)}, \widehat{\Sigma}_{n}^{(j)}) + \mathcal{H}(\widehat{\Sigma}_{n}^{(j)}, {\Sigma}_{n}^{(j)})\Big) + n^{-10}.
 \end{align*}

\medskip
\emph{Term $ \abs{ \mathcal{H}(\widehat{\Sigma}_{n}^{(j)}, {\Sigma}_{n}^{(j)}) - \mathcal{H}({\Sigma}_{n}^{(j)}, {\Sigma}_{n}^{(j)}) }$}: Notice that on event $\mathcal{E}_j \cap \mathscr{E}_j$, $\det ({\Sigma}_{n}^{(j)}) \wedge \det (\widehat{\Sigma}_{n}^{(j)}) \geq C^{-1}$ and $\abs{\det ({\Sigma}_{n}^{(j)}) - \det (\widehat{\Sigma}_{n}^{(j)})} \leq C \log^2 n/\sqrt{n}$. Thus we can deduce that 
\begin{align*}
  &\abs{ \mathcal{H}(\widehat{\Sigma}_{n}^{(j)}, {\Sigma}_{n}^{(j)}) - \mathcal{H}({\Sigma}_{n}^{(j)}, {\Sigma}_{n}^{(j)}) } \\
  &= \Bigg(\frac{1}{2\pi \sqrt{ \det (\widehat{\Sigma}_{n}^{(j)}) }} - \frac{1}{2\pi \sqrt{ \det ({\Sigma}_{n}^{(j)}) }}\Bigg)\cdot \int_{-\infty}^\infty \int_{-(\abs{y}+t )}^{\abs{y}+t} \rho((x,y)^\top ; {\Sigma}_{n}^{(j);-1})\mathrm{d}x \mathrm{d} y \\
  &\leq C\frac{\log^2 n}{\sqrt{n}} \cdot \mathcal{H}({\Sigma}_{n}^{(j)}, {\Sigma}_{n}^{(j)}). 
\end{align*}
Combining the estimates for $\abs{ \mathcal{H}(\widehat{\Sigma}_{n}^{(j)}, \widehat{\Sigma}_{n}^{(j)}) - \mathcal{H}(\widehat{\Sigma}_{n}^{(j)}, {\Sigma}_{n}^{(j)}) }$ and  $ \abs{ \mathcal{H}(\widehat{\Sigma}_{n}^{(j)}, {\Sigma}_{n}^{(j)}) - \mathcal{H}({\Sigma}_{n}^{(j)}, {\Sigma}_{n}^{(j)}) }$ above, we can obtain that 
\begin{align*}
   & \Prob^{\widehat{\mathsf{Z}}} \big( \abs{\widehat{\mathsf{G}}^{(j)}_1} -  \abs{\widehat{\mathsf{G}}^{(j)}_2} \geq t \big)  - \mathcal{P}_j(t) \\
    &\leq C\frac{\log^3 n}{\sqrt{n}} \cdot \Big( \Prob^{\widehat{\mathsf{Z}}} \big( \abs{\widehat{\mathsf{G}}^{(j)}_1} -  \abs{\widehat{\mathsf{G}}^{(j)}_2} \geq t \big) + \mathcal{P}_j(t) \Big) + n^{-10},
\end{align*}
which further entails that
\begin{align*}
	 \Prob^{\widehat{\mathsf{Z}}} \big( \abs{\widehat{\mathsf{G}}^{(j)}_1} -  \abs{\widehat{\mathsf{G}}^{(j)}_2} \geq t \big)  =\mathcal{P}_j(t) \cdot \bigg(1 + \mathcal{O}\bigg( \frac{\log^3 n}{\sqrt{n}} \bigg) \bigg) + n^{-10}.
\end{align*}

Combining the above estimate with (\ref{ineq:moder_1}), we can show that 
\begin{align*}
	\Prob^{\widehat{\mathsf{Z}}} \big( \abs{\widehat{\beta}^{\mathsf{LS}}_j} -  \abs{\widehat{\beta}^{\mathsf{LS}}_{\bar{j}}} \geq t \big) = \mathcal{P}_j(t)\cdot \bigg(1 + \mathcal{O}\bigg(\frac{(1 + t^3)\log^3 n }{\sqrt{n}} \bigg) \bigg) + n^{-10}.
\end{align*}
The conclusion follows by a law of total probability argument, which completes the proof of Theorem \ref{thm:moder_dis_ols}. 

\subsubsection{Proof of Lemma \ref{lem:indicator_app_ols}}
\begin{lemma}\label{lem:var_bound_ols}
Assume that conditions \eqref{cond:O1}--\eqref{cond:O3} hold. Then there exists some constant $C = C(\mathsf{M}, \tau, \mathsf{c}_{\mathsf{L}}) > 0$  such that for $n \geq C$ and  $t \in (0,n^{1/7}/C)$,
         \begin{align*}
         & \var \Big( \sum_{ j \in \mathcal{H}_0}\bm{1} \{ \widehat{\bm{W}}_j \geq t \} \Big)\\
         & \leq C\Bigg(\sum_{j\in \mathcal{H}_0}  (\abs{N(j)} + 1)  \Prob \big( \widehat{\bm{W}}_j \geq t \big) +  \frac{(1+t^3)\log^3 n }{\sqrt{n}}  \bigg(\sum_{j\in \mathcal{H}_0} \Prob \big( \widehat{\bm{W}}_j \geq t \big)    \bigg)^2+ n^{-10}\Bigg).
     \end{align*}
\end{lemma}

\noindent\textit{Proof of Lemma \ref{lem:var_bound_ols}}. 
Observe that 
    \begin{align*}
         &\var \Big( \sum_{ j \in \mathcal{H}_0}\bm{1} \{ \widehat{\bm{W}}_j \geq t \} \Big) = \sum_{j,\ell \in \mathcal{H}_0} \Big(  \E \bm{1} \{ \widehat{\bm{W}}_j \geq t \} \bm{1} \{ \widehat{\bm{W}}_\ell \geq t \} -   \E \bm{1} \{ \widehat{\bm{W}}_j \geq t \} \E\bm{1} \{ \widehat{\bm{W}}_\ell \geq t \}   \Big) \\
         &= \sum_{j,\ell \in \mathcal{H}_0, j \neq \ell} \Big( \Prob \big(\widehat{\bm{W}}_j \geq t, \widehat{\bm{W}}_\ell \geq t  \big) - \Prob \big(\widehat{\bm{W}}_j \geq t \big) \Prob \big(\widehat{\bm{W}}_\ell \geq t  \big)  \Big)\\
         &\quad + \sum_{j \in \mathcal{H}_0} \Big( \Prob \big(\widehat{\bm{W}}_j \geq t \big) - \Prob^2 \big(\widehat{\bm{W}}_j \geq t \big)\Big) \\
         &\leq \sum_{j,\ell \in \mathcal{H}_0, j \neq \ell} \Big( \Prob \big(\widehat{\bm{W}}_j \geq t, \widehat{\bm{W}}_\ell \geq t  \big) - \Prob \big(\widehat{\bm{W}}_j \geq t \big) \Prob \big(\widehat{\bm{W}}_\ell \geq t  \big)  \Big) + \sum_{j \in \mathcal{H}_0}\Prob \big(\widehat{\bm{W}}_j \geq t \big)\\
         & \equiv I_1 + \sum_{j \in \mathcal{H}_0}\Prob \big(\widehat{\bm{W}}_j \geq t \big).
    \end{align*}
    A slightly generalized version of Theorem \ref{thm:moder_dis_ols} yields that for $\ell \notin N(j) \cup \{ j\} $ and any $t \in (0,n^{1/7}/C)$,
    \begin{align*}
       &\Prob \big(\widehat{\bm{W}}_j \geq t, \widehat{\bm{W}}_\ell \geq t  \big) \\
       &= \Prob \big( \abs{\mathsf{G}^{(j\ell)}_1} - \abs{\mathsf{G}^{(j\ell)}_2} \geq t, \abs{\mathsf{G}^{(j\ell)}_3} - \abs{\mathsf{G}^{(j\ell)}_4} \geq t \big) \cdot \bigg(1 + \mathcal{O}\bigg( \frac{(1+t^3)\log^3 n}{\sqrt{n}}\bigg) \bigg) + n^{-10}\\
       &= \Prob \big( \abs{\mathsf{G}^{(j\ell)}_1} - \abs{\mathsf{G}^{(j\ell)}_2} \geq t\big) \Prob\big( \abs{\mathsf{G}^{(j\ell)}_3} - \abs{\mathsf{G}^{(j\ell)}_4} \geq t \big) \cdot \bigg(1 + \mathcal{O}\bigg( \frac{(1+t^3)\log^3 n}{\sqrt{n}}\bigg) \bigg) + n^{-10}\\
       &=\Prob \big(\widehat{\bm{W}}_j \geq t\big) \Prob\big( \widehat{\bm{W}}_\ell \geq t  \big) \cdot \bigg(1 + \mathcal{O}\bigg( \frac{(1+t^3)\log^3 n}{\sqrt{n}}\bigg) \bigg) + n^{-10},
    \end{align*}
    where $\mathsf{G}^{(j\ell)} \sim \mathcal{N}(0, \Sigma_n^{(j\ell)}) \in \R^4 $ with 
    \begin{align*}
        \Sigma_n^{(j\ell)} \equiv \frac{1}{1-2p/n} \big[ \Sigma^{-1}\big]_{\{j, \bar{j}, \ell ,\bar{\ell} \} , \{j, \bar{j}, \ell ,\bar{\ell} \}}.
    \end{align*}
    
    We can then bound $I_1$ as 
    \begin{align*}
    	 I_1 &\leq \bigg(\sum_{j\in \mathcal{H}_0} \sum_{\ell \in N(j) \cup \{j\} } + \sum_{j\in \mathcal{H}_0} \sum_{\substack{ \ell \in \mathcal{H}_0 \cap N(j)^c \\\ell \neq j} }\bigg) \Big(\Prob \big(\widehat{\bm{W}}_j \geq t, \widehat{\bm{W}}_\ell \geq t  \big) - \Prob \big(\widehat{\bm{W}}_j \geq t \big) \Prob \big(\widehat{\bm{W}}_\ell \geq t  \big)\Big)\\
    	 &\leq \sum_{j \in \mathcal{H}_0} \bigg\{ (\abs{N(j) } + 1) \cdot \Prob \big(\widehat{\bm{W}}_j \geq t\big) + \sum_{\substack{ \ell \in \mathcal{H}_0 \cap N(j)^c \\\ell \neq j} } \Big(\Prob \big(\widehat{\bm{W}}_j \geq t, \widehat{\bm{W}}_\ell \geq t  \big) - \Prob \big(\widehat{\bm{W}}_j \geq t \big) \Prob \big(\widehat{\bm{W}}_\ell \geq t  \big)\Big) \bigg\}\\
    	 &\leq \sum_{j\in \mathcal{H}_0}  (\abs{N(j) } + 1) \cdot  \Prob \big( \widehat{\bm{W}}_j \geq t \big) +  \bigg(\sum_{j\in \mathcal{H}_0} \Prob \big( \widehat{\bm{W}}_j \geq t \big)    \bigg)^2\cdot \mathcal{O}\bigg( \frac{(1+t^3)\log^3 n}{\sqrt{n}}\bigg)+ n^{-10}.
    \end{align*}
    This completes the proof of Lemma \ref{lem:var_bound_ols}.

With the aid of Lemma \ref{lem:var_bound_ols}, the proof of Lemma \ref{lem:indicator_app_ols} follows a similar structure to that of Lemma \ref{lem:indicator_app}. Below we will outline some key differences.

    First, we will choose $h_n = (p\alpha_n)^{1-\eta}m_n^{\eta} $ for some $\eta \in (0,1)$, ensuring that $h_n/(p\alpha_n) = \mathfrak{o}(1)$ and $h_n \to \infty$. Such choice guarantees that both (\ref{eq:P_continuity}) and (\ref{eq:dis_approx}) remain valid. Hence, it suffices to bound $ \Prob  ( D_n \geq \epsilon )$. 

    It follows from Lemma \ref{lem:var_bound_ols} that 
\begin{align*}
    \Prob  ( D_n \geq \epsilon ) &\leq \sum_{i = 0}^{l_n} \Prob \bigg(\biggabs{ \frac { \sum_{j \in \mathcal{H}_0 } \big(\bm{1} (\widehat{\bm{W}}_j \geq t_i) - \Prob ( \widehat{\bm{W}}_j \geq t_i) \big) } { \abs{\mathcal{H}_0} \cdot  \widehat{G}_+(t_i) } } \geq \epsilon \bigg) \\
    &\leq \sum_{i = 0}^{l_n} \frac{\var \big(\sum_{j \in \mathcal{H}_0 } \big(\bm{1} (\widehat{\bm{W}}_j \geq t_i) \big)   }{ \abs{\mathcal{H}_0}^2 \cdot  \widehat{G}_+^2(t_i) \cdot \epsilon^2}\\
    &\leq  \frac{m_n + 1}{\abs{\mathcal{H}_0}\epsilon^2 } \sum_{i = 0}^{l_n}\frac{1}{ \widehat{G}_+(t_i) }  + C\frac{l_n (1 + t_{0}^3) \log^3 (n)}{\epsilon^2 \sqrt{n}} + \frac{n^{-10}}{\abs{\mathcal{H}_0}^2 \cdot \epsilon^2}\cdot \sum_{i = 0}^{l_n}\frac{1}{ \widehat{G}^2_+(t_i) }. 
\end{align*}
The latter two terms can be bounded the same as those in the proof of Lemma \ref{lem:indicator_app}. For the first term, it holds that 
\begin{align*}
    \frac{m_n + 1}{\abs{\mathcal{H}_0}\epsilon^2 } \sum_{i = 0}^{l_n}\frac{1}{ \widehat{G}_+(t_i) } &= \frac{m_n + 1}{\abs{\mathcal{H}_0}\epsilon^2 } \sum_{i = 0}^{l_n}\frac{1}{ \mathcal{P}(t_i) } +\frac{m_n + 1}{\abs{\mathcal{H}_0}\epsilon^2 }   \sum_{i = 0}^{l_n}\frac{\mathcal{P}(t_i) - \widehat{G}_+(t_i)}{ \widehat{G}_+(t_i)\mathcal{P}(t_i) } \\
    &=\frac{m_n + 1}{\abs{\mathcal{H}_0}\epsilon^2 } \sum_{i = 0}^{l_n} \frac{1 + \mathfrak{o}(1)}{z_i } \leq C\frac{m_n + 1}{h_n\epsilon^2} \to 0,
\end{align*}
where the last inequality above follows from \cite[equation (A.107)]{fan2023ark}. This completes the proof of Lemma \ref{lem:indicator_app_ols}. 

\subsubsection{Proof of Lemma \ref{lem:local_T_ols}}
    The proof follows the same as that of Lemma \ref{lem:local_T}, upon applying Corollary \ref{cor:ratio_G_ols} and Lemma \ref{lem:indicator_app_ols} by setting $\alpha_n$ therein as $qa_n/2p$ and noting that $p\alpha_n/m_n = qa_n/2m_n \to \infty$ as $n \to \infty$. This completes the proof of Lemma \ref{lem:local_T_ols}. 

\subsection{Proofs for Section \ref{sec:proof_thm:FDR_dl}}\label{sec:proof_sec_sec_thm:FDR_dl}

\subsubsection{Proof of Theorem \ref{thm:moder_dis_dl}}
The proof of Theorem \ref{thm:moder_dis_dl} follows the same structure as that of Theorem \ref{thm:moder_dis_ols}. We outline the key differences below.

We need the following aprior estimates to prove Theorem \ref{thm:moder_dis_dl}. Recall that for $j \in [2p]$,
\begin{align*}
    \widehat{\bm{z}}_j = \widehat{\bm{Z}}_{\cdot,j} - \widehat{\bm{Z}}_{\cdot,-j} \widehat{\bm{\gamma}}_j,
\end{align*}
and $\widehat{\bm{\gamma}}_j \equiv \arg\min_{\bm{b} \in \mathbb{R}^{2p}}\{ (2n)^{-1}\pnorm{ \widehat{\bm{Z}}_{\cdot,j} - \widehat{\bm{Z}}_{\cdot,-j} \bm{b} }{}^2  + \lambda_j \pnorm{\bm{b}_j}{1} \}$.
\begin{lemma}\label{lem:aprior_est_dl}
    Assume that conditions \eqref{cond:O3} and \eqref{cond:L1} hold. Then if $b_n \log p /n = \mathfrak{o}(1)$, we have that there exists some constant $C = C(\mathsf{M},\mathsf{C}_1,\mathsf{C}_2,\mathsf{C}_3) > 0$ such that for $n \wedge p \geq C$ and  for any $j \in [2p]$,
    \begin{align*}
       \Prob \bigg(\max_{i \in [n]}  \biggabs{\frac{ \sqrt{n}\widehat{\bm{z}}_{j,i} }{ \langle \widehat{\bm{z}}_j, \widehat{\bm{Z}}_{\cdot,j} \rangle}} \geq  \frac{Cb_n^{1/2}\log (np)}{\sqrt{n}}  \bigg) \leq Cp^{-10}
    \end{align*}
\end{lemma}
\begin{proof}
The constant $C$ below may depend on $\mathsf{M},\mathsf{C}_1,\mathsf{C}_2,\mathsf{C}_3$, which may change from line to line. We first rewrite that 
    \begin{align}\label{ineq:est_1331}
        \biggabs{\frac{  \sqrt{n}\widehat{\bm{z}}_{j,i} }{ \langle \widehat{\bm{z}}_j, \widehat{\bm{Z}}_{\cdot,j} \rangle}} =  \frac{ \abs{\widehat{\bm{z}}_{j,i}} }{ \pnorm{ \widehat{\bm{z}}_j  }{}  }\cdot  \frac{ \sqrt{n}\pnorm{ \widehat{\bm{z}}_j  }{} }{ \abs{ \langle \widehat{\bm{z}}_j, \widehat{\bm{Z}}_{\cdot,j} \rangle } }.
    \end{align}
    It follows from \cite[equation (A.18)-(A.19)]{fan2023ark} that 
    \begin{align}\label{ineq:est_1332}
        \frac{ 1 }{ \pnorm{ \widehat{\bm{z}}_j  }{}  }\cdot  \frac{ \sqrt{n}\pnorm{ \widehat{\bm{z}}_j  }{} }{ \abs{ \langle \widehat{\bm{z}}_j, \widehat{\bm{Z}}_{\cdot,j} \rangle } } \leq  \frac{C}{\sqrt{n}}
    \end{align}
    holds with probability at least $1 - Cp^{-10}$. For any $i \in [n]$, using \cite[equation (A.16)]{fan2023ark},  we can find some deterministic quantity $\gamma_j$ satisfying $\max_{j \in [2p]}\pnorm{\gamma_j}{} \leq C$ and $\max_{j \in [2p]}\pnorm{\gamma_j}{0} \leq b_n$ such that $\max_{j \in [2p]} \pnorm{\widehat{\bm{\gamma}}_j - \gamma_j }{} \leq C\sqrt{b_n \log p /n}$ holds with probability at least $1 - Cp^{-10}$. Moreover, by the sub-Gaussinity of $\widehat{\bm{Z}}_{i,\cdot}$, we have $\Prob( \max_{i \in [n], j \in [2p]}\abs{\widehat{\bm{Z}}_{i,j}} \leq C\sqrt{\log (np)} ) \geq 1 - Cp^{-10}$. Combining the above high-probability estimates, we have with probability at least $1 - Cp^{-10}$ that 
    \begin{align}\label{ineq:est_1333}
       \max_{i\in [n]}\abs{ \widehat{\bm{z}}_{j,i}} &\leq  \max_{i\in [n]} \{ \abs{ \widehat{\bm{Z}}_{i,j}} + \abs{ \widehat{\bm{Z}}_{i,-j} \widehat{\bm{\gamma}}_j} \} \leq \max_{i\in [n]}\{\abs{ \widehat{\bm{Z}}_{i,j}} +\abs{ \widehat{\bm{Z}}_{i,-j} {\bm{\gamma}}_j} +  \abs{ \widehat{\bm{Z}}_{i,-j} (\widehat{\bm{\gamma}}_j - \gamma_j)} \}\nonumber \\
       &\leq \max_{i\in [n]}\{ \abs{ \widehat{\bm{Z}}_{i,j}} +\abs{ \widehat{\bm{Z}}_{i,-j} {\bm{\gamma}}_j} + \pnorm{ \widehat{\bm{Z}}_{i,-j}  }{} \cdot \pnorm{\widehat{\bm{\gamma}}_j - \gamma_j }{} \} \leq Cb_n^{1/2}\log (np).
    \end{align}
   The conclusion follows by combining \eqref{ineq:est_1331}-\eqref{ineq:est_1333}, which completes the proof of Lemma \ref{lem:aprior_est_dl}.
\end{proof}

	We will focus on the proof for $\mathsf{W} = \widehat{\bm{W}}$ below. The constant $C$ below may depend on $\mathsf{M},\mathsf{C}_1,\mathsf{C}_2,\mathsf{C}_3$, which may change from line to line. 
	Recall that $\beta_\ast^{\mathsf{au}} = (\beta_\ast^\top, 0_{p}^\top )^\top \in \R^{2p}$ and the linear regression model for the data can be written in the following form 
	\begin{align*}
		\bm{Y} = \widehat{\bm{Z}} \beta_\ast^{\mathsf{au}} + \bm{\xi}.
	\end{align*}
	Since for any $j \in \mathcal{H}_0$, $\beta_{\ast,j}^{\mathsf{au}}=\beta_{\ast,j+p}^{\mathsf{au}}  = 0$, we can then rewrite 
	\begin{align*}
		\widehat{\beta}^{\mathsf{dL}}_j &= \widehat{\beta}^{(0)}_j + \frac{\langle \widehat{\bm{z}}_j, \widehat{\bm{Z}}\beta_{\ast}^{\mathsf{au}} - \widehat{\bm{Z}} \widehat{\beta}^{(0)} \rangle}{\langle \widehat{\bm{z}}_j, \widehat{\bm{Z}}_{\cdot,j} \rangle} + \frac{\langle \widehat{\bm{z}}_j, \bm{\xi} \rangle}{\langle \widehat{\bm{z}}_j, \widehat{\bm{Z}}_{\cdot,j} \rangle}\\
		&=\widehat{\beta}^{(0)}_j + \beta_{\ast,j}^{\mathsf{au}} - \widehat{\beta}^{(0)}_j + \frac{1}{\langle \widehat{\bm{z}}_j, \widehat{\bm{Z}}_{\cdot,j} \rangle}\sum_{k \in [2p]\setminus \{ j\} }\iprod{ \widehat{\bm{z}}_j}{\widehat{\bm{Z}}_{\cdot,k} }\cdot (\beta_{\ast,k}^{\mathsf{au}} -  \widehat{\beta}^{(0)}_k )+ \frac{\langle \widehat{\bm{z}}_j, \bm{\xi} \rangle}{\langle \widehat{\bm{z}}_j, \widehat{\bm{Z}}_{\cdot,j} \rangle}\\
		&= \frac{\langle \widehat{\bm{z}}_j, \bm{\xi} \rangle}{\langle \widehat{\bm{z}}_j, \widehat{\bm{Z}}_{\cdot,j} \rangle} +  \frac{1}{\langle \widehat{\bm{z}}_j, \widehat{\bm{Z}}_{\cdot,j} \rangle}\sum_{k \in [2p]\setminus \{ j\} }\iprod{ \widehat{\bm{z}}_j}{\widehat{\bm{Z}}_{\cdot,k} }\cdot (\beta_{\ast,k}^{\mathsf{au}} -  \widehat{\beta}^{(0)}_k ) \\
		&\equiv\frac{\langle \widehat{\bm{z}}_j, \bm{\xi} \rangle}{\langle \widehat{\bm{z}}_j, \widehat{\bm{Z}}_{\cdot,j} \rangle} +  \mathsf{Err}_j. 
	\end{align*}
	Similarly, with $\bar{j} \equiv j +p$,  for $j \in \mathcal{H}_0$ we have 
	\begin{align*}
		\widehat{\beta}^{\mathsf{dL}}_{\bar{j}} =  \frac{\langle \widehat{\bm{z}}_{\bar{j}}, \bm{\xi} \rangle}{\langle \widehat{\bm{z}}_{\bar{j}}, \widehat{\bm{Z}}_{\cdot,\bar{j}} \rangle} +  \mathsf{Err}_{\bar{j}}.
	\end{align*}
    
	With the above representations for $\widehat{\beta}^{\mathsf{dL}}_{j}$ and $\widehat{\beta}^{\mathsf{dL}}_{\bar{j}}$, a simple probability inequality yields that  
	\begin{align*}
		\Prob \big( \widehat{\bm{W}}_j \geq t  \big) &= \Prob \bigg(\sqrt{n} \biggabs{\frac{\langle \widehat{\bm{z}}_j, \bm{\xi} \rangle}{\langle \widehat{\bm{z}}_j, \widehat{\bm{Z}}_{\cdot,j} \rangle} +  \mathsf{Err}_j  } -  \sqrt{n} \biggabs{\frac{\langle \widehat{\bm{z}}_{\bar{j}}, \bm{\xi} \rangle}{\langle \widehat{\bm{z}}_{\bar{j}}, \widehat{\bm{Z}}_{\cdot,\bar{j}} \rangle} +  \mathsf{Err}_{\bar{j}} }  \geq t \bigg) \\
		&\leq \Prob \bigg(\sqrt{n} \biggabs{\frac{\langle \widehat{\bm{z}}_j, \bm{\xi} \rangle}{\langle \widehat{\bm{z}}_j, \widehat{\bm{Z}}_{\cdot,j} \rangle} }   -  \sqrt{n} \biggabs{\frac{\langle \widehat{\bm{z}}_{\bar{j}}, \bm{\xi} \rangle}{\langle \widehat{\bm{z}}_{\bar{j}}, \widehat{\bm{Z}}_{\cdot,\bar{j}} \rangle}}   \geq t - \sqrt{n} \mathsf{Err}_j  - \sqrt{n} \mathsf{Err}_{\bar{j}} \bigg)
	\end{align*}
	and conversely,
	\begin{align*}
		\Prob \big( \widehat{\bm{W}}_j \geq t  \big) 
		\geq \Prob \bigg(\sqrt{n} \biggabs{\frac{\langle \widehat{\bm{z}}_j, \bm{\xi} \rangle}{\langle \widehat{\bm{z}}_j, \widehat{\bm{Z}}_{\cdot,j} \rangle} }   -  \sqrt{n} \biggabs{\frac{\langle \widehat{\bm{z}}_{\bar{j}}, \bm{\xi} \rangle}{\langle \widehat{\bm{z}}_{\bar{j}}, \widehat{\bm{Z}}_{\cdot,\bar{j}} \rangle}}   \geq t + \sqrt{n} \mathsf{Err}_j  + \sqrt{n} \mathsf{Err}_{\bar{j}} \bigg).
	\end{align*}
	It follows from \cite[equation (A.187)]{fan2023ark} that
\begin{align*}
    \Prob \Bigg( \mathcal{E}^{\mathsf{dL}} \equiv \bigg\{  \max_{j \in [2p]} \abs{\mathsf{Err}_j}   \leq \frac{ C b_n^{1/2} s \log p  }{n }  \bigg\}  \Bigg) \geq 1 - Cp^{-8}.
\end{align*}
	Thus, we can deduce that for any $t \in \R$,
	\begin{align*}
		&\Prob \big( \widehat{\bm{W}}_j \geq t  \big) \leq \Prob \big( \widehat{\bm{W}}_j \geq t, \mathcal{E}^{\mathsf{dL}}  \big) + \Prob \big( (\mathcal{E}^{\mathsf{dL}})^c \big) \\
		&\leq \Prob \bigg( \sqrt{n} \biggabs{\frac{\langle \widehat{\bm{z}}_j, \bm{\xi} \rangle}{\langle \widehat{\bm{z}}_j, \widehat{\bm{Z}}_{\cdot,j} \rangle} }   -  \sqrt{n} \biggabs{\frac{\langle \widehat{\bm{z}}_{\bar{j}}, \bm{\xi} \rangle}{\langle \widehat{\bm{z}}_{\bar{j}}, \widehat{\bm{Z}}_{\cdot,\bar{j}} \rangle}} \geq t - \frac{ C b_n^{1/2} s \log p  }{\sqrt{n} } \bigg) + Cp^{-8}
	\end{align*}
	and conversely,
	\begin{align*}
		&\Prob \big( \widehat{\bm{W}}_j \geq t  \big)=  1 - \Prob \big( \widehat{\bm{W}}_j < t  \big)\\
		& \geq 1 -  \Prob \big( \widehat{\bm{W}}_j < t , \mathcal{E}^{\mathsf{dL}} \big) - \Prob \big( (\mathcal{E}^{\mathsf{dL}})^c \big) \\
		&\geq \Prob \bigg( \sqrt{n} \biggabs{\frac{\langle \widehat{\bm{z}}_j, \bm{\xi} \rangle}{\langle \widehat{\bm{z}}_j, \widehat{\bm{Z}}_{\cdot,j} \rangle} }   -  \sqrt{n} \biggabs{\frac{\langle \widehat{\bm{z}}_{\bar{j}}, \bm{\xi} \rangle}{\langle \widehat{\bm{z}}_{\bar{j}}, \widehat{\bm{Z}}_{\cdot,\bar{j}} \rangle}}\geq t + \frac{ C b_n^{1/2} s \log p  }{\sqrt{n} } \bigg) - Cp^{-8}.
	\end{align*}
	
	Next we proceed with the Gaussian approximation. For any $i \in [n]$ and $j \in [2p]$, denote by $\alpha_i^{(j)} \equiv \sqrt{n}\widehat{\bm{z}}_{j,i}/\iprod{\widehat{\bm{z}}_j}{\widehat{\bm{Z}}e_j} $ and  $\bm{\alpha}^{(j)} \equiv (\alpha_1^{(j)},\ldots, \alpha_n^{(j)} )$. We further define $\widehat{\Sigma}^{(j)} \equiv (\bm{\alpha}^{(j)}, \bm{\alpha}^{(\bar{j})} )^\top (\bm{\alpha}^{(j)}, \bm{\alpha}^{(\bar{j})} )$. It follows from the discussion right before \cite[equation (A.221)]{fan2023ark} that 
\begin{align*}
    \Prob \bigg( \mathcal{E}_\Sigma \equiv \bigg\{ \pnorm{\widehat{\Sigma}^{(j)} - {\Sigma}^{(j)} }{\max} \leq C\sqrt{\frac{b_n \log p}{n} } \bigg\} \bigg)  \geq 1 - Cp^{-10}.
\end{align*}
Moreover, by Lemma \ref{lem:aprior_est_dl}, we have that 
\begin{align*}
    \Prob \bigg( \mathfrak{E}_\Sigma \equiv \bigg\{  \max_{i \in [n], \mathfrak{j} \in \{ j, \bar{j} \}} \biggabs{\frac{ \widehat{\bm{z}}_{j,i} }{ \langle \widehat{\bm{z}}_j, \widehat{\bm{Z}}_{\cdot,j} \rangle}} \leq \frac{Cb_n^{1/2}\log (np)}{\sqrt{n}} \bigg\}\bigg) \geq 1 - Cp^{-10}.
\end{align*}

Let $\widehat{\mathsf{G}}^{(j)} \sim \mathcal{N}(0, \sigma_\xi^2 \cdot \widehat{\Sigma}^{(j)} )$. Then on event $ \mathcal{E}_\Sigma \cap \mathfrak{E}_\Sigma$, we can resort to a similar argument as in (\ref{ineq:moder_1}) to obtain that for $t \in (0,n^{1/7}/C)$, 
\begin{align*}
	&\Prob^{\widehat{\bm{Z}}} \bigg( \sqrt{n} \biggabs{\frac{\langle \widehat{\bm{z}}_j, \bm{\xi} \rangle}{\langle \widehat{\bm{z}}_j, \widehat{\bm{Z}}_{\cdot,j} \rangle} }   -  \sqrt{n} \biggabs{\frac{\langle \widehat{\bm{z}}_{\bar{j}}, \bm{\xi} \rangle}{\langle \widehat{\bm{z}}_{\bar{j}}, \widehat{\bm{Z}}_{\cdot,\bar{j}} \rangle}}\geq t \bigg) \\
	&= \Prob^{\widehat{\bm{Z}}} \big( \abs{\widehat{\mathsf{G}}^{(j)}_1 } - \abs{\widehat{\mathsf{G}}^{(j)}_2 } \geq t \big)\cdot \bigg(1 + \mathcal{O}\bigg(\frac{(1 + t^3)b_n\log^3 (np) }{\sqrt{n}} \bigg) \bigg) + \mathcal{O}(p^{-3}).
\end{align*}
Further, by a similar covariance approximation argument as in the proof of Theorem \ref{thm:moder_dis_ols}, it holds that 
\begin{align*}
	&\Prob \bigg( \sqrt{n} \biggabs{\frac{\langle \widehat{\bm{z}}_j, \bm{\xi} \rangle}{\langle \widehat{\bm{z}}_j, \widehat{\bm{Z}}_{\cdot,j} \rangle} }   -  \sqrt{n} \biggabs{\frac{\langle \widehat{\bm{z}}_{\bar{j}}, \bm{\xi} \rangle}{\langle \widehat{\bm{z}}_{\bar{j}}, \widehat{\bm{Z}}_{\cdot,\bar{j}} \rangle}}\geq t \bigg) \\
	&=\Prob \big( \abs{{\mathsf{G}}^{(j)}_1 } - \abs{{\mathsf{G}}^{(j)}_2 } \geq t \big)\cdot \bigg(1 + \mathcal{O}\bigg(\frac{(1 + t^3)b_n\log^3 (np) }{\sqrt{n}} \bigg) \bigg) + \mathcal{O}(p^{-3}).
\end{align*}

Therefore, we can deduce that 
\begin{align*}
    & \mathcal{P}_j\bigg( t+\frac{ C b_n^{1/2} s \log p  }{\sqrt{n} } \bigg) \cdot \bigg(1 + \mathcal{O}\bigg(\frac{(1 + t^3)b_n\log^3 (np) }{\sqrt{n}} \bigg) \bigg) - \mathcal{O}(p^{-8}) \\
	&\leq \Prob \big( \widehat{\bm{W}}_j \geq t  \big) \leq \mathcal{P}_j\bigg( t-\frac{ C b_n^{1/2} s \log p  }{\sqrt{n} } \bigg) \cdot \bigg(1 + \mathcal{O}\bigg(\frac{(1 + t^3)b_n\log^3 (np) }{\sqrt{n}} \bigg) \bigg) + \mathcal{O}(p^{-8}).
\end{align*}
Then a Gaussian integral calculation as in the proof of Theorem \ref{thm:moderate_deviation} leads to 
\begin{align*}
	\Prob \big( \widehat{\bm{W}}_j \geq t  \big) = \mathcal{P}_j(t) \cdot \bigg(1 + \mathcal{O}\bigg(\frac{(1 + t^3)b_ns\log^3 (np) }{\sqrt{n}} \bigg) \bigg) + \mathcal{O}(p^{-8}),
\end{align*}
which completes the proof for the moderate deviation result in (\ref{ineq:moder_dis_dl}). The asymptotic symmetry result in (\ref{ineq:app_sym_dl}) then follows by some simple algebra. This completes the proof of Theorem \ref{thm:moder_dis_dl}. 

\subsubsection{Proof of Lemma \ref{lem:indicator_app_dl}}\label{sec:proof_lem:indicator_app_dl}
Combining the proof of Lemma \ref{lem:var_bound_ols} and the error bound in (\ref{ineq:moder_dis_dl}), we see that there exists some constant $C = C(\mathsf{M},\mathsf{C}_1,\mathsf{C}_2,\mathsf{C}_3)>0 $ such that for $n \wedge p \geq C$ and $t \in (0, n^{1/7}/(Cb_ns\log^4 (np) )^{1/3}  )$,
     \begin{align*}
         & \var \Big( \sum_{ j \in \mathcal{H}_0}\bm{1} \{ \widehat{\bm{W}}_j \geq t \} \Big)\\
         & \leq C\Bigg(\sum_{j\in \mathcal{H}_0}  (\abs{N(j)} + 1)  \Prob \big( \widehat{\bm{W}}_j \geq t \big) +  \frac{(1+t^3)b_n s\log^4 (np) }{\sqrt{n}}  \bigg(\sum_{j\in \mathcal{H}_0} \Prob \big( \widehat{\bm{W}}_j \geq t \big)    \bigg)^2\Bigg)+ \mathfrak{o}(p^{-2}).
     \end{align*}
    Then the conclusion follows from the proof of Lemma \ref{lem:indicator_app_ols}, which completes the proof of Lemma \ref{lem:indicator_app_dl}. 

\subsubsection{Proof of Lemma \ref{lem:local_T_dl}}
The proof of this lemma follows a similar approach to that of Lemma \ref{lem:local_T}, involving a careful bookkeeping of errors arising from the distributional approximation and the indicator function approximation. For brevity, the details are omitted. 

\subsection{Proofs for Section \ref{sec:mcks_correlated}}\label{sec:proof_sec_mcks_correlated}

\subsubsection{Proof of Proposition \ref{prop:conc_knockoff_mc_c}}
The proof follows the same structure as that of Proposition \ref{prop:conc_knockoff}, with minor modifications; we omit repetitive details. 

\subsubsection{Proof of Theorem \ref{thm:FDR_mc_c}}

The proof follows the same structure as that of Theorem \ref{thm:FDR_mcks}; we provide some details below. The constant $C$ below may depend on $\mathsf{M}$, which may change from line to line. In the proof below, we assume that $\log p = \mathfrak{o}(n^c)$ for some small enough $c > 0$.

\noindent
\textbf{Distributional approximation}

\noindent
Recall from \eqref{ineq:def_of_E_Y} that 
     \begin{align*}
         \Prob \bigg(\mathcal{E}_Y \equiv \bigg\{ \frac{\max_{i \in [n]} \abs{\bm{Y}_i}}{\pnorm{\bm{Y}}{} } \leq \frac{C \log (np)} {\sqrt{n}}  \bigg\} \bigg) \geq 1 - p^{-10}.
     \end{align*}
Then on event $\mathcal{E}_Y$, for any $j \in \mathcal{H}_0$, we can apply Theorem \ref{thm:moderate_deviation} by letting $\mathsf{X}_i$ therein as  
 $( \sqrt{n}{\bm{X}}_{i,j}\bm{Y}_i/ \pnorm{\bm{Y}}{},  \sqrt{n}\widehat{\bm{X}}_{i,j}\bm{Y}_i/\pnorm{\bm{Y}}{})^\top$. This yields that on event $\mathcal{E}_Y$, for any $j \in \mathcal{H}_0$ and $t \in (0, n^{1/7}/C)$,  
\begin{align*}
   &\Prob^{\bm{Y}} \big( \widehat{\bm{W}}_j \geq t \big)=\bigg(1 + \mathcal{O}\bigg(\frac{(1 + t^3)\log^3 (np) }{\sqrt{n}} \bigg) \bigg) \cdot  \Prob_{(\mathsf{G}_1, \mathsf{G}_2)^\top \sim \mathcal{N}(0,  \Sigma_{\{j,j+p \}, \{j,j+p \} }  ) }(\abs{\mathsf{G}_1} - \abs{\mathsf{G}_2}  \geq t  )
\end{align*}

It follows from the law of total probability that 
\begin{align}\label{ineq:moder_dis_mc_c1}
    \Prob\big( \widehat{\bm{W}}_j \geq t \big) = \bigg(1 + \mathcal{O}\bigg(\frac{(1 + t^3)\log^3 (np) }{\sqrt{n}} \bigg) \bigg)\cdot \mathcal{P}_j(t) + \mathcal{O}(p^{-10}).
\end{align}
A similar argument leads to 
\begin{align}\label{ineq:moder_dis_mc_c2}
    \Prob\big( \widehat{\bm{W}}_j \leq -t \big) = \bigg(1 + \mathcal{O}\bigg(\frac{(1 + t^3)\log^3 (np) }{\sqrt{n}} \bigg) \bigg)\cdot \mathcal{P}_j(t) + \mathcal{O}(p^{-10}).
\end{align}
Thus, by some simple algebraic calculations, we can obtain that 
\begin{align}\label{ineq:appro_symm_mc_c}
    \sup_{t \in \big(0, \mathcal{P}^{-1}\big(\frac{qa_n}{2p}\big) \big)}\biggabs{\frac{\sum_{j \in \mathcal{H}_0} \Prob \big( \widehat{\bm{W}}_j \geq t  \big)}{\sum_{j \in \mathcal{H}_0} \Prob \big( \widehat{\bm{W}}_j \leq -t   \big)} - 1}= \mathfrak{o}(1).
\end{align}

\medskip
\noindent
\textbf{Approximation for indicator functions}

\noindent
In view of the proof of Lemma \ref{lem:var_bound_ols} and the error bounds in (\ref{ineq:moder_dis_mc_c1}) and (\ref{ineq:moder_dis_mc_c2}), there exists some constant $C = C(\mathsf{M})>0 $ such that for $n \wedge p \geq C$ and $t \in (0, n^{1/7}/C  )$,
     \begin{align*}
         & \var \Big( \sum_{ j \in \mathcal{H}_0}\bm{1} \{ \widehat{\bm{W}}_j \geq t \} \Big)\\
         & \leq C\Bigg(\sum_{j\in \mathcal{H}_0}  (\abs{M(j)} +1)  \Prob \big( \widehat{\bm{W}}_j \geq t \big) +  \frac{(1+t^3)\log^3 (np) }{\sqrt{n}}  \bigg(\sum_{j\in \mathcal{H}_0} \Prob \big( \widehat{\bm{W}}_j \geq t \big)    \bigg)^2\Bigg)+ \mathcal{O}(p^{-10}).
     \end{align*}
    Then we can follow the proof of Lemma \ref{lem:indicator_app_ols} to conclude that         for $\# \in \{ +,- \}$, 
    \begin{align}\label{ineq:indicator_app_mc_c}
      \sup_{t \in \big(0, \mathcal{P}^{-1}\big(\frac{qa_n}{2p}\big)\big)} \biggabs{\frac{ \sum_{j \in \mathcal{H}_0}  \bm{1}\{ \# \widehat{\bm{W}}_j  \geq t\}  }{\sum _{j \in \mathcal{H}_0}  \Prob (\# \widehat{\bm{W}}_j   \geq t)  } - 1 } = \mathfrak{o}_\mathbf{P}(1).
    \end{align}

\medskip
\noindent
\textbf{Localization of $T_q$}

\noindent  With (\ref{ineq:moder_dis_mc_c1})--(\ref{ineq:indicator_app_mc_c}) at hand,  we can invoke the proof of Lemma \ref{lem:local_T} to obtain that 
\begin{align}\label{ineq:local_T_mc_c}
         \Prob \bigg( {T}_q > \mathcal{P}^{-1} \bigg( \frac{q  a_n }{2p}\bigg)  \bigg)  = \mathfrak{o}(1).
\end{align}

The asymptotic FDR control then follows by applying  Theorem \ref{thm:general_framwork}  with  (\ref{ineq:appro_symm_mc_c}), (\ref{ineq:indicator_app_mc_c}), and (\ref{ineq:local_T_mc_c}). This completes the proof of Theorem \ref{thm:FDR_mc_c}. 

\subsection{
Additional technical details}

We need the following Rosenthal's inequality. 

\begin{lemma}[\cite{Rosenthal1970}]\label{thm:rosenthal}
Let \( X_1, \dots, X_d \) be independent mean-zero random variables. For any \( \mathfrak{q} \geq 2 \), there exists some constant \( C_\mathfrak{q} \) depending only on \( \mathfrak{q} \) such that
\begin{align*}
	\mathbb{E}\biggabs{ \sum_{i \in [d]} X_i }^\mathfrak{q} \leq C_\mathfrak{q} \Bigg( \sum_{i \in [d]} \mathbb{E}|X_i|^\mathfrak{q} + \bigg( \sum_{i \in [d]} \mathbb{E} |X_i|^2 \bigg)^{\mathfrak{q}/2} \Bigg).
\end{align*}
\end{lemma}

The following large deviation inequality for the sum of independent nonnegative random variables is useful.

\begin{lemma}\label{thm:conc_nonnegative}
	Let $\mathsf{X}_1,\ldots,\mathsf{X}_n$ be independent nonnegative random variables satisfying $\max_{i \in [n]} \E \mathsf{X}_i^2 \\< \infty$. Let $\mu_n \equiv \sum_{i \in [n]}\E \mathsf{X}_i$ and $B_n \equiv \sum_{i \in [n]}\E \mathsf{X}_i^2$. Then for $0 <x < \mu_n $, we have 
	\begin{align*}
		\Prob \bigg( \sum_{i \in [n]} \mathsf{X}_i \leq x \bigg) \leq \exp \bigg( -\frac{(\mu_n - x)^2}{2B_n^2} \bigg).
	\end{align*}
\end{lemma}

\noindent\textit{Proof of Lemma \ref{thm:conc_nonnegative}}. 
	Note that for any $a \geq 0$, $e^{-a} \leq 1 -a + a^2/2$. Thus it holds that for any $t \geq 0$ and $x \leq \mu$,
	\begin{align*}
		\Prob \bigg( \sum_{i \in [n]} \mathsf{X}_i \leq x \bigg) &\leq e^{tx} \prod_{i \in [n]} \E e^{-t\mathsf{X}_i}\leq e^{tx}\prod_{i \in [n]} \E (1 -t\mathsf{X}_i +t^2\mathsf{X}_i^2/2 ) \\
		&\leq \exp ( -t(\mu_n-x) + t^2B_n^2/2).
	\end{align*}
	The conclusion follows by choosing $t = (\mu_n-x)/B_n^2$, which completes the proof of Lemma \ref{thm:conc_nonnegative}. 

The following proposition gives a generalized Azuma's inequality, which is a slight modification of \cite[Proposition 34]{Tannonhermitian2015}.
\begin{proposition}\label{prop:generalize_azuma}
    Let $\mathsf{X} \equiv (\mathsf{X}_1,\ldots,\mathsf{X}_n ) \in \R^{d \times n}$ be a random matrix with independent columns, and $f(\mathsf{X}) \equiv f(\mathsf{X}_1,\ldots, \mathsf{X}_n) \in \R$ a random variable depending on $\mathsf{X}_j, \, j \in [n]$. Then for any $\alpha_j \geq 0, \, j \in [n]$, we have
    \begin{align*}
          \Prob \Bigg( \abs{f(\mathsf{X}) - \E f(\mathsf{X})} \ge t\bigg(\sum_{j \in [n]}\alpha_j^2 \bigg)^{1/2} \Bigg) = \bigo\Big(\exp \big(-\Omega(t^2)\big)\Big) + \sum_{j \in [n]}\Prob \Big(C_j(\mathsf{X}) \ge \alpha_j\Big),
    \end{align*}
    where for $j \in [n]$, $C_j(\mathsf{X})$ is  defined as
    \begin{align*}
        	C_j(\mathsf{X}) \equiv \bigabs{\E \big( f(\mathsf{X}) \,|\, \mathsf{X}_1,\ldots,\mathsf{X}_j\big) - \E \big( f(X) \,|\, \mathsf{X}_1,\ldots,\mathsf{X}_{j-1} \big)}.
    \end{align*}
\end{proposition}

\noindent\textit{Proof of Proposition \ref{prop:generalize_azuma}}. 
     For each $\mathsf{X}$, let $j_{\mathsf{X}}$ be the first index where $C_j(\mathsf{X}) \ge \alpha_j$, and thus, sets $B_j \equiv \{\mathsf{X}: j_{\mathsf{X}} = j \}$ are disjoint. To see the measurability of $B_j$, we can rewrite it as
    \begin{align*}
        B_j =\bigg( \bigcap_{k \in [j-1]} \big\{\mathsf{X}: C_k(\mathsf{X}) < \alpha_k \big\}\bigg) \bigcap \big\{\mathsf{X}: C_j(\mathsf{X}) \ge \alpha_j \big\}.
    \end{align*} 
    Define another function $\tilde{f}(\mathsf{X})$ of $\mathsf{X}$ which agrees with $f(\mathsf{X})$ for $\mathsf{X}$ in $(\cup_{j \in [n]}B_j)^c$, and $\tilde{f}(X) \equiv \E (f(\mathsf{X}) \,|\, \bm{1}\{\mathsf{X}\in B_j\})  $ if $\mathsf{X} \in B_j$. Then we can deduce that 
    \begin{align*}
        &\E \tilde{f}(\mathsf{X}) = \E \tilde{f}(\mathsf{X}) \bm{1}\Big\{\mathsf{X} \in \Big(\bigcup_{j\in [n]}B_j\Big)^c\Big\} + \E \tilde{f}(\mathsf{X})\bm{1}\Big\{\mathsf{X} \in \bigcup_{j\in [n]}B_j\Big\} \\
        &=\E f(\mathsf{X}) \bm{1}\Big\{\mathsf{X} \in \Big(\bigcup_{j\in [n]}B_j\Big)^c\Big\} + \sum_{j \in [n]}\E \Big[ \E \Big( f(\mathsf{X}) \,|\, \bm{1}\{\mathsf{X} \in B_j \} \Big) \cdot \bm{1}\{\mathsf{X} \in B_j \} \Big] \\
        &\overset{(\ast)}{=}\E f(\mathsf{X})\bm{1}\Big\{\mathsf{X} \in \Big(\bigcup_{j\in [n]}B_j\Big)^c\Big\} + \E f(\mathsf{X})\bm{1}\Big\{\mathsf{X} \in \bigcup_{j\in [n]}B_j\Big\} = \E f(\mathsf{X})
    \end{align*}
    and
    \begin{align*}
        \Prob\Big (f(\mathsf{X}) \neq \tilde{f}(\mathsf{X})\Big) \le \sum_{j\in [n]} \Prob\Big (\mathsf{X} \in B_j\Big)\le \sum_{j \in [n]} \Prob\Big (C_j(\mathsf{X}) \ge \alpha_j\Big).
    \end{align*}
    Here, in step $(\ast)$ above, we have used Fubini's theorem to obtain that $\E (\E (f(\mathsf{X}) | \bm{1} \{\mathsf{X}\in B_j\})  \bm{1}\{\mathsf{X} \in B_j\} ) =\E (f(\mathsf{X})  \bm{1}\{\mathsf{X} \in B_j\} )$. 
    
    We define the  following new absolute value of martingale differences defined via $\tilde{f}(\mathsf{X})$,
\begin{align*}
        	\tilde{C}_j(\mathsf{X}) \equiv \bigabs{\E \big( \tilde{f}(\mathsf{X}) \,|\, \mathsf{X}_1,\ldots,\mathsf{X}_j\big) - \E \big( \tilde{f}(X)\,|\,\mathsf{X}_1,\ldots,\mathsf{X}_{j-1} \big)}.
    \end{align*}
    When $\mathsf{X} \in (\cup_{j \in [n]} B_j)^c$, it is easy to see that $\tilde{C}_j(\mathsf{X}) \le \alpha_j$ for all $j \in [n]$. If $\mathsf{X} \in B_j$, we have that for any $k \in [n]$,
    \begin{align*}
        \tilde{C}_k(\mathsf{X})& = \bigabs{\E \big[  \E \big( f(\mathsf{X}) \,|\,\bm{1}\{\mathsf{X}\in B_j\} \big) \cdot\bm{1}\{\mathsf{X}\in B_j\} \,|\, \mathsf{X}_1,\ldots,\mathsf{X}_k \big] \\
        &\quad- \E \big[ \E \big( f(\mathsf{X}) \,|\,|\bm{1}\{\mathsf{X}\in B_j\} \big) \cdot\bm{1} \{\mathsf{X}\in B_j\} \,|\, \mathsf{X}_1,\ldots,\mathsf{X}_{k-1} \big] }\\        
        &= \bigabs{\E \big[ \E [f(\mathsf{X}) \,|\,\bm{1}\{\mathsf{X}\in B_j\} = 1] \,|\, \mathsf{X}_1,\ldots,\mathsf{X}_k \big] \\
        &\quad - \E \big[ \E [f(\mathsf{X}) \,|\,\bm{1}\{\mathsf{X}\in B_j\} = 1] \,|\, \mathsf{X}_1,\ldots,\mathsf{X}_{k-1} \big]  }  = 0.
    \end{align*}
     This implies that for all $j \in [n]$, $
        \tilde{C}_j(\mathsf{X}) \le \alpha_j$ and   
    \begin{align*}
        &\Prob \Bigg( |f(\mathsf{X}) - \E f(\mathsf{X})| \ge t\bigg(\sum_{j \in [n]} \alpha_j^2 \bigg)^{1/2} \Bigg) \\
        &\le \Prob \Bigg( |\tilde{f}(\mathsf{X}) - \E \tilde{f}(\mathsf{X})| \ge t\bigg(\sum_{j \in [n]}\alpha_j^2 \bigg)^{1/2} \Bigg) +\Prob\Big (f(\mathsf{X}) \neq \tilde{f}(X)\Big) \\
        &\le \Prob \Bigg( |\tilde{f}(\mathsf{X}) - \E \tilde{f}(\mathsf{X})| \ge t\bigg(\sum_{j \in [n]}\alpha_j^2 \bigg)^{1/2} \Bigg) + \sum_{j \in [n]}\Prob \Big(C_j(\mathsf{X}) \ge \alpha_j\Big)\\
        &= \bigo\Big(\exp \big(-\Omega(t^2)\big)\Big) + \sum_{j \in [n]}\Prob \Big(C_j(\mathsf{X}) \ge \alpha_j\Big).
    \end{align*}
    Here, in the first step above we have used $\E \tilde{f}(\mathsf{X}) = \E f(\mathsf{X})$, and in the last step above we have applied the classical Azuma's inequality on $\tilde{C}_j(\mathsf{X})$. This completes the proof of Proposition \ref{prop:generalize_azuma}. 

We also need the following elementary matrix inequality.
\begin{lemma}\label{lem:matrix_ineq1}
    For any positive definite matrices $A, B \in \R^{d \times d}$, we have
    \begin{align*}
        \pnorm{A^{3/2} - B^{3/2}}{\op} \leq 3 ( \pnorm{A}{\op} \vee \pnorm{B}{\op}  )^{1/2} \pnorm{A - B}{\op}.
    \end{align*}
\end{lemma}

\noindent\textit{Proof of Lemma \ref{lem:matrix_ineq1}}. 
   This result is provided in \cite[Lemma 7]{fukumizu2007kernel} for the Hilbert--Schmidt norm. The proof for the operator norm is similar; we include it here for completeness. Without loss of generality, we assume that $ \pnorm{A}{\op} \vee \pnorm{B}{\op} = 1$. Define functions $f$ and $g$ on $\{z \in \mathbb{C} : \abs{z} \leq 1 \}$ as $f(z) = (1-z)^{3/2}$ and $g(z) = (1-z)^{1/2}$. We can express the functions in terms of power series as 
   \begin{align*}
       f(z) = \sum_{j = 0}^\infty b_j z^j, \quad g(z) = \sum_{j = 0}^\infty a_j z^j,
   \end{align*}
   where $b_0  = 1$, $b_1 = -3/2$, $b_j  = \prod_{k=1}^{j}(5/2-k)$ for $j \geq 2$, $c_0 = 1$, and $c_j = -\prod_{k = 1}^{j}(3/2 - k)$ for $j \geq 1$. Moreover, it can be verified that for $j \geq 1$, $j\abs{b_j} = 3c_{j-1}/2$. Noticing that for any $J \in \mathbb{N}$,
   \begin{align*}
       \sum_{j = 0}^J \abs{b_j} = 1 + \frac{3}{2} +  \sum_{j = 2}^J \abs{b_j} \leq 1 + \frac{3}{2} + \lim_{x \uparrow 1} \bigg\{ f(x) - 1 + \frac{3}{2} \bigg\} = 3,
   \end{align*}
   we have that $\sum_{j = 0}^\infty \abs{b_j} \leq 3$. 
   
   Similarly, it holds that $\sum_{j = 0}^\infty \abs{c_j} \leq 2$, and thus, $\sum_{j = 0}^\infty j\abs{b_j} =  \sum_{j = 0}^\infty 3\abs{c_j}/2 \leq 3$. Further, using the fact that $\pnorm{(I-A)^j - (I-B)^j}{\op} \leq j\pnorm{A-B}{\op}$ for all $j\in \mathbb{N}$ (proved by induction), we can deduce that 
   \begin{align*}
       \pnorm{A^{3/2} - B^{3/2} }{\op} &= \pnorm{(I - (I-A))^{3/2} -  (I - (I-B))^{3/2} }{\op} \\
       &= \pnorm{ \sum_{j = 0}^\infty b_j ( (I-A)^j - (I-B)^j)  }{\op} \\
       &\leq \sum_{j = 0}^\infty j\abs{ b_j} \pnorm{A - B}{\op} \leq 3  \pnorm{A - B}{\op},
   \end{align*}
   which completes the proof of Lemma \ref{lem:matrix_ineq1}. 

\end{document}